\documentclass[lettersize,journal]{IEEEtran}
\usepackage{amsmath,amsfonts,amssymb}
\usepackage{algorithmic}
\usepackage{algorithm}
\usepackage{array}
\usepackage[caption=false,font=normalsize,labelfont=sf,textfont=sf]{subfig}
\usepackage{textcomp}
\usepackage{stfloats}
\usepackage{url}
\usepackage{verbatim}
\usepackage{graphicx}
\usepackage{cite}
\usepackage{subfiles} 
\usepackage{xcolor}
\usepackage{amsthm}
\usepackage{amsmath}
\usepackage{mathtools}
\usepackage[switch]{lineno}
\usepackage{booktabs}
\usepackage[hidelinks]{hyperref}
\usepackage{textcircled}
\usepackage{cases}
\usepackage{bm}

\newtheorem{theorem}{Theorem}

\newtheorem{lemma}{Lemma}

\newtheorem{assumption}{Assumption}

\newtheorem{observation}{Observation}

\begin{document}

\title{Hierarchical Federated Learning with Momentum Acceleration in Multi-Tier Networks}

\author{Zhengjie~Yang, 
Sen~Fu,
Wei~Bao, 
Dong~Yuan,
and~Albert~Y.~Zomaya
}



\maketitle

\begin{abstract}
In this paper, we propose Hierarchical Federated Learning with Momentum Acceleration (HierMo), a three-tier worker-edge-cloud federated learning  algorithm that applies momentum for training acceleration. Momentum is calculated and aggregated in the three tiers. 
We provide convergence analysis for HierMo, showing a convergence rate of $\mathcal{O}\left(\frac{1}{T}\right)$. In the analysis, we develop a new approach to characterize model aggregation, momentum aggregation, and their interactions.
Based on this result, {we prove that HierMo achieves a tighter convergence upper bound compared with HierFAVG without momentum}. We also propose HierOPT, which optimizes the aggregation periods (worker-edge and edge-cloud aggregation periods) to minimize the loss given a limited training time. 
By conducting the experiment, we verify that HierMo outperforms existing mainstream benchmarks  under a wide range of settings. In addition, HierOPT can achieve a near-optimal performance when we test  HierMo under different aggregation periods.

\end{abstract}


\section{Introduction}

With the advancement of Industry 4.0, Internet of Things (IoT), and Artificial Intelligence,  machine learning applications such as image classification~\cite{lu2007survey}, automatic driving~\cite{naranjo2005power}, and automatic speech recognition~\cite{yu2016automatic} are rapidly developed. Since the machine learning dataset is distributed in individual users and in many situations they are not willing to share these sensitive raw data, Federated Learning (FL) emerges~\cite{mcmahan2017FedAvg}. It allows workers to participate in the model training without sharing their raw data. Typically, FL is  implemented in two tiers, where multiple  devices (workers) are distributed and connected to a remote aggregator (usually located in the cloud). A potential issue of the two-tier FL setting is its scalability.  The communication overhead between workers and the cloud is proportional to the number of workers, which causes problems when there are a large number of geo-distributed workers connecting to the remote cloud via the public Internet. 


With the development of edge computing~\cite{lim2020federated}, a more effective solution is  
adding the edge tier between local workers and the remote cloud to address the scalability issue. Different from the typical two-tier architecture, in the three-tier hierarchical architecture as shown in Fig.~\ref{fig:architecture}, workers can first communicate with the edge node for edge-level aggregation, and then the edge nodes communicate with the remote cloud for cloud-level aggregation. 
Each edge node is closer to the workers and is usually connected with them in the same local/edge network, so that the communication cost is much cheaper compared with the two-tier case when the workers directly communicate with the cloud. In Fig.~\ref{fig:architecture}, we can see that much of the traffic through the public Internet (left subfigure) is restrained in the local edge networks (right subfigure) due to the existence of the edge nodes. 
Therefore,  the three-tier architecture is a good fit for larger-scale FL, and has  attracted attentions from researchers in recent years~\cite{abad2020hierarchical,christopher2020clustering,liu2021hierQSGD}. 

Although the three-tier FL can improve the communication efficiency \emph{in one training iteration} by replacing worker-cloud communication with worker-edge communication, there is also a need to accelerate its convergence performance to reduce the \emph{number of iterations}. One obstacle in the three-tier FL is that each edge node can only aggregate the updates of its local workers, and there is a discrepancy among edge nodes. The edge nodes are to be synchronized in the cloud-level aggregation. The two-level aggregation causes delayed synchronization, leading to less training efficiency. 
Therefore, it is a strong motivation for us to develop a more efficient algorithm to accelerate the convergence, { reducing the number of training iterations} in the three-tier hierarchical architecture, and finally improve the overall training efficiency (considering both per-iteration cost and the number of iterations).



\begin{figure}[tb!]
    \centering
    \includegraphics[width=\linewidth]{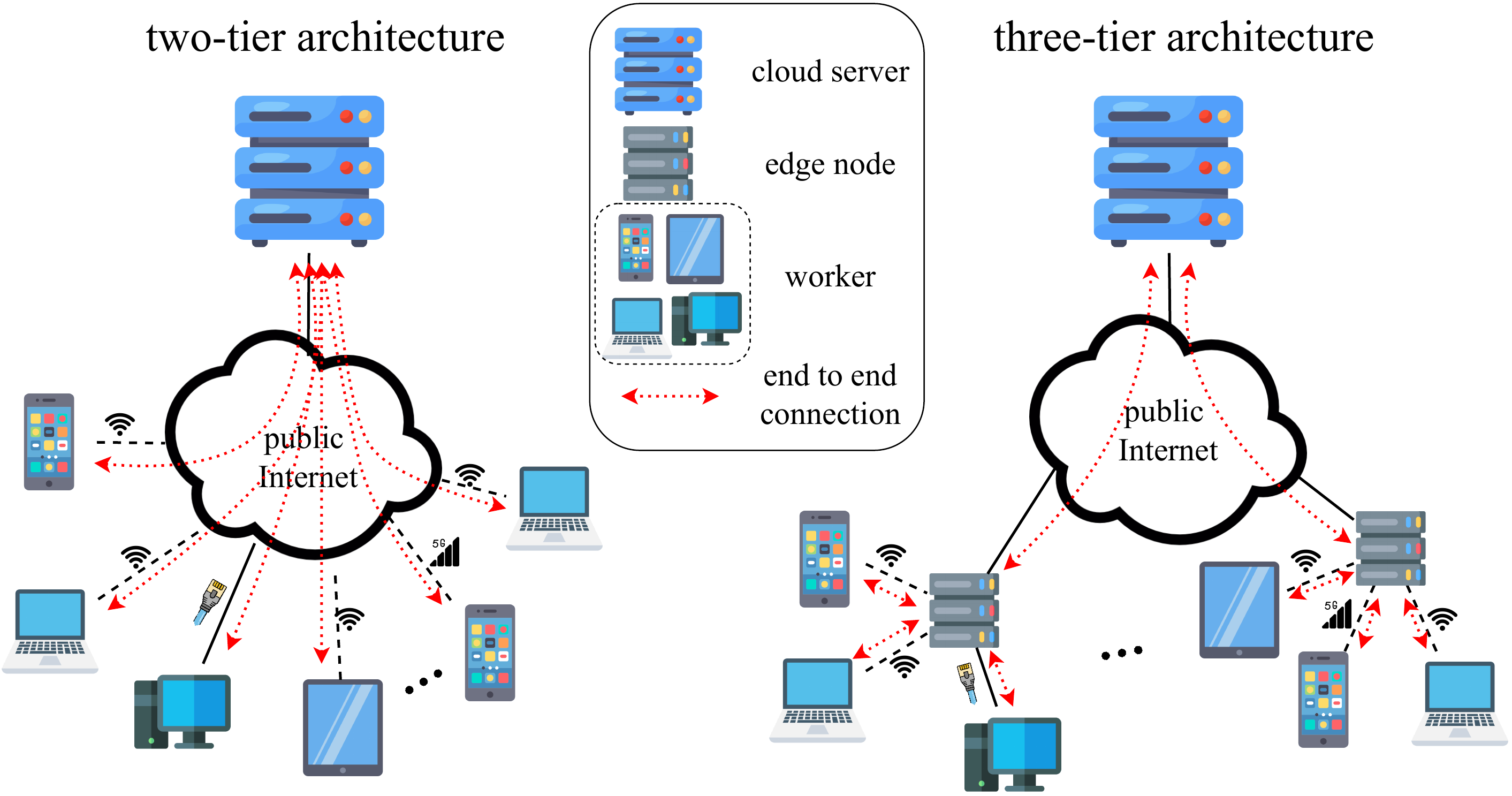}
    \caption{Two-tier architecture vs. three-tier architecture. 6 connections are through the public Internet in the two-tier architecture but only 2 connections are through the public Internet in the three-tier architecture. Communication burdens are restrained in the local/edge networks.}
    \label{fig:architecture}
\end{figure}

Momentum  is proved to be an effective mechanism to accelerate model training. Many studies have demonstrated its advantage in both centralized machine learning environment~\cite{yan2018unified,liu2018accelerating, vaswani2019fast, assran2020convergence} and two-tier FL environment~\cite{liu2020accelerating,yu2019linear,yang2021achieving,gao2021convergence}. 
Apart from the conventional gradient descent step, the momentum method conducts additional momentum steps~\cite{ruder2016overview} to accelerate convergence.   
In this paper, we propose Hierarchical Federated Learning with Momentum Acceleration (HierMo), which leverages momentum to accelerate three-tier FL. 
HierMo is operated as follows: \tikztextcircle{1} In each iteration, 
each worker locally updates its own model and worker momentum;  \tikztextcircle{2} In every $\tau$ iterations ($\tau$ is called the worker-edge aggregation period), 
each edge node receives, averages, and sends back the  models and momentum values with its connected workers.\footnote{Each edge node also calculates another momentum for its own usage to further accelerate convergence. See Section~\ref{sec:alg} for the detailed algorithm.} \tikztextcircle{3} In every $\tau\cdot\pi$ iterations ($\pi$ is called the edge-cloud aggregation period), 
the cloud receives, averages, and sends back the models and momentum values with edge nodes. The edge nodes will then distribute them to connected workers. The above \tikztextcircle{1}--\tikztextcircle{3} steps are repeated for multiple rounds until the loss  is sufficiently small.


{
Theoretically, we prove that HierMo is convergent and has an $\mathcal{O}\left(\frac{1}{T}\right)$ convergence rate for smooth non-convex problems for a given $T$ iterations. In this step, we need to address substantial new challenges, compared with two-tier FL. In particular, we develop a new method to characterize the \emph{multi-time cross-two-tier momentum interaction} and \emph{cross-three-tier momentum interaction}, which do not exist in the two-tier FL.} After we theoretically prove the convergence, we observe that the worker-edge and edge-cloud aggregation periods $\tau$ and $\pi$ are key design variables we aim to optimize. 
Based on the result of the convergence analysis, we propose HierOPT  algorithm, which can find a local optimal $(\tau, \pi)$ value pair. 


{

In the experiment, we demonstrate the performance of HierMo compared with various mainstream hierarchical FL and momentum-based FL algorithms, including  hierarchical FL without momentum (HierFAVG~\cite{liu2020HierFAVG} and CFL~\cite{wang2021CFL}), two-tier FL with momentum (FedMom~\cite{huo2020FedMom}, SlowMo~\cite{Wang2020SlowMo}, FedNAG~\cite{yang2020FedNAG}, Mime~\cite{karimireddy2020mime}, DOMO~\cite{xu2021DOMO}, and FedADC~\cite{FedADC}), and two-tier FL without momentum (FedAvg~\cite{mcmahan2017FedAvg}). The experiment is implemented on different kinds of models (linear regression, logistic regress, CNN~\cite{cnnModel}, VGG16~\cite{vgg16code}, and ResNet18~\cite{tinyimagenet}) based on various real-world datasets (MNIST~\cite{mnist}, CIFAR-10~\cite{cifar10}, ImageNet~\cite{deng2009imagenet,tinyimagenet} for image classification, and UCI-HAR~\cite{anguita2013UCI-HAR} for human activity recognition). The experimental results illustrate that HierMo drastically outperforms benchmarks under a wide range of settings. We also verify HierOPT can output a near-optimal $(\tau, \pi)$ in the real-world settings. All these results match our expectations by the theoretical analysis.

The contributions of this paper are summarized as follows.

\begin{itemize}
    \item We have proved that HierMo is convergent and has an $\mathcal{O}\left(\frac{1}{T}\right)$ convergence rate for smooth non-convex problems for a given $T$ iterations under non-i.i.d. data. 
    \item We have proved that as long as learning step size $\eta$ is sufficiently small, HierMo (with momentum acceleration) achieves the tighter convergence upper bound than HierFAVG (without momentum acceleration).
    \item We have proposed the new HierOPT algorithm which can find a local optimal pair of $(\tau^*, \pi^*)$ when total training time is constrained.
    \item HierMo is efficient and decreases the total training time by 21--70\% compared with the mainstream two-tier momentum-based algorithms and three-tier algorithms.
    \item HierOPT generates the near-optimal pair of $(\tau^*, \pi^*)$ when the total training time is constrained. HierOPT achieves the near-optimal accuracy with only 0.23--0.29\% (CNN on MNIST) and 0.04--0.16\% (CNN on CIFAR10) gap from the real-world optimum.
    
\end{itemize}

}


The rest of the paper is organized as follows. In Section~\ref{sec:related_work}, we introduce related works. The HierMo algorithm design is described in Section~\ref{sec:alg}. In Section~\ref{sec:theory}, we provide theoretical results including the convergence analysis of HierMo {and the performance gain of momentum}. The algorithm to optimize the aggregation periods, i.e., HierOPT, is proposed in Section~\ref{sec:optmization}. Section~\ref{sec:exp} provides our experimental results and the conclusion is made in Section~\ref{sec:conclusion}.

\section{Related Work}\label{sec:related_work}
\subsection{Momentum in Machine Learning and Federated Learning}
Momentum 
\cite{polyak1964some} is a method that helps accelerate gradient descent in the relevant direction by adding a fraction $\gamma$ of the difference between past and current model vectors. In the classical centralized setting, the update rule of the momentum (Polyak's momentum) is as follows:
\begin{align}
        \mathbf{m}(t) &= \gamma\mathbf{m}(t-1) - \eta\nabla F(\mathbf{w}(t-1)),\\
        \mathbf{w}(t) &= \mathbf{w}(t-1) + \mathbf{m}(t),\label{HB.2}
\end{align}
with $\gamma\in[0,1), t=1,2,\ldots, \mathbf{m}(0)=\mathbf{0}$, where $\gamma$ is momentum factor (weight of momentum), $t$ is update iteration, $\mathbf{m}(t)$ is momentum term at iteration $t$, and $\mathbf{w}(t)$ is model parameter at iteration $t$. Through this method, the momentum term increases for dimensions whose gradients point in the same directions and reduces updates for dimensions whose gradients change directions. As a result, momentum gains faster convergence and reduces oscillation \cite{ruder2016overview,goh2017why}.

Momentum has been investigated in both centralized machine learning and FL. In the centralized environment, another form of momentum called Nesterov Accelerate Gradient (NAG)  \cite{NAG,ruder2016overview} is proposed. NAG\footnote{{There are two mainstream equivalent representations of NAG. In this paper, we employ the representation in \cite{bubeck2014convex,yan2018unified}. }} calculates the gradient based on an approximation of the next position of the parameters, i.e., $\nabla F(\mathbf{w}(t-1)+\gamma\mathbf{m}(t-1))$, instead of  $\nabla F(\mathbf{w}(t-1))$ in Polyak's momentum, leading to better convergence performance. In \cite{vaswani2019fast}, authors study the utilization of momentum in over-parameterized models. \cite{yan2018unified} provides a unified convergence analysis for both Polyak's  momentum and NAG. \cite{assran2020convergence} studies  NAG in stochastic settings. 

All the above works show the advantages of momentum to accelerate the centralized training and it attracts researchers' attention to apply momentum in FL environment. Depending on where the momentum is adopted, we can categorize them into the worker momentum, aggregator momentum, and combination momentum.  For the worker momentum (e.g., FedNAG~\cite{yang2020FedNAG} and Mime~\cite{karimireddy2020mime}), momentum acceleration is adopted at workers in each local iteration. However, it is vulnerable to data heterogeneity among workers, which may harm the long-run  performance.
For the aggregator momentum (e.g., FedMom~\cite{huo2020FedMom} and SlowMo~\cite{Wang2020SlowMo}), the momentum acceleration is adopted only at the aggregator based on the global model and it shares the same property of acceleration as in centralized setting and dampens oscillations~\cite{ruder2016overview}. Nevertheless, it is conducted less frequently (every $\tau$ iterations\footnote{$\tau$ the is aggregation period}) compared with worker momentum (every iteration), and the performance gain may not be obvious especially when $\tau$ is large. To address the above limitations, works in \cite{xu2021DOMO,yang2022fastslowmo,FedADC} combine the worker and aggregator momenta and they show a better convergence performance than only using either worker or aggregator momentum. The above forms of momentum are only adopted and analyzed in the two-tier FL and we focus on the three-tier scenarios in this paper.

\begin{table}[b!t]
\centering
\caption{Key Notations}
\label{tab:notation}
\begin{tabular}{p{0.05\columnwidth}p{0.85\columnwidth}}
\hline
$\eta$  & worker model learning rate \\
$\tau$ & worker-edge aggregation period \\
$\pi$ & edge-cloud aggregation period \\
$\gamma$ & worker momentum factor \\
$\gamma_a$ & edge momentum factor \\
$T$  & number of total local (worker) iterations indexed by $t$ \\
$K$ & number of total edge aggregations indexed by $k$ \\
$P$ & number of total global (cloud) aggregations indexed by $p$ \\

$L$ & number of edge nodes indexed by $\ell$ \\
$C_\ell$ & number of workers under edge node $\ell$ \\
$N$ & number of workers in the system indexed by $\{i,\ell\}$\\
$\boldsymbol{x}^t_{i,\ell}$ & worker model parameter in worker $\{i,\ell\}$ at iteration $t$ \\
$\boldsymbol{y}^t_{i,\ell}$ & worker momentum parameter in worker $\{i,\ell\}$ at iteration $t$ \\
{$\boldsymbol{y}^t_{\ell-}$} &  {aggregated} worker momentum in edge node $\ell$ at iteration $t$ \\
{$\boldsymbol{x}^t_{\ell-}$} &  {aggregated} worker model in edge node $\ell$ at iteration $t$ \\
{$\boldsymbol{y}^t_{\ell+}$} & {updated} edge momentum in edge node $\ell$ at iteration $t$ \\
{$\boldsymbol{x}^t_{\ell+}$} & {updated} edge model in edge node $\ell$ at iteration $t$ \\
$\boldsymbol{y}^t$ &  worker momentum cloud aggregation in the cloud at iteration $t$ \\
$\boldsymbol{x}^t$ & cloud model in the cloud at iteration $t$ \\
\hline
\end{tabular}
\end{table}

\subsection{Three-Tier Hierarchical Federated Learning}
Three-tier FL has attracted more attention in recent years. Without considering momentum, studies have demonstrated the convergence performance in three-tier FL \cite{liu2020HierFAVG,wang2020H-SGD,castiglia2020multiSGD,wang2021CFL}. The communication overhead can be further optimized in \cite{SHARE}. 
The convergence analysis extended from two-tier to three-tier FL is not straightforward. Different from two-tier FL where the global aggregation is executed every $\tau$ local iterations, in three-tier FL, each worker's local model will be first aggregated by the connected edge node every $\tau$ local iterations, and will then be aggregated by the cloud in another level of every $\pi$ edge aggregations. Existing two-tier methods can only bound the two-tier effects, but not the three-tier effects. { Substantial new challenges are encountered in this paper. When momentum is leveraged in the three-tier scenario, it additionally introduces
 \emph{multi-time cross-two-tier momentum interaction} and \emph{cross-three-tier momentum interaction}. This is completely different from the two-tier scenario.  Existing two-tier  analyses cannot deal with the above two new terms. They can only characterize multi-time inner-tier momentum acceleration and one-time cross-two-tier momentum interaction. We devise a two-level virtual update (edge and cloud) method, which is able to bound the aforementioned new terms 
so that the convergence of HierMo still holds.

}

\section{{HierMo Problem Formulation}}\label{sec:alg}

\subsection{Overview}
We consider a  three-tier hierarchical FL system consisting of a cloud server, $L$ edge nodes, and $N$ workers. Each edge node $\ell$ serves $C_\ell$  workers, and the total number of workers is $N=\sum_{\ell=1}^{L}C_\ell$. Worker $\{i,\ell\}$ denotes the $i$th worker served by edge node $\ell$, where $i=1,2,\ldots, C_\ell$. It contains its local dataset with the number of data samples denoted by $D_{i,\ell}$. The total training dataset in the cluster of workers served by edge node $\ell$ is $D_\ell\triangleq\sum_{i=1}^{C_\ell}D_{i,\ell}$ and the total training dataset  $D\triangleq\sum_{\ell=1}^{L}D_\ell=\sum_{\ell=1}^{L}\sum_{i=1}^{C_\ell}D_{i,\ell}$. The target of three-tier hierarchical FL is to find the stationary point $\boldsymbol{x}^*$ that minimizes the global loss function $F(\boldsymbol{x})$ that is the weighted average of all workers' loss functions. 
The problem can be formulated as follows:
\begin{align}\label{problem}
\min_{\boldsymbol{x}\in\mathbb{R}^d} F(\boldsymbol{x}) &\triangleq \frac{1}{D}\sum_{\ell=1}^{L}\sum_{i=1}^{C_\ell}D_{i,\ell}F_{i,\ell}(\boldsymbol{x})\\
&=\sum_{\ell=1}^{L}\frac{D_\ell}{D}\sum_{i=1}^{C_\ell}\frac{D_{i,\ell}}{D_\ell}F_{i,\ell}(\boldsymbol{x})\label{eq:min_2}\\
&\triangleq\sum_{\ell=1}^{L}\frac{D_\ell}{D}F_{\ell}(\boldsymbol{x})\label{eq:min_3},
\end{align}

\begin{algorithm}[b!t]\small
\caption{HierMo algorithm.}
\label{alg:HierMo}
\textbf{Input}: $\tau,\pi$, $T=K\tau=P\tau\pi$, $\eta$, 
$\gamma$, $\gamma_a$\\
\textbf{Output}: Final cloud (global) model parameter  $\boldsymbol{x}_{(a)}^{T}$\\
\begin{algorithmic}[1] 
\STATE For each worker, initialize: $\boldsymbol{x}_{i,\ell}^0$ as same value for all $i,\ell$, and $\boldsymbol{y}_{i,\ell}^0=\boldsymbol{x}_{i,\ell}^0$
\STATE For each edge node, initialize: $\boldsymbol{x}_{\ell(a)}^0=\boldsymbol{x}_{i,\ell}^0$, and $\boldsymbol{y}_{\ell(a)}^0=\boldsymbol{x}_{\ell(a)}^0$ \\
\FOR{$t=1,2,\ldots,T$}
\STATE For each worker $i=1,2,\ldots,N$ in parallel,
\STATE $\boldsymbol{y}_{i,\ell}^t \gets \boldsymbol{x}_{i,\ell}^{t-1} - \eta\nabla F_{i,\ell}(\boldsymbol{x}_{i,\ell}^{t-1})$\label{eq:yit}\\
// \texttt{\footnotesize Worker momentum update}
\STATE $\boldsymbol{x}_{i,\ell}^t \gets \boldsymbol{y}_{i,\ell}^t + \gamma(\boldsymbol{y}_{i,\ell}^t - \boldsymbol{y}_{i,\ell}^{t-1})$\label{eq:xit}\\
// \texttt{\footnotesize Worker model update}

\IF {$t==k\tau$ where $k = 1,\ldots, K$}
\STATE For each edge node $\ell=1,2,\ldots,L$ in parallel, 
\STATE $\boldsymbol{y}_{\ell-}^{k\tau} \gets \sum_{i=1}^{C_\ell} \frac{D_{i,\ell}}{D_\ell} \boldsymbol{y}_{i,\ell}^{k\tau}$\label{eq:yktau}\\
// \texttt{\footnotesize Worker momentum edge aggregation}
\STATE $ \boldsymbol{y}_{\ell+}^{k\tau} \gets \boldsymbol{x}_{\ell+}^{(k-1)\tau} - 
\sum_{i=1}^{C_\ell} \frac{D_{i,\ell}}{D_\ell}\left(\boldsymbol{x}_{\ell+}^{(k-1)\tau}-\boldsymbol{x}_{i,\ell}^{k\tau}\right)$\label{eq:ygktau}\\
// \texttt{\footnotesize Edge momentum update}
\STATE $\boldsymbol{x}_{\ell+}^{k\tau} \gets \boldsymbol{y}_{\ell+}^{k\tau} +\gamma_a\left(\boldsymbol{y}_{\ell+}^{k\tau}-\boldsymbol{y}_{\ell+}^{(k-1)\tau}\right)$\label{eq:xgktau}\\
// \texttt{\footnotesize Edge  model update}
\STATE Set $\boldsymbol{y}_{i,\ell}^{k\tau} \gets \boldsymbol{y}_{\ell-}^{k\tau}$ for all worker $i\in C_\ell$\label{eq:yiktau_set}\\
// \texttt{\footnotesize Edge aggregated worker momentum}\\
\texttt{\footnotesize \quad re-distribution to workers}
\STATE Set $\boldsymbol{x}_{i,\ell}^{k\tau} \gets \boldsymbol{x}_{\ell+}^{k\tau}$ for all worker $i\in C_\ell$\label{eq:xgktau_set}\\
// \texttt{\footnotesize Edge model re-distribution to workers}
\ENDIF
\IF{$t==p\tau\pi$ where $p=1,2,\ldots,P$}
\STATE Aggregate $\boldsymbol{y}^{p\tau\pi} \gets \sum_{\ell=1}^{L} \frac{D_\ell}{D}\boldsymbol{y}_{\ell-}^{p\tau\pi}$\label{eq:yptaupi}\\
// \texttt{\footnotesize Worker momentum cloud aggregation}
\STATE Aggregate $\boldsymbol{x}^{p\tau\pi} \gets \sum_{\ell=1}^{L} \frac{D_\ell}{D}\boldsymbol{x}_{\ell+}^{p\tau\pi}$\label{eq:xaptaupi}\\
//  \texttt{\footnotesize Edge model cloud aggregation}
\STATE Set $\boldsymbol{y}_{\ell-}^{p\tau\pi} \gets \boldsymbol{y}^{p\tau\pi}$ for all edge node $l\in L$\label{eq:yptaupi_set_edge}\\
// \texttt{\footnotesize Cloud aggregated worker momentum}\\
\texttt{\footnotesize \quad re-distribution to edge nodes}
\STATE Set $\boldsymbol{x}_{\ell+}^{p\tau\pi} \gets \boldsymbol{x}^{p\tau\pi}$ for all edge node $l\in L$\label{eq:xaptaupi_set_edge}\\
// \texttt{\footnotesize Cloud model re-distribution to edge nodes}
\STATE Set $\boldsymbol{y}_{i,\ell}^{p\tau\pi} \gets \boldsymbol{y}_{\ell-}^{p\tau\pi}$ for all worker $i\in C_\ell,l\in L$\label{eq:yptaupi_set}\\
// \texttt{\footnotesize Cloud aggregated worker momentum}\\ \texttt{\footnotesize \quad re-distribution from edge nodes to workers}
\STATE Set $\boldsymbol{x}_{i,\ell}^{p\tau\pi} \gets \boldsymbol{x}_{\ell+}^{p\tau\pi}$ for all worker $i\in C_\ell,l\in L$\label{eq:xaptaupi_set}\\
// \texttt{\footnotesize Cloud model re-distribution}\\
\texttt{\footnotesize \quad from edge nodes to workers}
\ENDIF
\ENDFOR
\end{algorithmic}
\end{algorithm}

\noindent where $d$ is the dimension of $\boldsymbol{x}$,  
$F(\boldsymbol{x})$ is the global loss function at the cloud server, and $F_{i,\ell}(\boldsymbol{x})$ is the local loss function at worker $\{i,\ell\}$. {\eqref{eq:min_2} is the mathematical transformation from \eqref{problem} by adding $D_\ell$}. We also define the edge loss function at edge node $\ell$ as 
$F_\ell(\boldsymbol{x}) \triangleq \sum_{i=1}^{C_\ell}\frac{D_{i,\ell}}{D_\ell}F_{i,\ell}(\boldsymbol{x})$, which is the weighted average of edge node $\ell$'s connected workers' local loss functions $F_{i,\ell}(\boldsymbol{x})$. Therefore, {by replacing $\sum_{i=1}^{C_\ell}\frac{D_{i,\ell}}{D_\ell}F_{i,\ell}(\boldsymbol{x})$ with $F_\ell(\boldsymbol{x})$ in \eqref{eq:min_2}, we can directly derive \eqref{eq:min_3}}, demonstrating that the global loss function is the weighted average of all edge loss functions as $F(\boldsymbol{x}) \triangleq \sum_{\ell=1}^{L}\frac{D_\ell}{D} F_\ell(\boldsymbol{x})$. We assume the problem is within the scope of cross-siloed federated learning~\cite{kairouz2021advances} where all workers are required to participate in the training with siloed data. Each worker represents a repository of data, and data are sensitive and non-i.i.d.. The key notations are summarized in Table~\ref{tab:notation}.

\subsection{Worker Momentum and Edge Momentum}

We notice that there are two types of momentum in two-tier FL: One type (i.e., worker momentum) is calculated at each worker and is aggregated; The other type (i.e., aggregator momentum) is calculated at the aggregator. Since both types can accelerate the convergence, we adopt both of them in our work. 
In the three-tier case in our paper, the worker momentum is individually computed in each worker and aggregated in the edge node (worker momentum edge aggregation) and the cloud (worker momentum cloud aggregation). We still call it \emph{worker momentum} throughout the paper. For the  aggregator momentum, we apply it at each edge node. Each edge node computes its own momentum and it is not shared with the workers or the cloud. We call it \emph{edge momentum} throughout this paper. 

\subsection{HierMo Algorithm}
In Algorithm~\ref{alg:HierMo}, we propose a momentum-based three-tier hierarchical FL algorithm, named as HierMo, which applies both worker momentum and edge momentum. HierMo aims to find the final cloud model $\boldsymbol{x}_{(a)}^T$ to solve the formula \eqref{problem}. It conducts $T$ local iterations, $K$ edge aggregations, and $P$ cloud aggregations, where $T=K\tau=P\tau\pi$, $\tau$ is the worker-edge aggregation period, and $\pi$ is the edge-cloud aggregation period.

\subsubsection{Worker update} 
In each local iteration $t$, each worker~$\{i,\ell\}$ computes its worker update, which includes two things: \tikztextcircle{1} worker momentum update $\boldsymbol{y}_{i,\ell}^t$ (Line~\ref{eq:yit}) and \tikztextcircle{2} worker model update $\boldsymbol{x}_{i,\ell}^t$ (Line~\ref{eq:xit}). \tikztextcircle{1} and \tikztextcircle{2} follow the Nesterov Accelerated Gradient (NAG)~\cite{NAG} momentum update and are conducted every iteration. Through this way, each worker can utilize its own worker momentum acceleration.

\subsubsection{Edge update}
When $t=k\tau, k=1,2,\ldots,K$, each edge node $\ell$ receives workers' momenta and models in $C_\ell$ and performs edge update, which includes two operations: \tikztextcircle{1} Worker momentum edge aggregation $\boldsymbol{y}_{\ell-}^{k\tau}$ (Line~\ref{eq:yktau}) with re-distribution (Line~\ref{eq:yiktau_set}). Through this way, some straggler workers with high data-heterogeneity whose local momenta $\boldsymbol{y}_{i,\ell}^{k\tau}$ pointing to an inappropriate direction can be refined from $\boldsymbol{y}_{\ell-}^{k\tau}$. \tikztextcircle{2} Edge momentum $\boldsymbol{y}_{\ell+}^{k\tau}$ and model $\boldsymbol{x}_{\ell+}^{k\tau}$ update (Lines \ref{eq:ygktau}--\ref{eq:xgktau}) with model re-distribution (Line \ref{eq:xgktau_set}). 
Since the computation of edge momentum and model update is based on the edge model, it is equivalent to perform it in edge setting involving all workers' dataset under edge node $\ell$ ($D_\ell=\sum_{i=1}^{C_\ell}D_{i,\ell}$). By doing so, it dampens oscillations~\cite{ruder2016overview} within the edge node. { Please note that \tikztextcircle{1} and \tikztextcircle{2} are two operations on the same edge node, so that we use subscript ``$-$'' and ``$+$'' to label the momentum/model  right after operations \tikztextcircle{1} and \tikztextcircle{2} respectively.} Finally, both \tikztextcircle{1} and \tikztextcircle{2} are conducted in each edge node every $\tau$ iterations.


\subsubsection{Cloud update}
When $t=p\tau\pi, p=1,2,\ldots,P$, the cloud receives edge aggregated worker momentum $\boldsymbol{y}_{\ell-}^{p\tau\pi}$ and edge model $\boldsymbol{x}_{\ell+}^{p\tau\pi}$ for all $\ell \in L$ and performs cloud update, which includes two things: \tikztextcircle{1} Worker momentum cloud aggregation $\boldsymbol{y}^{p\tau\pi}$ (Line~\ref{eq:yptaupi}) and re-distribution (Lines~\ref{eq:yptaupi_set_edge} and \ref{eq:yptaupi_set}). 
Through this way,  all edge nodes and workers receive the cloud aggregated worker momentum and mitigate the disadvantage caused by non-i.i.d. data heterogeneity. \tikztextcircle{2} Edge model cloud aggregation $\boldsymbol{x}^{p\tau\pi}$ (Line \ref{eq:xaptaupi}) and cloud model re-distribution (Lines~\ref{eq:xaptaupi_set_edge} and \ref{eq:xaptaupi_set}).
Please note that the cloud will re-distribute the momentum and model to all edge nodes and all edge nodes will then distribute them to all workers when $t$ is a multiple of $\tau\pi$.  




\section{Convergence Analysis of HierMo}\label{sec:theory}
In this section, we present the theoretical analysis of HierMo. We first provide preliminaries. Then, we introduce the concept of virtual update which is a significant intermediate step to conduct convergence analysis. Afterward, we show the convergence guarantee of HierMo. Finally, we compare the convergence upper bound of HierMo and HierFAVG to analyze the performance gain of momentum.

\subsection{Preliminaries}
We assume $F_{i,\ell}(\cdot)$ satisfies the following standard conditions that are commonly adopted in the literature~\cite{liu2020accelerating,yang2020FedNAG,wang2019adaptive}.
\begin{assumption}\label{assum:2}
 $F_{i,\ell}(\boldsymbol{x})$ is $\rho$-Lipschitz, i.e., $\Vert F_{i,\ell}(\boldsymbol{x}_{1})-F_{i,\ell}(\boldsymbol{x}_{2})\Vert \leq \rho\Vert\boldsymbol{x}_{1}-\boldsymbol{x}_{2}\Vert$
for any $\boldsymbol{x}_{1}, \boldsymbol{x}_{2}, i, \ell$.
\end{assumption}
\begin{assumption}\label{assum:3}
$F_{i,\ell}(\boldsymbol{x})$ is $\beta$-smooth, i.e., $\Vert\nabla F_{i,\ell}(\boldsymbol{x}_{1})-\nabla F_{i,\ell}(\boldsymbol{x}_{2})\Vert\leq\beta\Vert \boldsymbol{x}_{1}-$
$\boldsymbol{x}_{2} \Vert$ for any $\boldsymbol{x}_{1}, \boldsymbol{x}_{2}, i, \ell$.
\end{assumption}
\begin{assumption}\label{def:delta}
(Bounded diversity) The variance of local gradient to edge gradient is bounded. i.e., $\Vert \nabla F_{i,\ell}(\boldsymbol{x}) - \nabla F_\ell(\boldsymbol{x}) \Vert \leq \delta_{i,\ell}$ for $\forall i$, $\forall \ell$,  and $\forall \boldsymbol{x}$. We also define $\delta_\ell$ as the weighted average of $\delta_{i,\ell}$ and $\delta$ as the weighted average of $\delta_{\ell}$, i.e., $\delta_\ell \triangleq \sum_{i\in C_\ell} \frac{D_{i,\ell}}{D_\ell}\delta_{i,\ell}$ and $\delta \triangleq \sum_{\ell \in L} \frac{D_{\ell}}{D}\delta_{\ell}$.
\end{assumption}

According to Assumptions~\ref{assum:2} and \ref{assum:3}, and applying the Triangle Inequality to $F_{i,\ell}(\boldsymbol{x})$, it is straightforward to show that $F_\ell(\boldsymbol{x})$ is $\rho$-Lipschitz and $\beta$-smooth. Applying the Triangle Inequality to $F_\ell(\boldsymbol{x})$, we can also derive that $F(\boldsymbol{x})$ is $\rho$-Lipschitz and $\beta$-smooth. { Assumptions~\ref{assum:2} and \ref{assum:3} indicate that the function and the gradient of the function are not changing too fast. Assumption \ref{def:delta} indicates that the data distributed to all workers are heterogeneous and non-i.i.d..} $\delta_{i,\ell}$ is used to quantify the level of gradient divergence and is different at different workers.


\subsection{Virtual Update}

\begin{figure}[tb!]
    \centering
    \includegraphics[width=\linewidth]{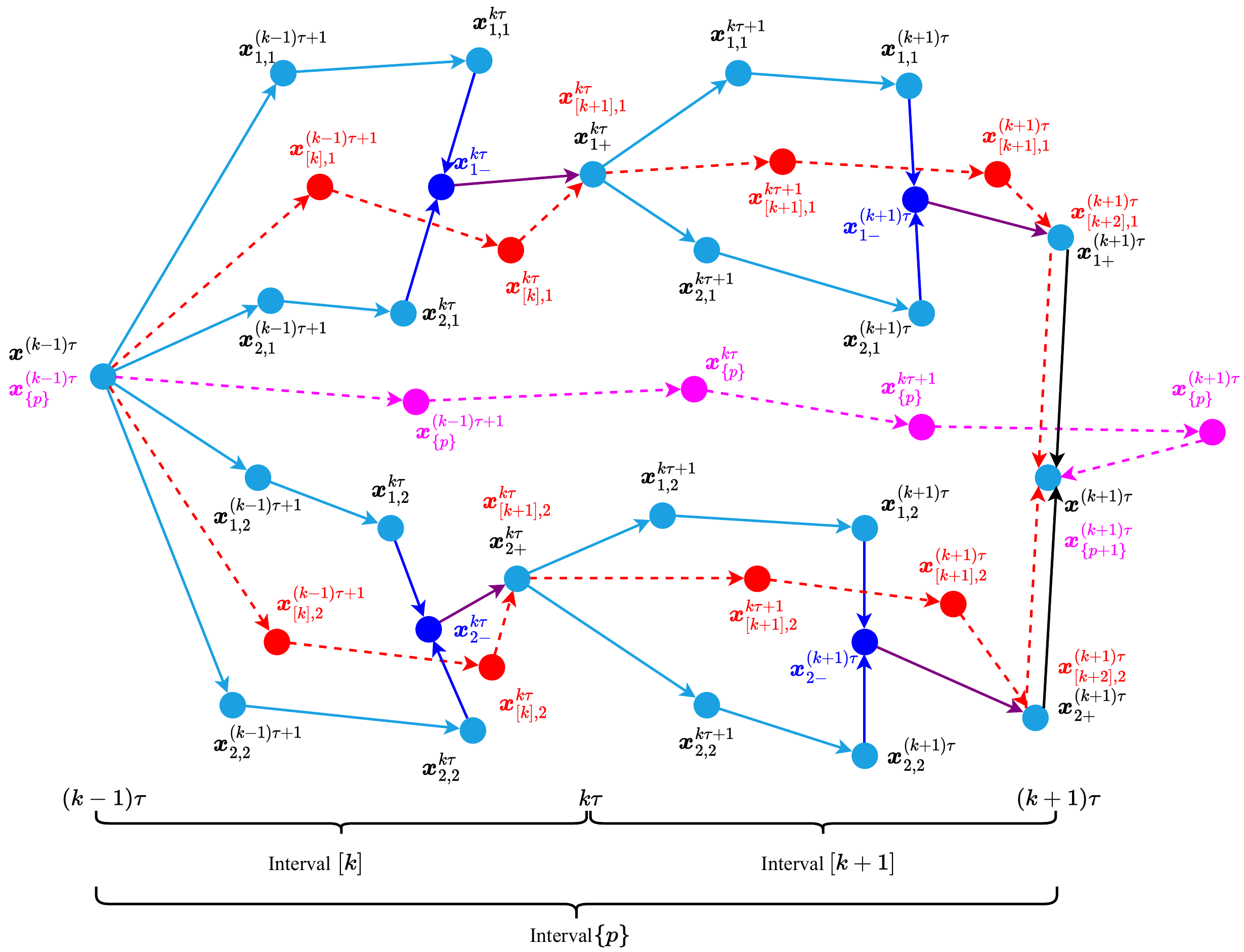}
    \caption{Illustration of $\boldsymbol{x}_{\ell-}^t$, $\boldsymbol{x}_{\ell+}^t$, $\boldsymbol{x}_{[k],\ell}^t$, $\boldsymbol{x}_{\{p\}}^t$, and $\boldsymbol{x}_{}^t$, when $N=4, \tau=2, \pi=2$ with each edge node serving 2 workers. Cyan lines show worker model update. Blue lines show worker model edge aggregation. Purple lines show edge model accelerated by edge momentum. Black lines show edge model cloud aggregation.
    Red dashed lines show edge model virtual update. Magenta dashed lines show cloud model virtual update.}
    \label{fig:xt}
\end{figure}

In order to index the edge aggregation and cloud aggregation, we divide the total $T$ local iterations into $K$ edge intervals and $P$ cloud intervals. $T=K\tau=P\tau\pi$. 
We use $[k]$ to denote the edge interval $t\in [(k-1)\tau, k\tau]$ for $k=1,2,\ldots,K$, and $\{p\}$ to denote the cloud interval $t\in [(p-1)\tau\pi, p\tau\pi]$ for $p=1,2,\ldots,P$. Please note that the edge aggregation occurs at the end of each edge interval and the cloud aggregation occurs at the end of each cloud interval. Therefore, each edge interval $[k]$ contains $\tau$ local iterations with one edge aggregation, and each cloud interval $\{p\}$ contains $\pi$ edge intervals with one cloud aggregation, i.e., $\{p\}=\cup_k[k]$ for $k=(p-1)\pi+1,(p-1)\pi+2,\ldots,p\pi$.

At the beginning of edge interval $[k]$ when $t=(k-1)\tau$, we set \emph{edge virtual update}
\begin{align}
    \boldsymbol{y}_{[k],\ell}^{(k-1)\tau} &\gets \boldsymbol{y}_{\ell-}^{(k-1)\tau}\label{eq:yk(k-1)tau},\\
    \boldsymbol{x}_{[k],\ell}^{(k-1)\tau} &\gets \boldsymbol{x}_{\ell+}^{(k-1)\tau}\label{eq:xk(k-1)tau},
\end{align}
for each edge node $\ell$, where $\boldsymbol{y}_{[k],\ell}^{(k-1)\tau}$ and $\boldsymbol{x}_{[k],\ell}^{(k-1)\tau}$ are set as the virtual aggregated values right after the edge aggregation occurs. Then, we further conduct \emph{edge virtual update} as if model and momentum updates are conducted in the edge node. When $t\in ((k-1)\tau, k\tau]$, we conduct \emph{edge virtual update} as
\begin{align}
\boldsymbol{y}_{[k],\ell}^t &\gets \boldsymbol{x}_{[k],\ell}^{t-1} - \eta\nabla F_\ell(\boldsymbol{x}_{[k],\ell}^{t-1}),\label{eq:ykt}\\
\boldsymbol{x}_{[k],\ell}^t 
    &\gets \boldsymbol{y}_{[k],\ell}^t + \gamma(\boldsymbol{y}_{[k],\ell}^t - \boldsymbol{y}_{[k],\ell}^{t-1})\label{eq:xkt}.
\end{align}
We repeat \eqref{eq:yk(k-1)tau}--\eqref{eq:xkt} for each edge interval $[k]$ where $k=1,2,\ldots, K.$ Please note that only if $t=k\tau, k=1,\ldots,K$, $\boldsymbol{y}_{\ell-}^{t}$ and $\boldsymbol{x}_{\ell+}^{t}$ are computed. For ease of analysis, we define intermediate value $\boldsymbol{x}_{\ell-}^t= \sum_{i=1}^{ C_{\ell}}\frac{D_{i,\ell}}{D_\ell}\boldsymbol{x}_{i,\ell}^t$ and $\boldsymbol{y}_{\ell-}^t= \sum_{i=1}^{ C_{\ell}}\frac{D_{i,\ell}}{D_{\ell}}\boldsymbol{y}_{i,\ell}^t$ that are meaningful at any iteration $t$.

Same as edge intervals, for each cloud interval $\{p\}$ where $p=1,2,\ldots,P$, the  \emph{cloud virtual update} is also conducted:
\begin{align}
    \boldsymbol{y}_{\{p\}}^{(p-1)\tau\pi} &\gets \boldsymbol{y}^{(p-1)\tau\pi}\label{eq:yk(k-1)tau_c},\\
    \boldsymbol{x}_{\{p\}}^{(p-1)\tau\pi} &\gets \boldsymbol{x}^{(p-1)\tau\pi}\label{eq:xk(k-1)tau_c},
\end{align}
when $t=(p-1)\tau\pi$, and
\begin{align}
\boldsymbol{y}_{\{p\}}^t &\gets \boldsymbol{x}_{\{p\}}^{t-1} - \eta\nabla F(\boldsymbol{x}_{\{p\}}^{t-1}),\label{eq:ykt_c}\\
\boldsymbol{x}_{\{p\}}^t 
    &\gets \boldsymbol{y}_{\{p\}}^t + \gamma(\boldsymbol{y}_{\{p\}}^t - \boldsymbol{y}_{\{p\}}^{t-1})\label{eq:xkt_c},
\end{align}
when $p \in ((p-1)\tau\pi,p\tau\pi]$.

By applying virtual updates on edge nodes and the cloud, we can bound the gap between real updates and these virtual updates that can be then used to prove the convergence. Since in HierMo, the momenta and the models are aggregated on both edge nodes and the cloud, it brings much more challenges to conduct convergence analysis. The virtual update is an important intermediate process for convergence analysis and is one of our contributions in this paper.

Fig.~\ref{fig:xt} illustrates the evolution of $\boldsymbol{x}_{\ell-}^t$, $\boldsymbol{x}_{\ell+}^t$, $\boldsymbol{x}_{[k],\ell}^t$, $\boldsymbol{x}_{\{p\}}^t$, and $\boldsymbol{x}_{}^t$ when $\tau=2,\pi=2$. There are 2 edge nodes and each edge node serves 2 workers (in total 4 workers in Fig.~\ref{fig:xt}). 
After every 2 local updates, there is an edge aggregation, and after every 2 edge aggregations (4 local updates), there is a cloud aggregation. Please note \tikztextcircle{1} $\boldsymbol{x}_{[k],\ell}^{k\tau}$ and $\boldsymbol{x}_{[k+1],\ell}^{k\tau}$ are different. $\boldsymbol{x}_{[k],\ell}^{k\tau}$ is calculated from $\boldsymbol{x}_{[k],\ell}^{(k-1)\tau}$ after $\tau$ edge virtual updates, while $\boldsymbol{x}_{[k+1],\ell}^{k\tau}$ is directly given by $\boldsymbol{x}_{\ell+}^{k\tau}$. \tikztextcircle{2} $\boldsymbol{x}_{\ell-}^{k\tau}$ and $\boldsymbol{x}_{\ell+}^{k\tau}$ are different. $\boldsymbol{x}_{\ell-}^{k\tau}$ is the intermediate value that is used for edge model/momentum update, while $\boldsymbol{x}_{\ell+}^{k\tau}$ is calculated from $\boldsymbol{x}_{\ell-}^{k\tau}$ during edge model/momentum update. \tikztextcircle{3} $\boldsymbol{x}_{\{p\}}^{(k+1)\tau}$ and $\boldsymbol{x}_{\{p+1\}}^{(k+1)\tau}$ are different. $\boldsymbol{x}_{\{p\}}^{(k+1)\tau}$ is calculated from $\boldsymbol{x}_{\{p\}}^{(k-1)\tau}$ after $\tau\cdot\pi$ cloud virtual updates, while $\boldsymbol{x}_{\{p+1\}}^{(k+1)\tau}$ is directly given by $\boldsymbol{x}_{}^{(k+1)\tau}$.

\subsection{Convergence Analysis}
In this section, we provide the convergence analysis of HierMo. 
In Theorem~\ref{theorem:xt-xkt}, we first focus on worker models under each edge node $\ell$ to bound the distance between edge intermediate value $\boldsymbol{x}_{\ell-}^t$  and edge virtual update $\boldsymbol{x}_{[k],\ell}^t$ within interval $[k]$.

\begin{theorem}\label{theorem:xt-xkt}
For any edge interval $[k]$, $\forall t \in ((k-1)\tau, k\tau]$ and $\forall \ell \in L$, we have
\begin{align}\label{ieq:xt-xkt}
    \Vert\boldsymbol{x}_{\ell-}^t-\boldsymbol{x}_{[k],\ell}^t \Vert \leq h(t-(k-1)\tau,\delta_\ell),
\end{align}
where $h(x,\delta_\ell)$ is
\begin{flalign} \label{eq:h(x)}
h(x,\delta_\ell)=& \eta \delta_\ell\left(I(\gamma A)^{x}+J(\gamma B)^{x}-\frac{1}{\eta \beta}\right.\nonumber\\
&\left.-\frac{\gamma^2(\gamma^{x}-1)-(\gamma-1) x}{(\gamma-1)^{2}}\right),
\end{flalign}
and $A, B, I,$ and $J$ are constants defined in Appendix~\ref{app_theorem1}, 
for $0<\gamma<1$ and any positive integer $x$.
\end{theorem}
Please note that when $t=(k-1)\tau$ for all $[k]$, we have $\Vert\boldsymbol{x}_{\ell-}^t-\boldsymbol{x}_{[k],\ell}^t \Vert =0 = h(0,\delta_\ell)$, which also satisfies~\eqref{eq:h(x)}. Also,  $F_\ell(\boldsymbol{x})$ is $\rho$-Lipschitz, so that we also have
\begin{align}\label{ieq:Fxt-Fxkt}
F_\ell(\boldsymbol{x}_{\ell-}^t)-F_\ell(\boldsymbol{x}_{[k],\ell}^t)\leq \rho h(t-(k-1)\tau,\delta_\ell).
\end{align}

\begin{proof} [Proof sketch]
We first obtain the worker momentum upper bound $\Vert\boldsymbol{y}_{i,\ell}^t-\boldsymbol{y}_{[k],\ell}^t\Vert$ for each worker $\{i,\ell\}$. Based on it and worker momentum update rules in Lines \ref{eq:yit}--\ref{eq:xit} in Algorithm~\ref{alg:HierMo}, we bound the worker model parameter gap $\Vert\boldsymbol{x}_{i,\ell}^t-\boldsymbol{x}_{[k],\ell}^t\Vert$. Then, we extend above two bounds to obtain edge aggregated worker momentum upper bound $\Vert\boldsymbol{y}_{\ell-}^t-\boldsymbol{y}_{[k],\ell}^t\Vert$. Finally, the gap of edge model parameter $\Vert\boldsymbol{x}_{\ell-}^t-\boldsymbol{x}_{[k],\ell}^t\Vert$ is obtained.
See Appendix~\ref{app_theorem1} for the complete proof.
\end{proof}



In Theorem~\ref{theorem:xgktau-xktau}, we then bound the edge momentum update between $\boldsymbol{x}_{\ell+}^{k\tau}$ and $\boldsymbol{x}_{\ell-}^{k\tau}$ within interval $[k]$.

\begin{theorem}\label{theorem:xgktau-xktau}
For any edge interval $[k]$ in any edge node $\ell \in L$, suppose $0<\gamma<1, 0<\gamma_a<1$, and any $\tau = 1,2,\dots$, we have
\begin{align}\label{ieq:xgktau-xktau}
\Vert\boldsymbol{x}_{\ell+}^{k\tau}-\boldsymbol{x}_{\ell-}^{k\tau}\Vert\leq s(\tau),
\end{align}
where $s(\tau)$ is
\begin{align}\label{eq:s(tau)}
    s(\tau) =&\gamma_a\tau\eta\rho(\gamma\mu+\gamma+1)
\end{align}
and constant $\mu$ is defined in Appendix~\ref{app_theorem2_new}.
\end{theorem}
\begin{proof} [Proof sketch]
Based on the edge momentum update rules in Lines \ref{eq:ygktau}--\ref{eq:xgktau} in Algorithm~\ref{alg:HierMo}, we can derive $\boldsymbol{x}_{\ell+}^{k\tau}-\boldsymbol{x}_{\ell-}^{k\tau}=\gamma_a\left(\boldsymbol{x}_{\ell-}^{k\tau}-\boldsymbol{x}_{\ell-}^{(k-1)\tau}\right)
=\gamma_a\sum_{t=(k-1)\tau}^{k\tau-1}\left(\boldsymbol{x}_{\ell-}^{t+1}-\boldsymbol{x}_{\ell-}^t\right)$. Then we prove the bound of $\Vert\boldsymbol{x}_{\ell-}^{t+1}-\boldsymbol{x}_{\ell-}^t\Vert$ based on the definition of intermediate value where $\boldsymbol{x}_{\ell-}^t= \sum_{i=1}^{ C_{\ell}}\frac{D_{i,\ell}}{D_\ell}\boldsymbol{x}_{i,\ell}^t$, and then the result is obtained.
See Appendix~\ref{app_theorem2_new} for the complete proof.
\end{proof}

By combining the results of Theorem~\ref{theorem:xt-xkt} and Theorem~\ref{theorem:xgktau-xktau}, we can telescope the bound within edge interval $[k]$ to the cloud interval $\{p\}$ where $k=(p-1)\pi+1,(p-1)\pi+2,\ldots,p\pi$. Then, we are ready to bound the gap between weighted average of edge virtual update $\sum_{\ell=1}^{L}\frac{D_\ell}{D}\boldsymbol{x}_{[p\pi],\ell}^{p\tau\pi}$ and cloud virtual update $\boldsymbol{x}_{\{p\}}^{p\tau\pi}$ in Theorem~\ref{theorem:x_ppi_ptaupi-x_p_ptaupi}.

\begin{theorem}\label{theorem:x_ppi_ptaupi-x_p_ptaupi}
For any cloud interval $\{p\}, 0<\gamma<1$, and $0<\gamma_a<1$, 
when edge interval $[k]=[p\pi]$ (the last edge interval in cloud interval $\{p\}$), and $\forall\tau, \pi \in\{ 1,2,\ldots\}$ 
 we have
\begin{align}
  \Vert\boldsymbol{x}_{[p\pi]}^{p\tau\pi}-\boldsymbol{x}_{\{p\}}^{p\tau\pi}\Vert\leq h(\tau\pi,\delta)+\pi\sum_{\ell=1}^L\frac{D_\ell}{D}\left(h(\tau,\delta_\ell)+s(\tau)\right) , 
\end{align}
where we define $\boldsymbol{x}_{[p\pi]}^{p\tau\pi}\triangleq \sum_{\ell=1}^{L}\frac{D_\ell}{D}\boldsymbol{x}_{[p\pi],\ell}^{p\tau\pi}$, for $\forall \ell \in L$.
\end{theorem}
\begin{proof} [Proof sketch]
We propose an intermediate sequence of edge virtual update  on the cloud $\boldsymbol{x}_{\{p\},\ell}^{p\tau\pi}$. We then bound $\Vert\boldsymbol{x}_{[p\pi]}^{p\tau\pi}-\boldsymbol{x}_{\{p\},\ell}^{p\tau\pi}\Vert$ and $\Vert\boldsymbol{x}_{\{p\},\ell}^{p\tau\pi}-\boldsymbol{x}_{\{p\}}^{p\tau\pi}\Vert$ respectively to obtain the final result. See Appendix~\ref{app_theorem3} for complete proof.
\end{proof}


\begin{theorem} \label{theorem:Fwt-Fw*}
Under the following conditions: (1) $0<\beta\eta(\gamma+1)\leq 1$, $0<\gamma<1$, $0<\gamma_{a}<1$, and $\forall\tau, \pi \in\{ 1,2,\ldots\}$; (2) $\exists \varepsilon >0$,  (2.1) $\omega\alpha\sigma^2-\frac{\rho j(\tau,\pi,\delta_\ell,\delta) }{\tau\pi \varepsilon^{2}}>0$; (2.2) $F(\boldsymbol{x}_{\{p\}}^{p \tau\pi})-F\left(\boldsymbol{x}^{*}\right) \geq \varepsilon, \forall p$; and (2.3) $F(\boldsymbol{x}_{}^T)-F(\boldsymbol{x}^{*}) \geq \varepsilon$ are satisfied;  Algorithm 1 gives
\begin{align} \label{ieq:FwT-Fw*}
F(\boldsymbol{x}_{}^T)-F(\boldsymbol{x}^{*}) \leq \frac{1}{T\left(\omega\alpha\sigma^2-\frac{\rho j(\tau,\pi,\delta_\ell,\delta)  }{\tau\pi \varepsilon^{2}}\right)}.
\end{align}
where $j(\tau,\pi,\delta_\ell,\delta)$ is
\begin{align}\label{eq:j(tau,pi)}
j(\tau,\pi,\delta_\ell,\delta)= h(\tau\pi,\delta)+(\pi+1)\sum_{\ell=1}^{L}\frac{D_\ell}{D}\left((h(\tau,\delta_\ell)+s(\tau))\right).
\end{align}
We define $F(\boldsymbol{x}^{*})$ as the minimum value, if there exists some $\varphi>0$ such that $F(\boldsymbol{x}^{*})\leq F(\boldsymbol{x})$ for all $\boldsymbol{x}$ within distance $\varphi$ of $\boldsymbol{x}^*$. 
Constant $\mu$ is defined in Appendix~\ref{app_theorem2_new} and constants $\omega, \sigma$, and $\alpha$ are defined in Appendix~\ref{app_theorem4}.
\end{theorem}
\begin{proof} [Proof sketch]
We first analyze the convergence of $F(\boldsymbol{x}_{\{p\}}^{t+1})-F(\boldsymbol{x}_{\{p\}}^t)$ within cloud interval $\{p\}$ when $t\in [(p-1)\tau\pi, p\tau\pi)$. Then, we merge $h(\tau,\delta_\ell)$, $s(\tau)$, and result of Theorem~\ref{theorem:x_ppi_ptaupi-x_p_ptaupi} to handle the overall effects and telescope the gap of overall effects to all $P$ cloud intervals, and then the final result is obtained. See Appendix~\ref{app_theorem4} for complete proof.
\end{proof}

{ Please note in the proof of Theorems~\ref{theorem:xgktau-xktau}, \ref{theorem:x_ppi_ptaupi-x_p_ptaupi}, and \ref{theorem:Fwt-Fw*}, we have characterized the multi-time cross-two-tier momentum interaction and cross-three-tier momentum interaction brought by the three-tier FL. To analyze $\pi$ times cross-two-tier momentum interactions, we devise a new telescope form to bound these new deviations (Equations \eqref{eq:app_xgktau-xktau}--\eqref{eq:app_xl+ktau-xl-ktau} and \eqref{ieq:theorem3_2}). To analyze cross-three-tier momentum interaction, we devise a new mechanism to analyze such momentum interactions across multi-tiers (Equations \eqref{eq:yk(k-1)tau_e}--\eqref{ieq:theorem3_2} and \eqref{ieq:cKT-c10}--\eqref{ieq:frac:FwT-cKT}). 

}

We have demonstrated that the gap between the global loss function value $F(\boldsymbol{x}_{}^T)$ and the stationary point $F(\boldsymbol{x}^*)$  is upper bounded by a function of $T$ ($T=K\tau=P\tau\pi$) which is inversely proportional to $T$. It converges with the convergence rate $\mathcal{O}\left(\frac{1}{T}\right)$ for smooth non-convex problems under non-i.i.d. data distribution. 
We also give the following observations based on the above theorems.

\begin{observation}
\textup{The overall gap in Theorem~\ref{theorem:Fwt-Fw*},  $F(\boldsymbol{x}_{}^T)-F(\boldsymbol{x}^*)$ decreases when $T$ is larger. From Appendix~\ref{app_monotone}, we have $h(x)\geq0$ for any $x=1,2,\dots,$ and it increases with $x$. According to~\eqref{eq:s(tau)}, $s(\tau)$ increases with $\tau$. According to \eqref{eq:j(tau,pi)}, $j(\tau, \pi)$ increases with $\tau$ and $\pi$. Therefore, the value of $\frac{\rho j(\tau,\pi)}{\tau\pi\varepsilon^2}$ increases with $\tau$ and $\pi$ so as to increase the overall bound $F(\boldsymbol{x}_{}^T)-F(\boldsymbol{x}^*)$. However, in order to let the Condition~(2.1) in Theorem~\ref{theorem:Fwt-Fw*} hold, we cannot set a very large $\tau$ and $\pi$, implying that convergence is guaranteed when $j(\tau,\pi)$ is below a certain threshold. 
Experiments on the effects of $\tau$ and $\pi$ further verify that larger $\tau$ and $\pi$ decreases the convergence performance. }
\end{observation}


In Theorem~\ref{theorem:Fwf-Fw*}, we further eliminate the value $\varepsilon$ in Theorem~\ref{theorem:Fwt-Fw*} and further demonstrate the bound between the final loss function value that the algorithm can obtain $F(\boldsymbol{x}^{\mathrm{f}})$ and the stationary point $F(\boldsymbol{x}^*)$, where we define 
\begin{align} \label{eq:wf}
\boldsymbol{x}^{\mathrm{f}} \triangleq \underset{\boldsymbol{x} \in\{\boldsymbol{x}_{(a)}^{p\tau\pi}: p=1,2, \ldots, P\}}{\arg \min } F(\boldsymbol{x}).
\end{align}

\begin{theorem} \label{theorem:Fwf-Fw*}
Under the following condition:  $0<\beta\eta(\gamma+1)\leq 1$, $0<\gamma< 1$, $0<\gamma_{a}<1$, and $\forall\tau, \pi \in\{ 1,2,\ldots\}$, we have 
\begin{align} \label{ieq:Fwf-Fw*}
F(\boldsymbol{x}^{\mathrm{f}})&-F\left(\boldsymbol{x}^{*}\right)
\leq \frac{1}{2 T \omega \alpha\sigma^2}+\rho j(\tau,\pi,\delta_\ell,\delta)\\ \nonumber&+\sqrt{\frac{1}{4 T^{2} \omega^{2} \alpha^{2} \sigma^4}+\frac{\rho j(\tau,\pi,\delta_\ell,\delta)}{\omega \alpha \sigma^2\tau\pi}}\triangleq f_{HierMo}(T).
\end{align}
\end{theorem}
\begin{proof}
See Appendix~\ref{app_theorem5} for complete proof.
\end{proof}

Theorem~\ref{theorem:Fwf-Fw*} will be used in {Section \ref{sec:compare_mom}} and Section \ref{sec:optmization} to { compare the convergence upper bound and} formulate the optimization problem respectively. 

{

\subsection{Comparison between HierMo and HierFAVG}\label{sec:compare_mom}
In this section, we theoretically quantify the performance gain brought by HierMo compared with HierFAVG (without momentum). The convergence upper bound of HierFAVG can be derived from \cite{liu2020HierFAVG} as follows:
\begin{align} \label{f_FAVG}
F(\hat{\boldsymbol{x}}^{\mathrm{f}})&-F\left(\boldsymbol{x}^{*}\right)
\leq \frac{1}{2 T \omega \hat{\alpha}\sigma^2}+\rho \hat{j}(\tau,\pi,\delta_\ell,\delta)\\ \nonumber&+\sqrt{\frac{1}{4 T^{2} \omega^{2} \hat{\alpha}^{2} \sigma^4}+\frac{\rho \hat{j}(\tau,\pi,\delta_\ell,\delta)}{\omega \hat{\alpha} \sigma^2\tau\pi}}\triangleq f_{HierFAVG}(T).
\end{align}
The definitions of $\hat{\alpha}$ and $\hat{j}(\cdot)$ can be found in \cite{liu2020HierFAVG}.




To prevent the gradient descent from overshooting \cite{goodfellow2016deep}, it is common to choose a very small $\eta$. The following theorem is made when $\eta\to 0^+$.

\begin{theorem} \label{theorem:fast}
When $0<\beta\eta(\gamma+1)\leq 1$, $0<\gamma< 1$, $0<\gamma_{a}<1$, and $\forall\tau, \pi \in\{ 1,2,\ldots\}$, HierMo outperforms HierFAVG, i.e.,
\begin{align*}
    f_{HierFAVG}(T) - f_{HierMo}(T) > 0 
\end{align*}
for any $T$ and  $\eta\to 0^+$.
\end{theorem}
\begin{proof}
See Appendix \ref{app_theorem6} for detailed proof.
\end{proof}
The above theorem indicates that HierMo leads to a tighter convergence upper bound compared with HierFAVG, showing that HierMo theoretically outperforms HierFAVG. 

}

\section{Aggregation Period Optimization by HierOPT}\label{sec:optmization}
We have proved that HierMo is convergent in section~\ref{sec:theory}. We observe that the worker-edge and edge-cloud aggregation periods $\tau$ and $\pi$ are two key design variables that will influence the convergence performance. The values of $\tau$ and $\pi$ will also influence the usage of communication and computation resources in the real-world  training process. Therefore, we aim to optimize these two variables and formulate an optimization problem: Under a given total training time denoted as $\Psi$, how the HierMo algorithm achieves the best performance (min global model loss). 

We denote the worker computation delay for one iteration as $\Theta_{w}$, edge computation delay for one edge aggregation as $\Theta_{e}$, and cloud computation delay for one cloud aggregation as $\Theta_{c}$. 
We also denote the 
worker communication delay to the edge as $\Phi_{w2e}$ and 
edge communication delay to the cloud as $\Phi_{e2c}$. All the above values are assumed to be given as they can be measured in the real world. 
We assume each worker $\{i,\ell\}$ communicates with connected edge node $\ell$ in parallel and each edge node $\ell$ communicates with cloud in parallel~\cite{liu2020HierFAVG,dinh2020proxVR,liu2021hierQSGD}. The above assumptions are commonly adopted in the literature~\cite{wang2019adaptive,dinh2020proxVR}. As a result, the total training time for HierMo is calculated as follows
\begin{align}
    \Psi\triangleq P\cdot\left(\tau\pi \Theta_{w}+\pi \Theta_{e}+ \Theta_{c} + \pi \Phi_{w2e} + \Phi_{e2c}\right),
\end{align}
where $P$ is the total number of cloud aggregations ($P=\frac{T}{\tau\pi}$). 

In order to find the optimal pair of $(\tau, \pi)$, we target to minimize \eqref{ieq:Fwf-Fw*}, where \eqref{ieq:Fwf-Fw*} demonstrates the bound between the global loss and the stationary point~\cite{wang2019adaptive,liu2020HierFAVG}. By incorporating the constraints, the optimization problem can be formulated as follows
\begin{align}
&\min _{\tau, \pi}  \frac{1}{2 T \omega \alpha\sigma^2}+\rho j(\tau,\pi,\delta_\ell,\delta)\label{opt:obj}\\
&\quad\quad\quad\quad\quad\quad\quad+\sqrt{\frac{1}{4 T^{2} \omega^{2} \alpha^{2} \sigma^4}+\frac{\rho j(\tau,\pi,\delta_\ell,\delta)}{\omega \alpha \sigma^2\tau\pi}},\nonumber\\
&\text { s.t. } P\cdot(\tau\pi \Theta_{w}+\pi \Theta_{e}+ \Theta_{c} + \pi \Phi_{w2e} + \Phi_{e2c})=\Psi, \tag{\ref{opt:obj}{a}}\label{opt:math_T}\\
&\quad\quad\quad\quad\quad\quad\quad\quad\quad\quad T=P\tau\pi\tag{\ref{opt:obj}{b}},\label{opt:T}\\
&\quad\quad\quad\quad\quad\quad\quad\quad\quad\quad\tau\geq1,\tag{\ref{opt:obj}{c}}\\
&\quad\quad\quad\quad\quad\quad\quad\quad\quad\quad\pi\geq1.\tag{\ref{opt:obj}{d}}
\end{align}
From constraints \eqref{opt:math_T} and \eqref{opt:T}, we obtain
\begin{align}
    \frac{1}{T}=\frac{\Theta_{e}+\Phi_{w2e}}{\Psi}\frac{1}{\tau}+\frac{\Theta_c+\Phi_{e2c}}{\Psi}\frac{1}{\tau\pi}+\frac{\Theta_{w}}{\Psi}\label{opt:1/T}.
\end{align}
Substituting \eqref{opt:1/T} into \eqref{opt:obj}, we can eliminate the equation constraints. We also define
\begin{align}
    q(\tau,\pi)&\triangleq \frac{1}{2T\omega\alpha\sigma^2}\\
    &=\frac{\Theta_{e}+\Phi_{w2e}}{2\Psi\omega \alpha\sigma^2}\frac{1}{\tau}+\frac{\Theta_c+\Phi_{e2c}}{2\Psi\omega \alpha\sigma^2}\frac{1}{\tau\pi}+\frac{\Theta_{w}}{2\Psi\omega \alpha\sigma^2}.\nonumber
\end{align}
\begin{algorithm}[b!t]\small
\caption{HierOPT algorithm.}
\label{alg:optimization}
\textbf{Input}: $\Psi, \Theta_{w}, \Theta_{e}, \Theta_{c}, \Phi_{w2e}, \Phi_{e2c}$\\
\textbf{Output}: $\tau^*$ and $\pi^*$
\begin{algorithmic}[1] 
\STATE Initialize $\tau_0$ and $\pi_0$ as random positive integers, $i=0$ as the index of search iteration.
\WHILE{$true$}
\STATE Calculate $\mathcal{R}^\prime(\tau_i)$\label{alg2:tau_start}
\IF{$\mathcal{R}^\prime(\tau_i)>0$ } 
    \STATE $\tau_{i+1}\gets\max\{\tau_i-1, 1\}$\label{alg2:tau_>0}
\ELSIF{$\mathcal{R}^\prime(\tau_i)<0$}
    \STATE $\tau_{i+1}\gets\tau_i+1$\label{alg2:tau_<0}
\ENDIF\label{alg2:tau_end}
\STATE Calculate $\mathcal{R}^\prime(\pi_i)$\label{alg2:pi_start}
\IF{$\mathcal{R}^\prime(\pi_i)>0$ }
    \STATE $\pi_{i+1}\gets\max\{\pi_i-1,1\}$\label{alg2:pi_>0}
\ELSIF{$\mathcal{R}^\prime(\pi_i)<0$}
    \STATE $\pi_{i+1}\gets\pi_i+1$\label{alg2:pi_<0}
\ENDIF\label{alg2:pi_end}
\STATE Record $(\tau_i,\pi_i)$
\IF{the pair of values ($\tau_i$, $\pi_i$) is visited before}\label{alg2:visit_start}
    \STATE Set $\tau^*\gets\tau_i$ and $\pi^*\gets\pi_i$
    \STATE \texttt{BREAK}
\ENDIF\label{alg2:visit_end}
\STATE $i\gets i+1$
\ENDWHILE
\end{algorithmic}
\end{algorithm}\noindent
The problem \eqref{opt:obj} can be re-formulated as
\begin{align}
\min _{\tau, \pi}  q(\tau,\pi)+\rho j(\tau,\pi,\delta_\ell,\delta)&+\sqrt{q^2(\tau,\pi)+\frac{\rho j(\tau,\pi,\delta_\ell,\delta)}{\omega \alpha \sigma^2\tau\pi}},\label{opt:obj_new}\\
\text {  s.t. }\tau&\geq1\tag{\ref{opt:obj_new}{a}},\label{opt:tau}\\
\pi&\geq1\tag{\ref{opt:obj_new}{b}}.\label{opt:pi}
\end{align}

It is non-trivial to find a closed-form optimal pair of $(\tau, \pi)$ in the three-tier hierarchical FL because problem \eqref{opt:obj_new} 
includes both polynomial and exponential terms of $\tau$ and $\pi$, where the exponential term is nest-embedded in $h(\cdot)$ that is embedded in $j(\cdot)$. Even if for  a two-tier FL problem, the objective function of the bound is complicated, and it is still infeasible to find an optimal solution in closed form~\cite{wang2019adaptive,dinh2020proxVR}. 
In what follows, we propose the Hierarchical Optimizing Periods (HierOPT) algorithm to find a local optimal solution to problem~\eqref{opt:obj_new}. 


In Algorithm~\ref{alg:optimization}, for convenience, we define the objective function \eqref{opt:obj_new} as $\mathcal{R}(\tau,\pi)$ with respect to $\tau$ and $\pi$. We also define the partial derivative of $\tau$ and $\pi$ as $\mathcal{R}^\prime(\tau)$ and $\mathcal{R}^\prime(\pi)$ respectively. Since $\mathcal{R}(\tau,\pi)$ is in  closed-form, $\mathcal{R}^\prime(\tau)$ and $\mathcal{R}^\prime(\pi)$ are also in  closed-form and can be calculated numerically given any $\pi$ and $\tau$ respectively. Algorithm \ref{alg:optimization} is  operated as follows: \tikztextcircle{1} We take turns to calculate $\mathcal{R}^\prime(\tau)$ (Lines~\ref{alg2:tau_start}--\ref{alg2:tau_end}) and $\mathcal{R}^\prime(\pi)$ (Lines~\ref{alg2:pi_start}--\ref{alg2:pi_end}). When the gradient is greater than zero, implying that the objective function has the trend to increase, we decrease the value by $1$ (Lines~\ref{alg2:tau_>0} and \ref{alg2:pi_>0}). When the gradient is less than zero, implying that the objective function has the trend to decrease, we increase the value by $1$ (Lines~\ref{alg2:tau_<0} and \ref{alg2:pi_<0}). Due to constraints \eqref{opt:tau} and \eqref{opt:pi}, we restrict the values of $\tau$ and $\pi$ to be equal or greater than $1$. \tikztextcircle{2} If the pair of value ($\tau$, $\pi$) is visited before (Lines \ref{alg2:visit_start}--\ref{alg2:visit_end}), it means  Algorithm~\ref{alg:optimization} converges and $(\tau,\pi)$ oscillates within a number of feasible value pairs (because $\tau$ and $\pi$ can only be integers). In this case, we find a local optimal pair of $(\tau^*, \pi^*)$ and we can exit the algorithm. 

\section{Experimental Results}\label{sec:exp}
In this section, we evaluate the convergence performance of HierMo compared with three typical categories of benchmark algorithms: \tikztextcircle{1} three-tier FL without momentum (HierFAVG~\cite{liu2020HierFAVG} and {CFL~\cite{wang2021CFL}}), \tikztextcircle{2} two-tier FL with momentum (DOMO~\cite{xu2021DOMO}, {FedADC~\cite{FedADC}}, FedMom~\cite{huo2020FedMom}, SlowMo~\cite{Wang2020SlowMo}, FedNAG~\cite{yang2020FedNAG}, and Mime~\cite{karimireddy2020mime}), and \tikztextcircle{3} two-tier FL without momentum (FedAvg~\cite{mcmahan2017FedAvg}).
For the two-tier benchmarks, we assume that the edge nodes do not exist and the workers are directly connected to the cloud. 
We then discuss the effects of $\tau$ and $\pi$ respectively and their joint effects. 
{Afterwards, we explicitly quantify different levels of non-i.i.d. data and analyze their effects.} 
Finally, we perform a trace-driven simulation of the 
three-tier hierarchical 
FL environment as if real-world hierarchical FL is implemented so that we can test the overall training time. Through this way, we verify that $(\tau^*, \pi^*)$ derived in Section~\ref{sec:optmization} leads to near-optimal performance in the realistic scenario. 


\begin{table*}[htb!]
  \caption{Performance comparison of different FL algorithms (accuracy \%).}
  \label{tab:exp_compare}
  \centering
  \begin{tabular}{p{0.2\columnwidth}p{0.2\columnwidth}p{0.2\columnwidth}p{0.2\columnwidth}p{0.2\columnwidth}p{0.2\columnwidth}p{0.2\columnwidth}p{0.2\columnwidth}}
  \toprule
               & Linear on \newline MNIST & Logistic on \newline MNIST & CNN on \newline MNIST & CNN on\newline CIFAR10 & VGG16 on\newline CIFAR10 & ResNet18 on\newline ImageNet & {CNN on \newline UCI-HAR}\\
    \midrule
    HierMo  & $\pmb{85.97}\pm0.03$ & $\pmb{89.23}\pm0.04$& $\pmb{96.13}\pm0.07$   & $\pmb{64.18}\pm0.08$  & $\pmb{90.06}\pm0.15$    & $\pmb{69.64}\pm0.12$ & {$\pmb{88.36}\pm0.06$}\\
    HierFAVG \cite{liu2020HierFAVG}    & $83.62\pm0.03$    & $87.00\pm0.05$    & $93.40\pm0.07$      & $38.46\pm0.13$      & $89.46\pm0.12$      & $68.63\pm0.10$ & {$54.56\pm0.11$}\\
    {CFL} \cite{wang2021CFL}         & {$83.36\pm0.04$}    & {$86.98\pm0.06$}    & {$93.58\pm0.06$}      & {$38.79\pm0.11$}      & {$89.80\pm0.11$}      & {$68.87\pm0.09$} & {$69.19\pm0.09$}\\    
    DOMO \cite{xu2021DOMO}        & $85.79\pm0.05$    & $89.02\pm0.05$    & $95.90\pm0.05$      & $59.39\pm0.07$      & $88.53\pm0.09$      & $67.05\pm0.10$ & {$88.15\pm0.06$}\\
    {FedADC} \cite{FedADC}     & {$85.51\pm0.04$}    & {$88.18\pm0.05$}    & {$95.09\pm0.07$}      & {$56.00\pm0.11$}      & {$89.38\pm0.08$}      & {$67.76\pm0.12$} & {$85.14\pm0.09$}\\
    FedMom \cite{huo2020FedMom}      & $84.84\pm0.06$    & $88.05\pm0.05$    & $94.74\pm0.05$      & $54.87\pm0.07$      & $88.03\pm0.10$      & $66.91\pm0.11$ & {$84.69\pm0.07$}\\
    SlowMo \cite{Wang2020SlowMo}     & $84.82\pm0.06$    & $88.00\pm0.06$    & $94.88\pm0.05$      & $54.43\pm0.06$      & $88.47\pm0.09$      & $66.84\pm0.09$ & {$83.03\pm0.10$}\\
    FedNAG \cite{yang2020FedNAG}     & $84.97\pm0.04$    & $88.14\pm0.05$    & $95.04\pm0.06$      & $55.54\pm0.09$      & $88.33\pm0.06$      & $66.81\pm0.14$ & {$84.69\pm0.06$}\\
    Mime \cite{karimireddy2020mime}       & $84.41\pm0.06$    & $87.73\pm0.06$    & $93.89\pm0.08$      & $48.24\pm0.15$      & $81.76\pm0.11$      & $64.33\pm0.21$ & {$76.75\pm0.11$}\\
    FedAvg \cite{mcmahan2017FedAvg}     & $83.57\pm0.04$    & $86.89\pm0.05$    & $93.31\pm0.08$      & $37.79\pm0.19$      & $88.27\pm0.15$      & $66.59\pm0.09$ & {$53.31\pm0.12$}\\
    \bottomrule
  \end{tabular}
\end{table*}

\begin{figure*}[htb!]
    \centering
    \begin{minipage}[t]{0.241\textwidth}
        \centering
        \includegraphics[width=\textwidth]{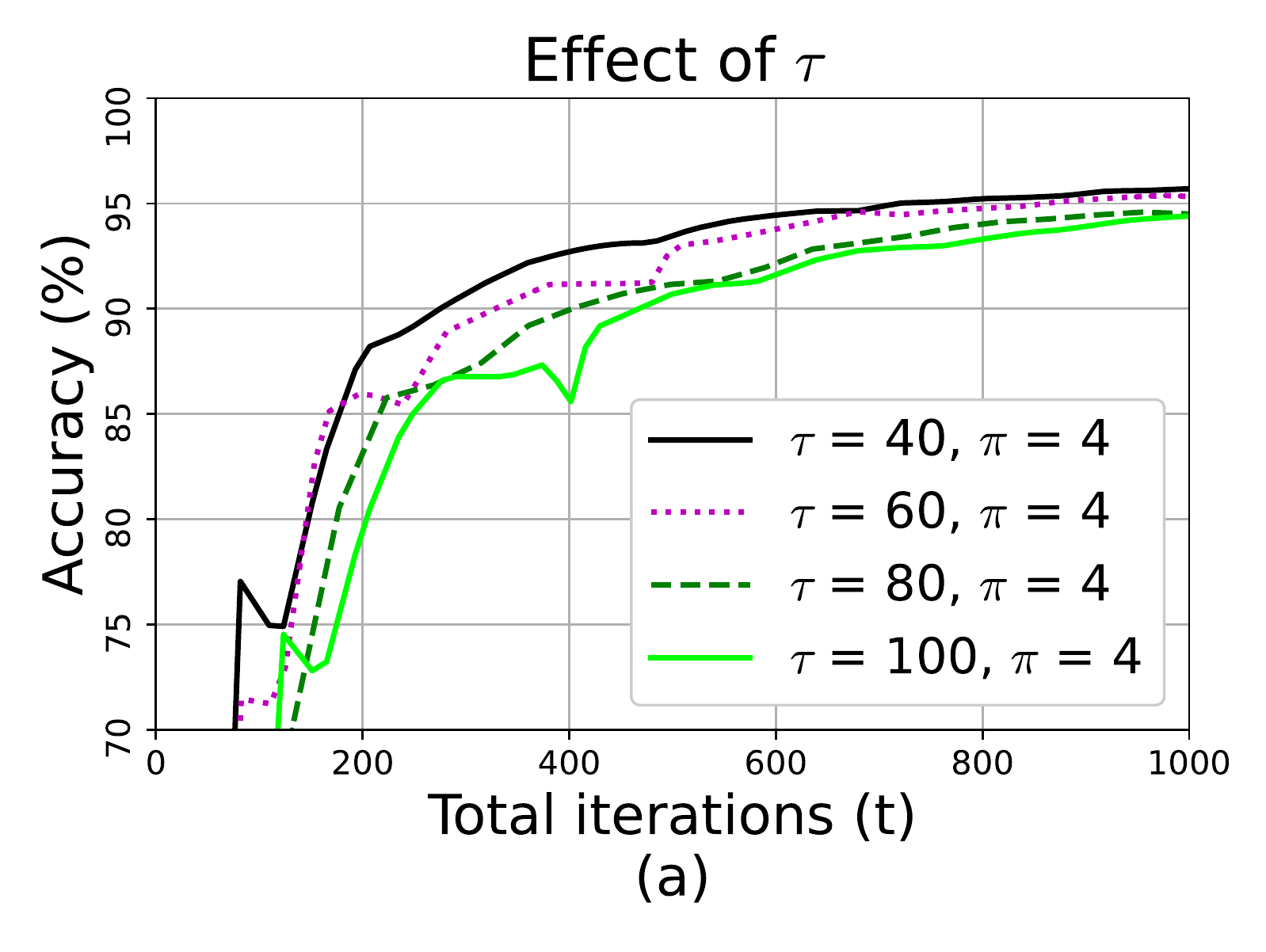}
    \end{minipage}
    \begin{minipage}[t]{0.241\textwidth}
        \centering
        \includegraphics[width=\textwidth]{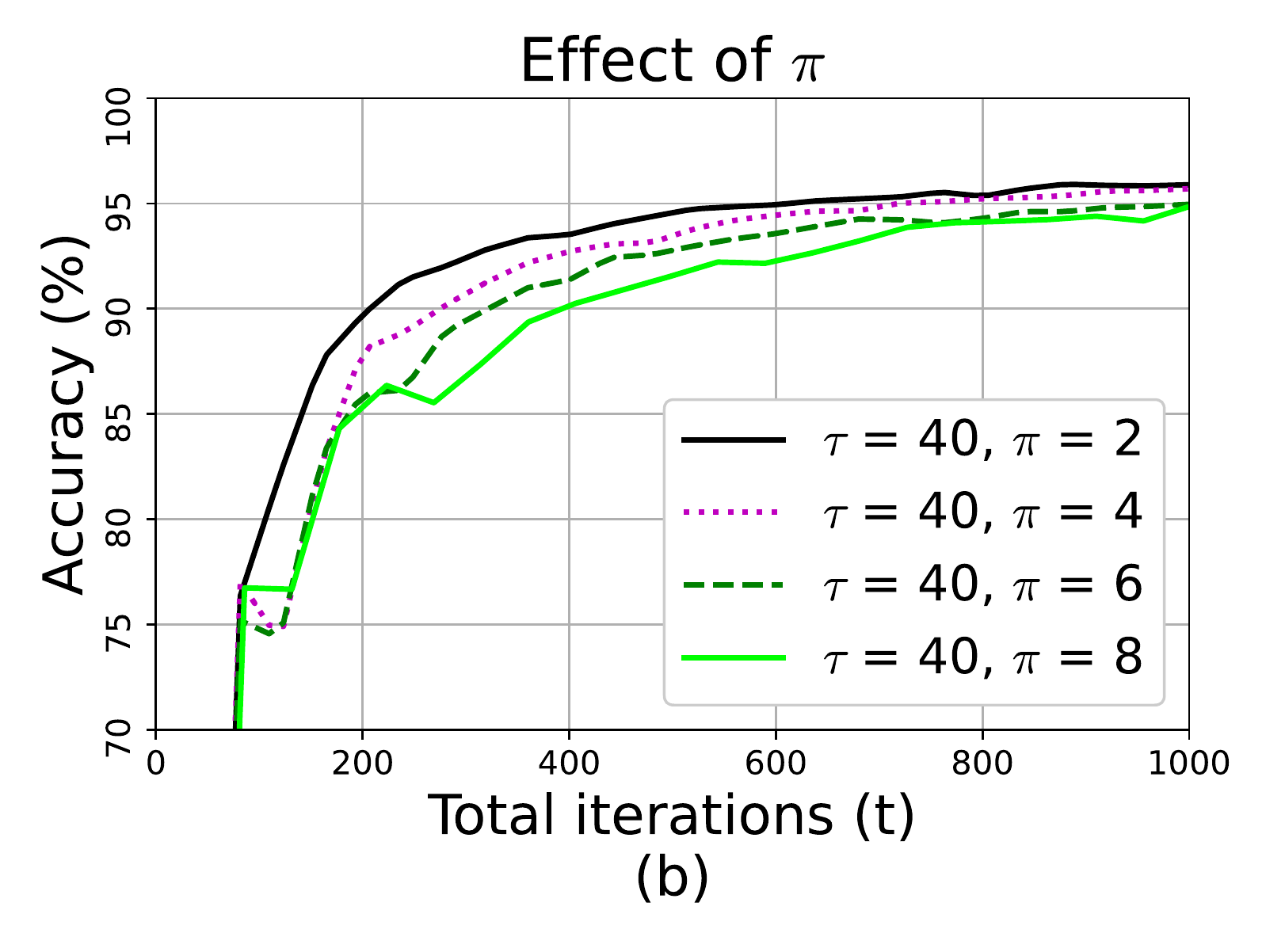}
    \end{minipage}
    \begin{minipage}[t]{0.241\textwidth}
        \centering
        \includegraphics[width=\textwidth]{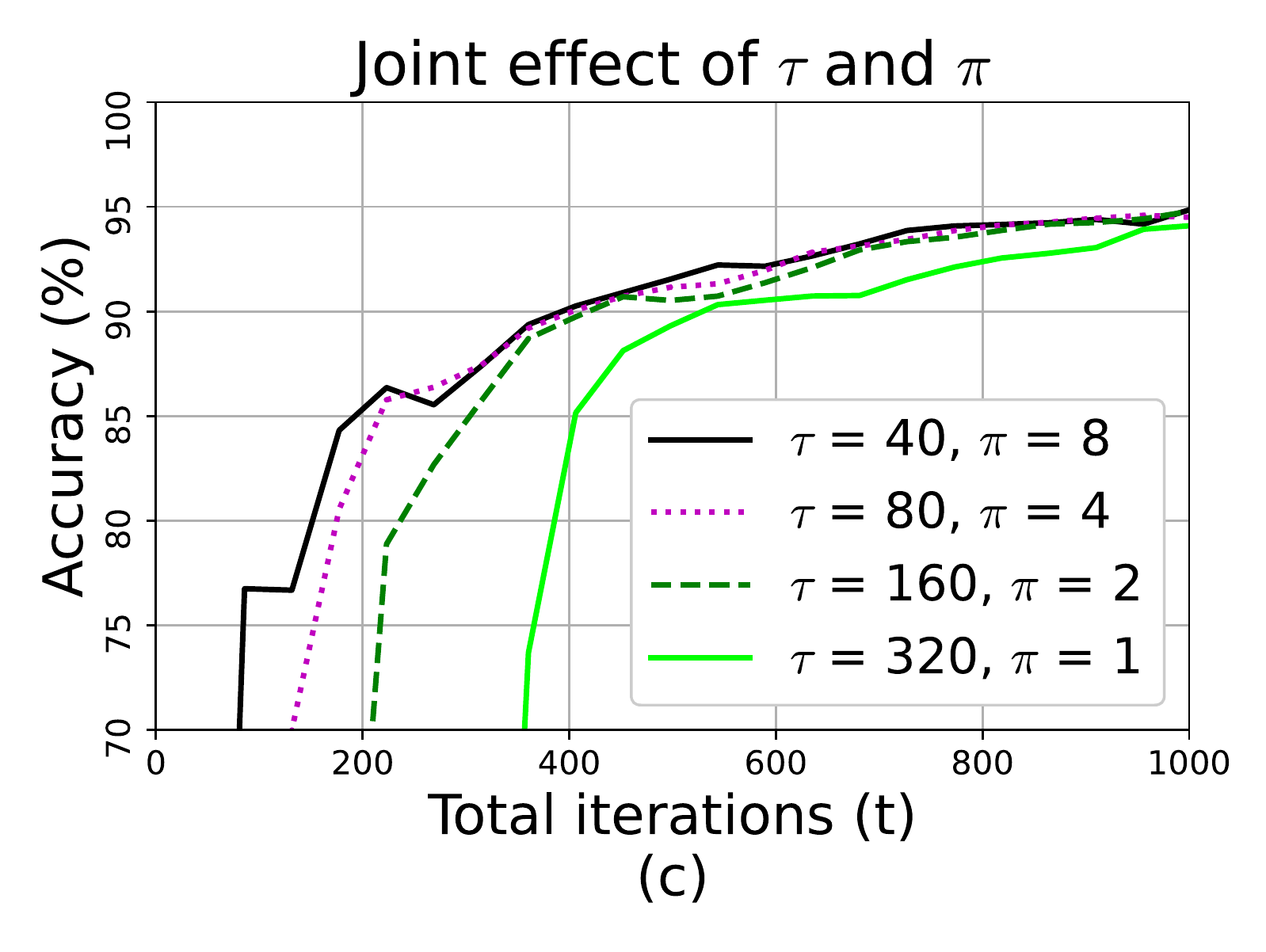}
    \end{minipage}
    \begin{minipage}[t]{0.241\textwidth}
        \centering
        \includegraphics[width=\textwidth]{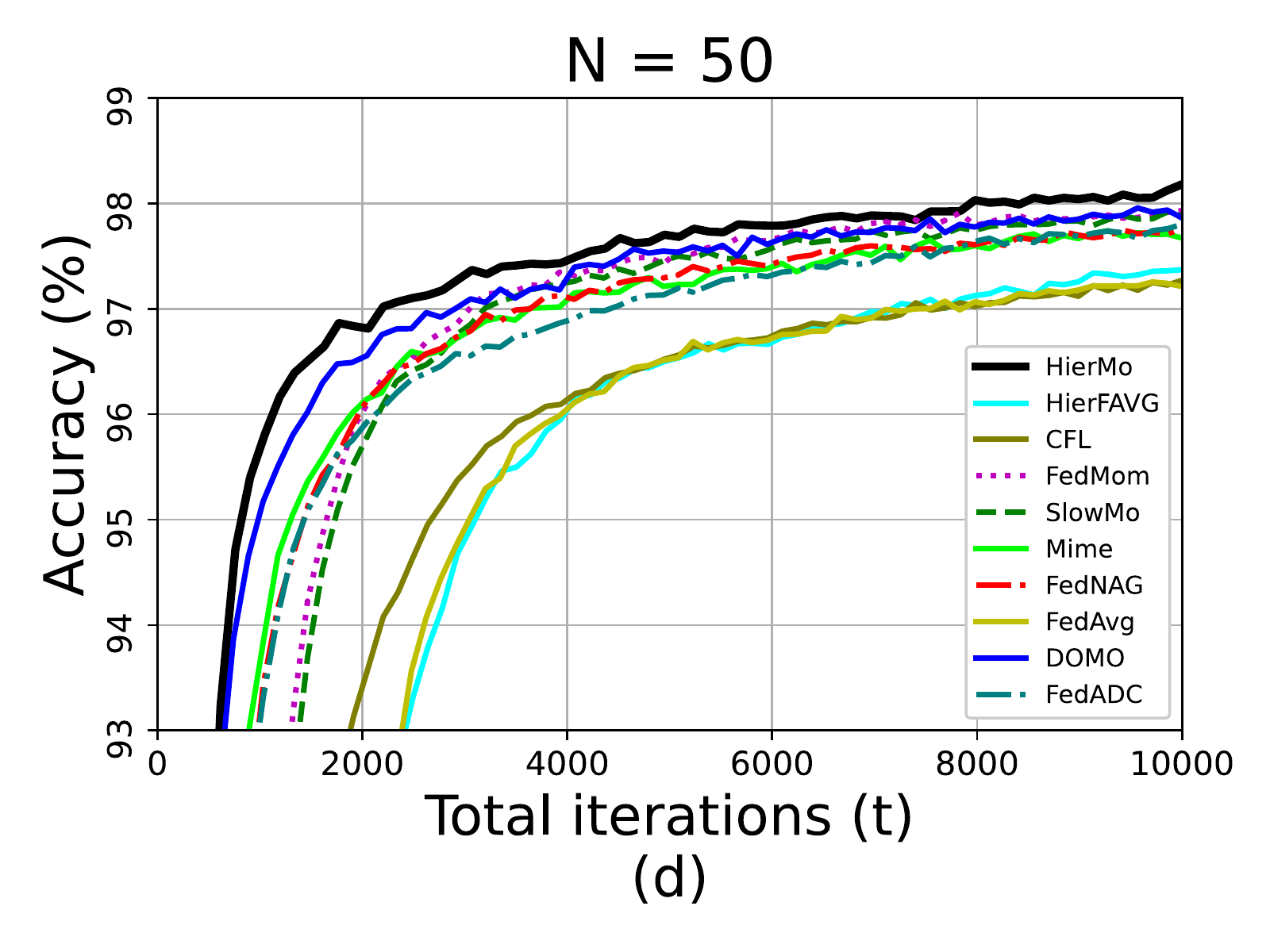}
    \end{minipage}
        \begin{minipage}[t]{0.241\textwidth}
        \centering
        \includegraphics[width=\textwidth]{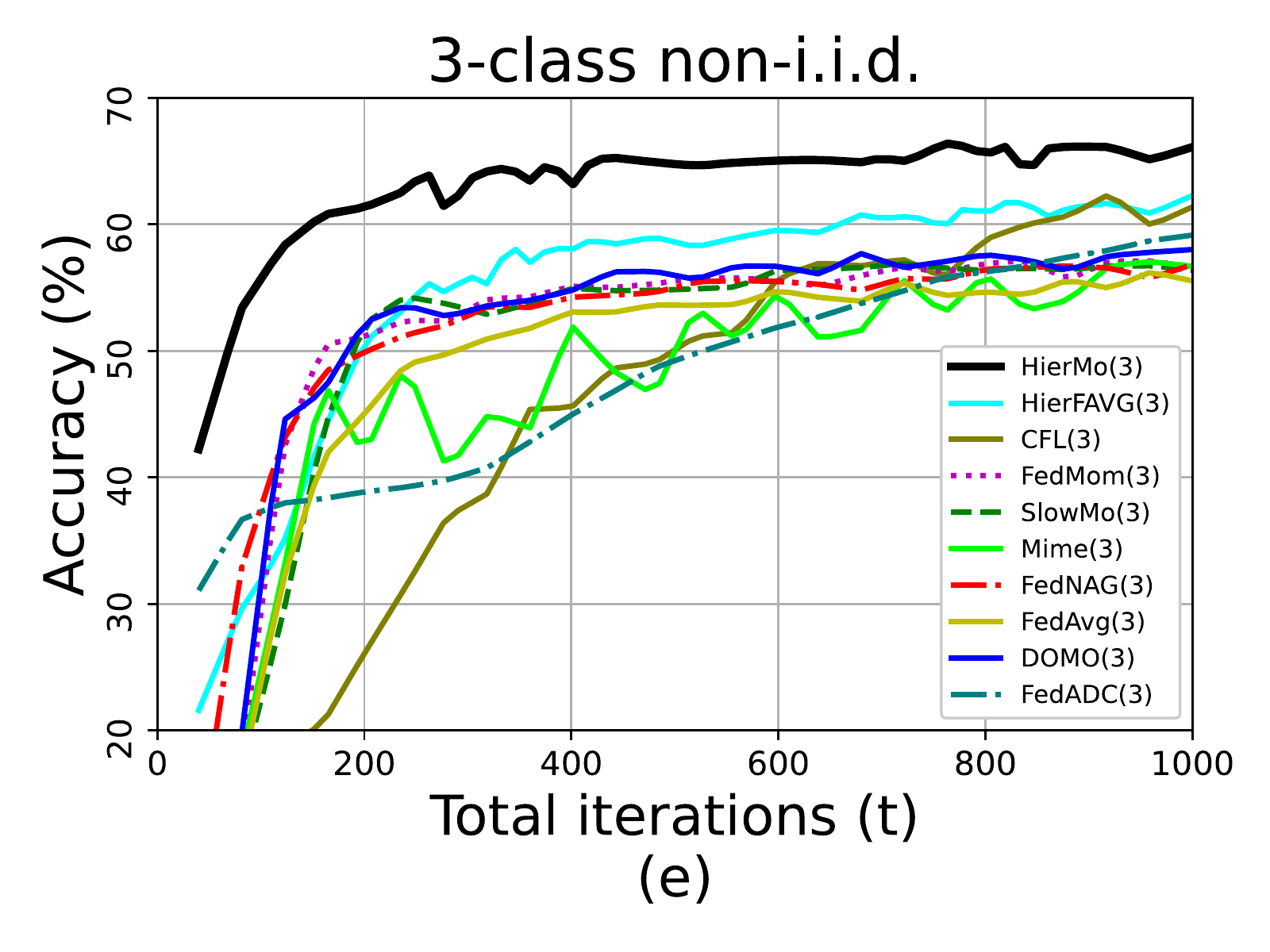}
    \end{minipage}
    \begin{minipage}[t]{0.241\textwidth}
        \centering
        \includegraphics[width=\textwidth]{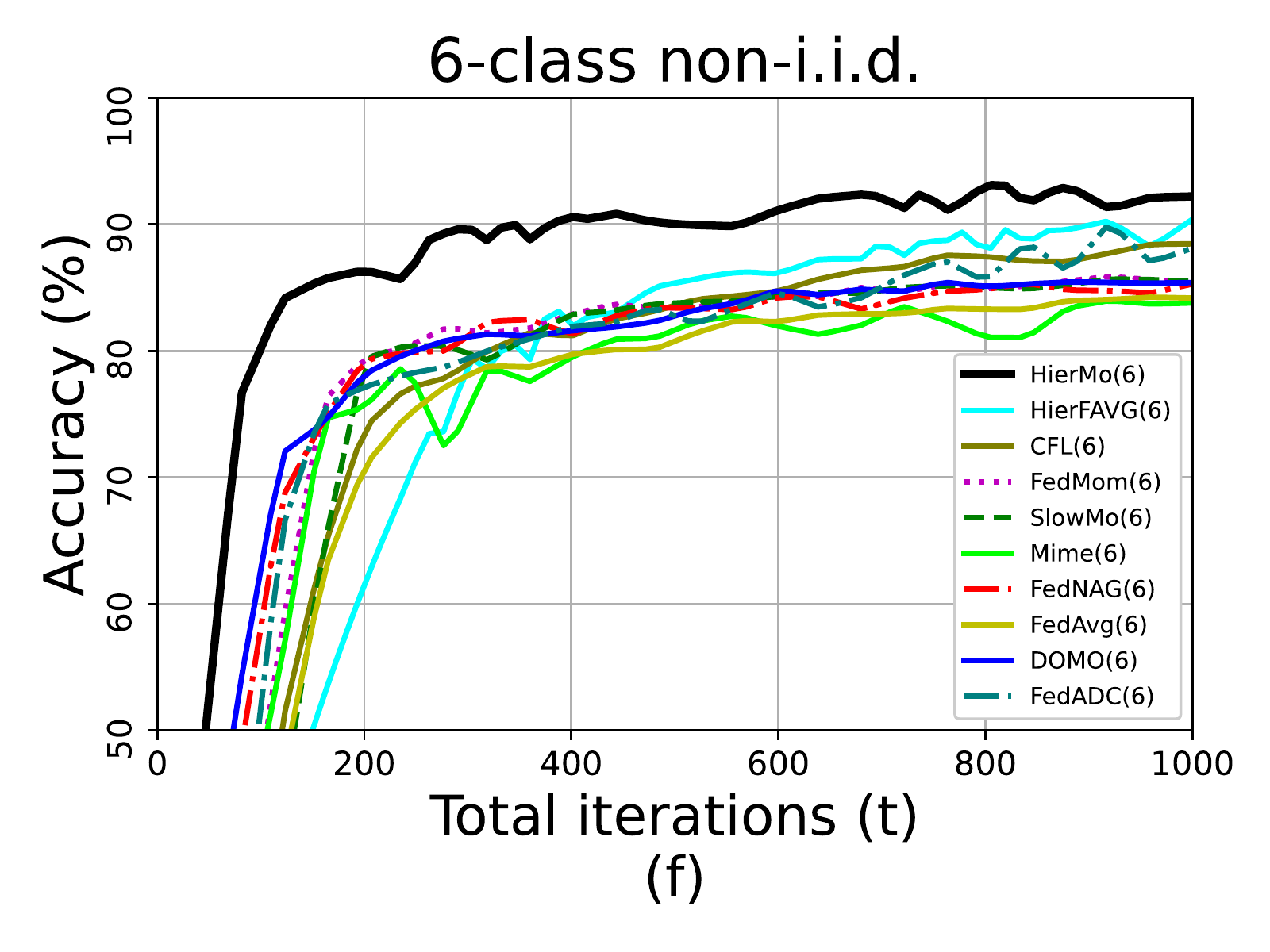}
    \end{minipage}
    \begin{minipage}[t]{0.241\textwidth}
        \centering
        \includegraphics[width=\textwidth]{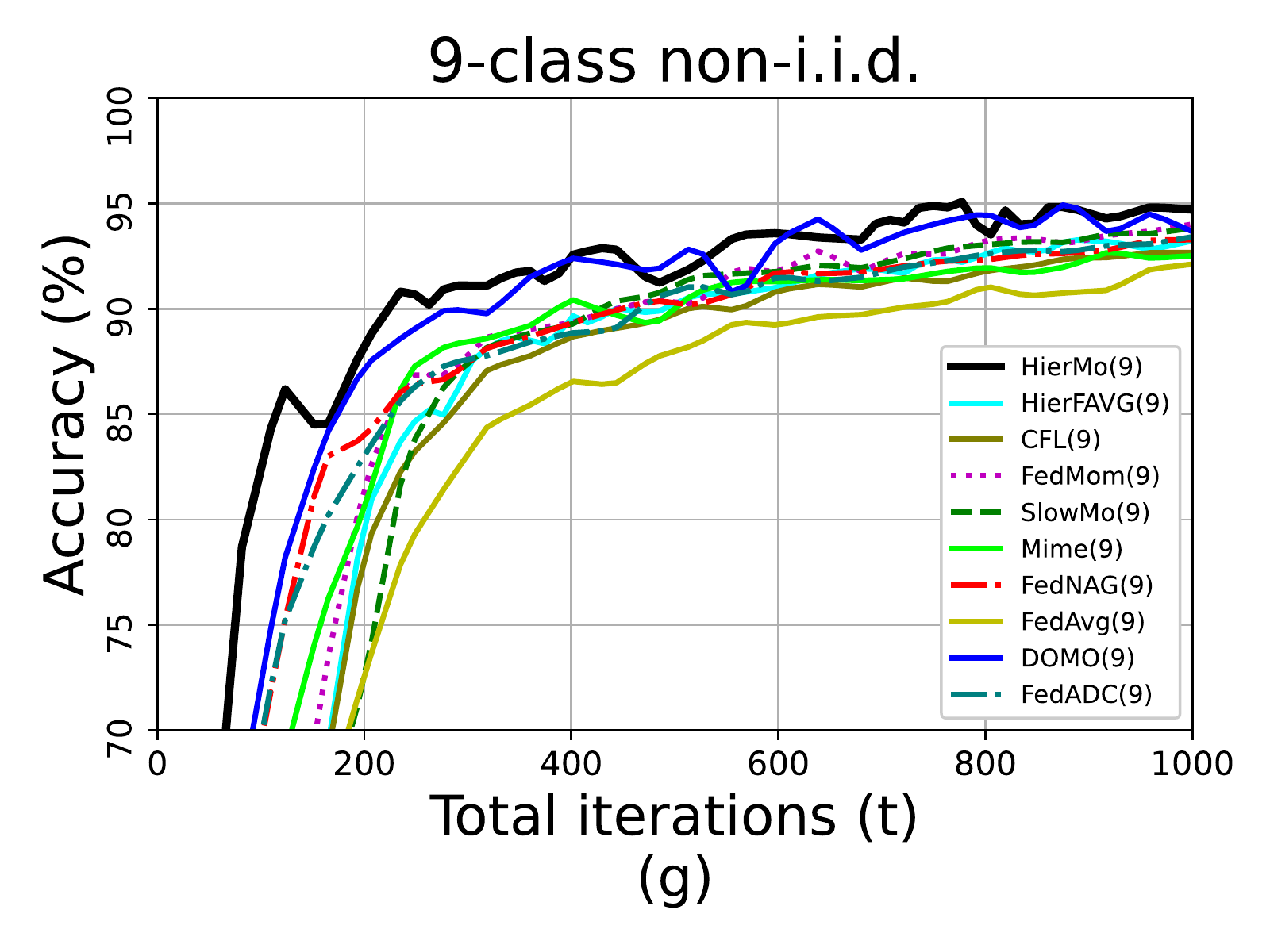}
    \end{minipage}
        \begin{minipage}[t]{0.241\textwidth}
        \centering
        \includegraphics[width=\textwidth]{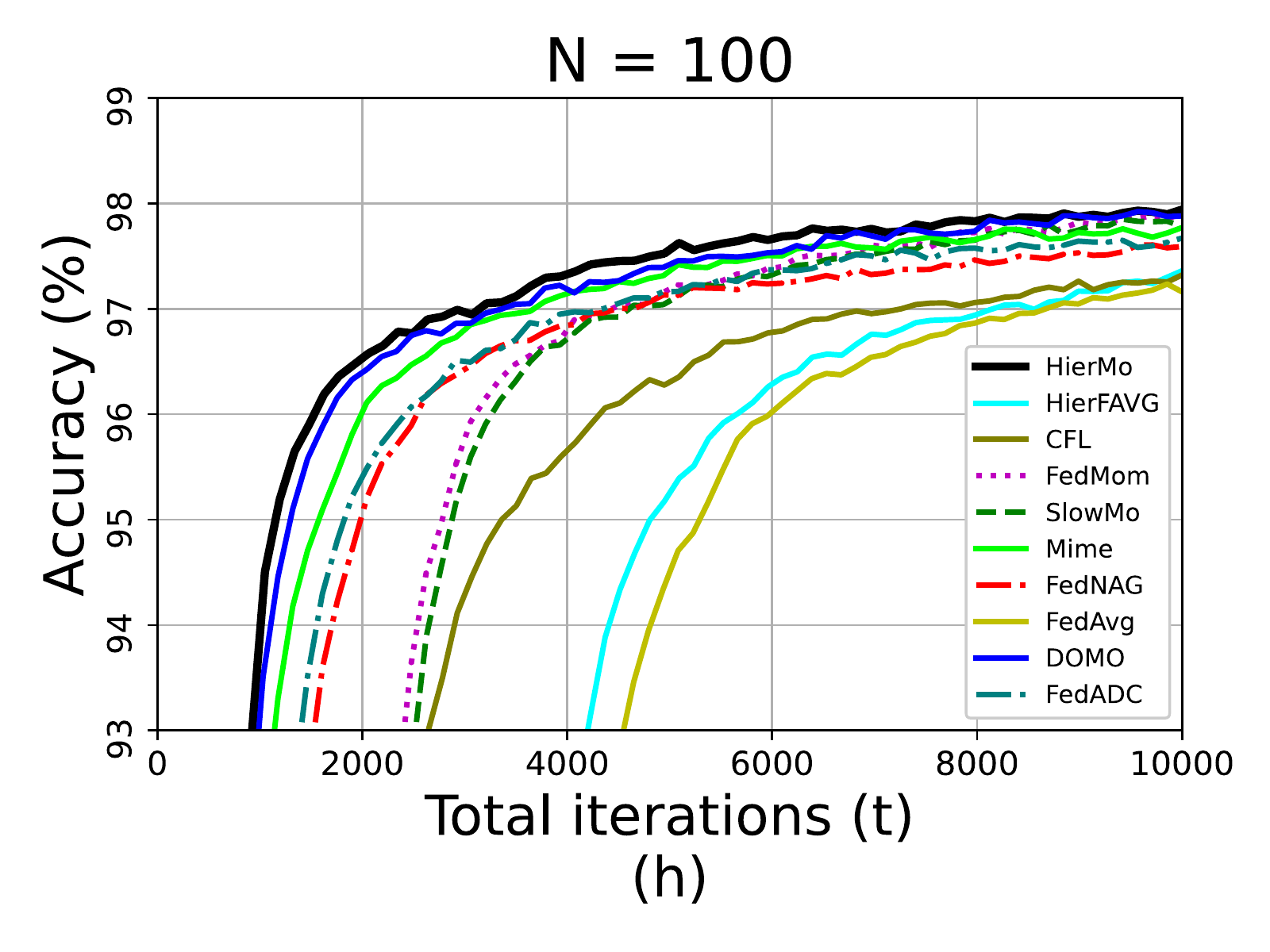}
    \end{minipage}
    \caption{(a)--(c): Accuracy comparison for HierMo under different settings of worker-edge aggregation period $\tau$ and edge-cloud aggregation period $\pi$ when CNN is trained on MNIST. {(e)--(g): Accuracy comparison under 3-class (e), 6-class (f), and 9-class (g) non-i.i.d. data when CNN is trained on MNIST. (d) and (h): Accuracy comparison for large $N$ ($N=50$ and $N=100$) when CNN is trained on MNIST.} }
    \label{fig:compare_tau_pi}
\end{figure*}

\subsection{Experiment on Convergence of HierMo}\label{sec:exp:performance}
\subsubsection{Experimental Setup}
We employ four real-world datasets including MNIST~\cite{mnist}, CIFAR-10~\cite{cifar10}, and ImageNet~\cite{deng2009imagenet,tinyimagenet} for image classification, and {UCI-HAR~\cite{anguita2013UCI-HAR} for human activity recognition.} 
All training and testing samples are randomly shuffled and distributed to workers. Please note there is no restriction on how the data is distributed at different workers, therefore, the level of non-i.i.d. data distribution captured by $\delta_{i,\ell}$ is different for each worker $\{i,\ell\}$. 
The training is run on a GPU tower server with 4 NVIDIA GeForce RTX 2080Ti GPUs. 

We use five models including linear regression, logistic regression, CNN, VGG16, and ResNet18. 
The CNN model's structure is the classic one in~\cite{cnnModel}, which has two $5\times5$ convolutional layers with 32 and 64 channels respectively. In each convolutional layer, $2\times2$ max pooling is used. The last three layers are fully connected layers with ReLu activation and softmax. The structure of VGG16 and ResNet18 can be found in~\cite{vgg16code,tinyimagenet} respectively. We use mini-batch in all experiments, and the batch size is 64. We set the learning rate $\eta=0.01$. Other hyper-parameters will be specified in each experiment. 

In this experiment, we focus on the convergence performance (i.e., accuracy given the number of iterations) of different algorithms. We do not consider the real-world delay for now. 
The results do not depend on hardware but on the algorithm itself. Therefore, we can create several virtual machines within a single server to carry out the experiment.  (Even if real-world hardware is used in the experiment, it will still give the same results.) The experiment on the optimization considering real-world delay will be discussed in Section~\ref{sec:exp_trace-driven}.

\subsubsection{Performance Comparison}
In Table~\ref{tab:exp_compare}, we compare the convergence performance of HierMo with benchmark algorithms. The numbers show the accuracy when different algorithms are run for $T$ iterations. The experiment is conducted on linear regression, logistic regression, CNN, VGG16, and ResNet18.
We set $T=1000$ (MNIST), $T=4000$ (UCI-HAR), or $T=10000$ (CIFAR10 and ImageNet), $\gamma=0.5, \gamma_{a}=0.5$. There are 4 workers and 2 edge nodes with each edge node serving 2 workers (three-tier algorithm). There are 4 workers directly served by the cloud (two-tier algorithm).
For two-tier algorithms, we set $\tau=20$ (convex model) or $\tau=40$ (non-convex model). For three-tier algorithms, we set $\tau=10, \pi=2$ (convex model) or $\tau=20,\pi=2$ (non-convex model).
 Please note that since $\pi$ does not exist for two-tier algorithms, we set $\tau$ value for two-tier algorithms equal to  $\tau\pi$ value for three-tier algorithms for a fair comparison.  
These hyper-parameters are typically used in existing works~\cite{liu2020accelerating,yu2019linear,yang2020FedNAG,liu2020HierFAVG,liu2021hierQSGD}.

In all cases, HierMo outperforms all other benchmarks. 
This confirms that applying momentum on both worker-level and edge-level with three-tier architecture achieves the best performance.

Comparing HierMo with HierFAVG and {CFL}, we observe that HierMo $>$ {CFL} $>$ HierFAVG . (We use ``$>$'' to indicate ``is better than'' for  presentation convenience.) This verifies that the momentum can accelerate the convergence in three-tier architecture.

Comparing HierMo with DOMO and {FedADC}, we observe that HierMo $>$ DOMO $>$ {FedADC}. This verifies that when two types of momentum are applied, the three-tier architecture outperforms the two-tier architecture. 
This is because the additional edge aggregation can decrease the effect of data heterogeneity among workers under the same edge node, so as to improve the performance. 

Comparing DOMO and {FedADC} with FedMom, SlowMo, FedNAG, and Mime, we observe that DOMO $>$ {FedADC} $>$ FedNAG $>$ FedMom $\approx$ SlowMo $>$ Mime. This confirms that using combined worker momentum and aggregator momentum can accelerate the convergence compared with those using momentum only on workers or only on the aggregator. For worker momentum only or aggregator momentum only algorithms, we can still observe their acceleration compared with FedAvg. 
We also observe Mime may not perform well. Sometimes, it is even worse than FedAvg. This is because Mime uses the fixed momentum value in worker momentum update, where such value can be refreshed only in the global aggregation phase. As a result, the momentum value may be stale, especially when $\tau$ is as large as $40$. 

Comparing HierFAVG and {CFL} with two-tier momentum-based algorithms (DOMO, {FedADC}, FedMom, SlowMo, FedNAG, and Mime), we observe that for DNN, HierFAVG and {CFL} outperform two-tier momentum-based algorithms, while for convex model and CNN, the later is better. This shows that for complicated models, the three-tier architecture plays a more significant role to accelerate the convergence while for less complicated models, 
the momentum plays a more significant role to accelerate the convergence. 

{We also compare the training accuracy when more workers ($N=50$ and $N=100$) participate the training to demonstrate the cross-siloed FL~\cite{kairouz2021advances} (typically up to one hundred participants). The results in Fig.~\ref{fig:compare_tau_pi}(d) and (h) show the same trend as results in Table~\ref{tab:exp_compare}.

}


\subsubsection{Effects of $\tau$ and $\pi$}
In Fig.~\ref{fig:compare_tau_pi}, we evaluate the effects of $\tau$ and $\pi$, and their joint effects. The curves in the figure show the accuracy when CNN is trained on MNIST. We set $T=1000, \gamma=0.5, \gamma_a=0.5$. There are 16 workers and 4 edge nodes with each edge node serving 4 workers.

When $\pi$ and $\tau$ are fixed in Fig.~\ref{fig:compare_tau_pi}(a) and Fig.~\ref{fig:compare_tau_pi}(b) respectively, we observe that larger $\tau$ or $\pi$  lowers the performance. This observation matches our expectation and verifies the result of Theorem~\ref{theorem:Fwt-Fw*} showing that the larger $\tau$ or $\pi$ leads to larger convergence upper bound.




\begin{figure}[tb!]
    \centering
    \begin{minipage}[t]{0.241\textwidth}
        \centering
        \includegraphics[width=\textwidth]{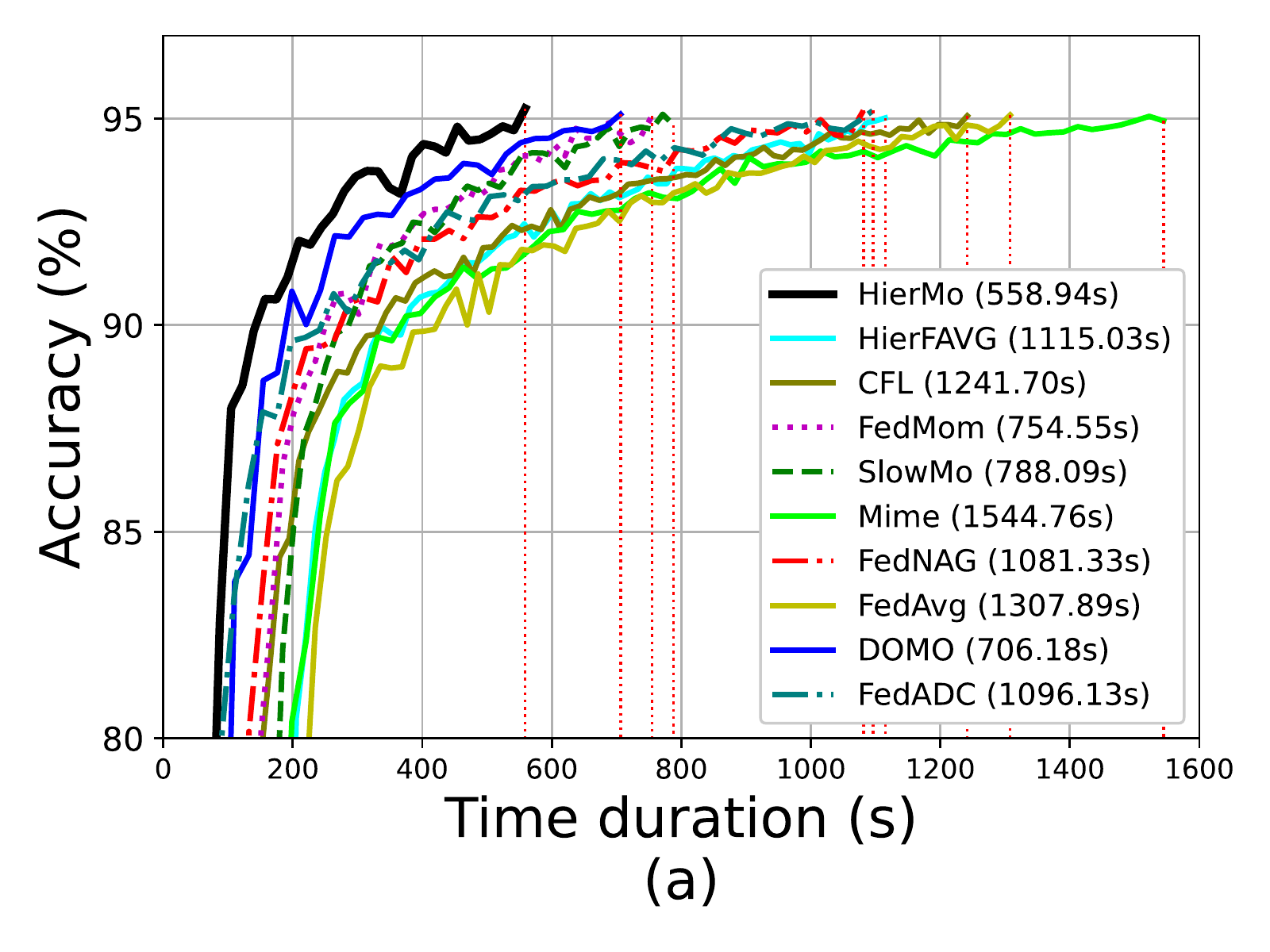}
    \end{minipage}
    \begin{minipage}[t]{0.241\textwidth}
        \centering
        \includegraphics[width=\textwidth]{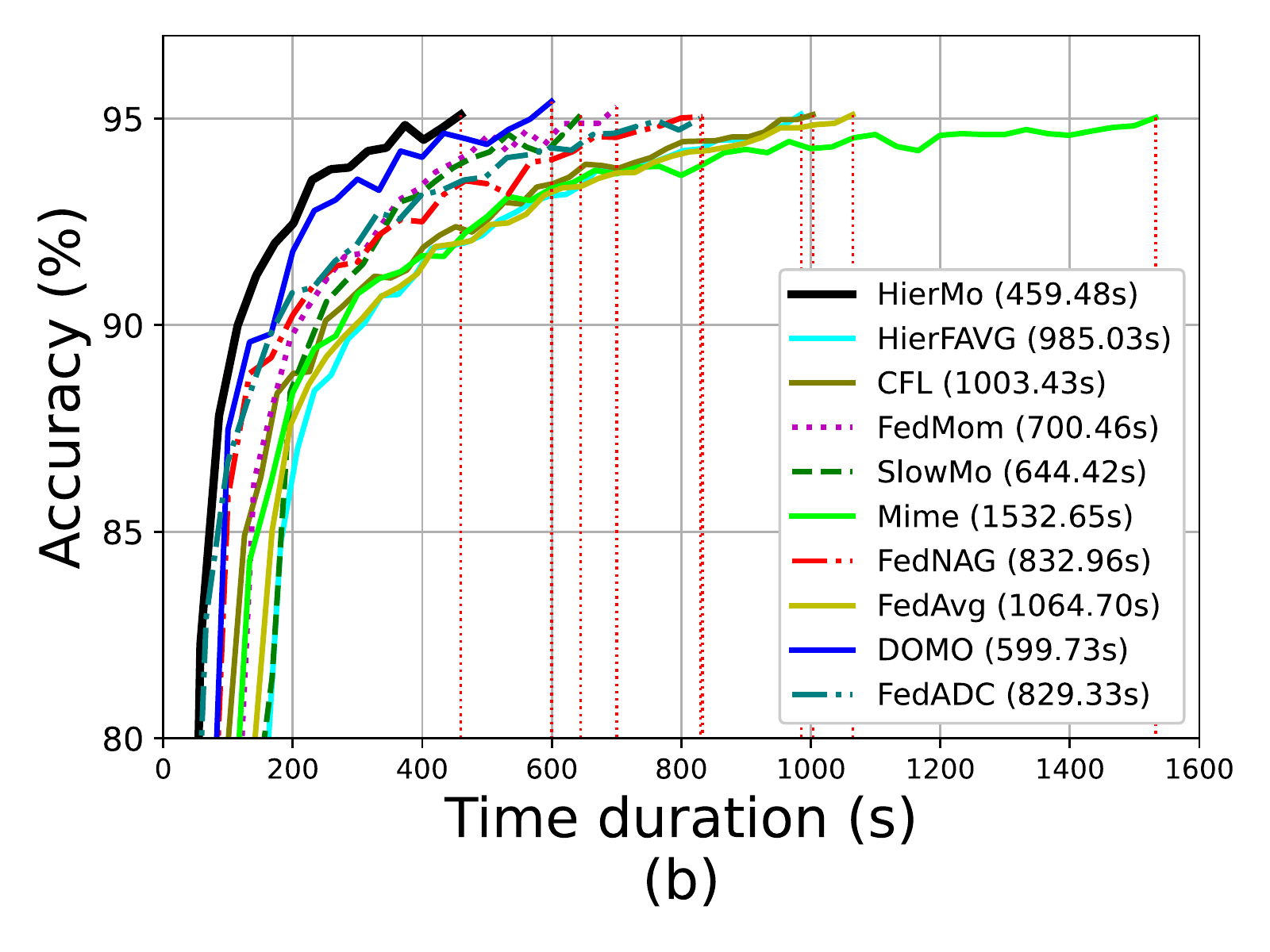}
    \end{minipage}
    \caption{Comparison of total training time to reach  0.95 accuracy under two different settings when CNN is trained on MNIST. The time to reach 0.95 accuracy is labeled in the legends. (a): $\gamma=0.5, \gamma_a=0.5$, $\tau=20$ (two-tier) or $\tau=10,\pi=2$ (three-tier). (b): $\gamma=0.5, \gamma_a=0.5$, $\tau=40$ (two-tier) or $\tau=20, \pi=2$ (three-tier).}
    \label{fig:compare_delay}
\end{figure}

When $\tau\cdot\pi$ (the product of $\tau$ and $\pi$) is fixed in Fig.~\ref{fig:compare_tau_pi}(c), we observe that smaller $\tau$ (larger $\pi$) leads to better performance. This shows that more frequent edge aggregation is more effective compared with more frequent cloud aggregation. 

{

\subsubsection{Effects of non-i.i.d. data distribution}\label{sec:exp_non_IID}
In Fig.~\ref{fig:compare_tau_pi}(e)--(g), we evaluate the effects of different levels of non-i.i.d. data distribution. 
We train CNN on MNIST with the setting $\tau=40$ (two-tier) or $\tau=20, \pi=2$ (three-tier), $N=4, L=2$, and $T=1000$. The curves show the training accuracy. To quantify the level of non-i.i.d. data distribution, we explicitly assign only $x<10$ out of 10 classes of data for each worker. (Each worker has data samples from a subset of classes.) The class of data is randomly allocated to each worker. Smaller \textit{x} represents higher level of non-i.i.d. setting. We use \textit{3-class non-i.i.d.}, \textit{6-class non-i.i.d.}, and \textit{9-class non-i.i.d.} to represent high, middle and low level of non-i.i.d. data respectively.

We observe that HierMo $>$ HierFAVG $>$ DOMO $>$ FedADC $>$ FedNAG $>$ CFL $>$ FedMom $>$ SlowMo $>$ Mime $\approx$ FedAvg in most cases. This is consistent with the results in Table~\ref{tab:exp_compare}, showing that HierMo outperforms all benchmarks under any levels of non-i.i.d. data distribution.
We also observe higher level of non-i.i.d. setting decreases convergence performance for all algorithms. Specifically, HierMo achieves 66.11\% accuracy for high level non-i.i.d. data, while achieving 92.21\% accuracy and 94.70\% accuracy for middle and low level non-i.i.d. data respectively. This matches our expectations where higher level of non-i.i.d. setting causes more data divergence that is denoted by larger $\delta$, and therefore lowers the accuracy.

}

\subsection{Experiment on Trace-driven simulation of HierMo}\label{sec:exp_trace-driven}
\subsubsection{Experimental Setup}
We emulate the real-world three-tier hierarchical FL environment to test the performance of HierMo in the following two aspects. \tikztextcircle{1} To reach a target training accuracy (0.95), we compare the total training time of HierMo and benchmarks. \tikztextcircle{2} For a given total training time $\Psi$, we compare the performance of HierMo under different $(\tau, \pi)$ and verify that  $(\tau^*, \pi^*)$ derived by HierOPT is near optimal. 

We train the CNN on MNIST in the GPU tower server to keep the trace of the sequence of iterations. We use real-world devices as workers (one laptop with Intel Core i3 M380 CPU, three Android phones: Nubia z17s with Qualcomm Snapdragon 835 CPU, Realme GT Neo with MTK Dimensity 1200 CPU, Redmi K30 Ultra with MTK Dimensity 1000+ CPU) to sample worker computation delays. We use Macbook Pro 2018 with Intel Core i7-8750H CPU as the edge node to sample the edge computation delays. The GPU tower server is regarded as the cloud server and the cloud computation delays  are sampled on it. The workers are connected to a HUAWEI Honor router X2+ with 5GHz WiFi. The edge node is also connected to the router with a wired cable (1 Gbps Ethernet). The router is then connected to the public Internet. 

\begin{figure*}[tb!]
    \centering
    \begin{minipage}[t]{0.241\textwidth}
        \centering
        \includegraphics[width=\textwidth]{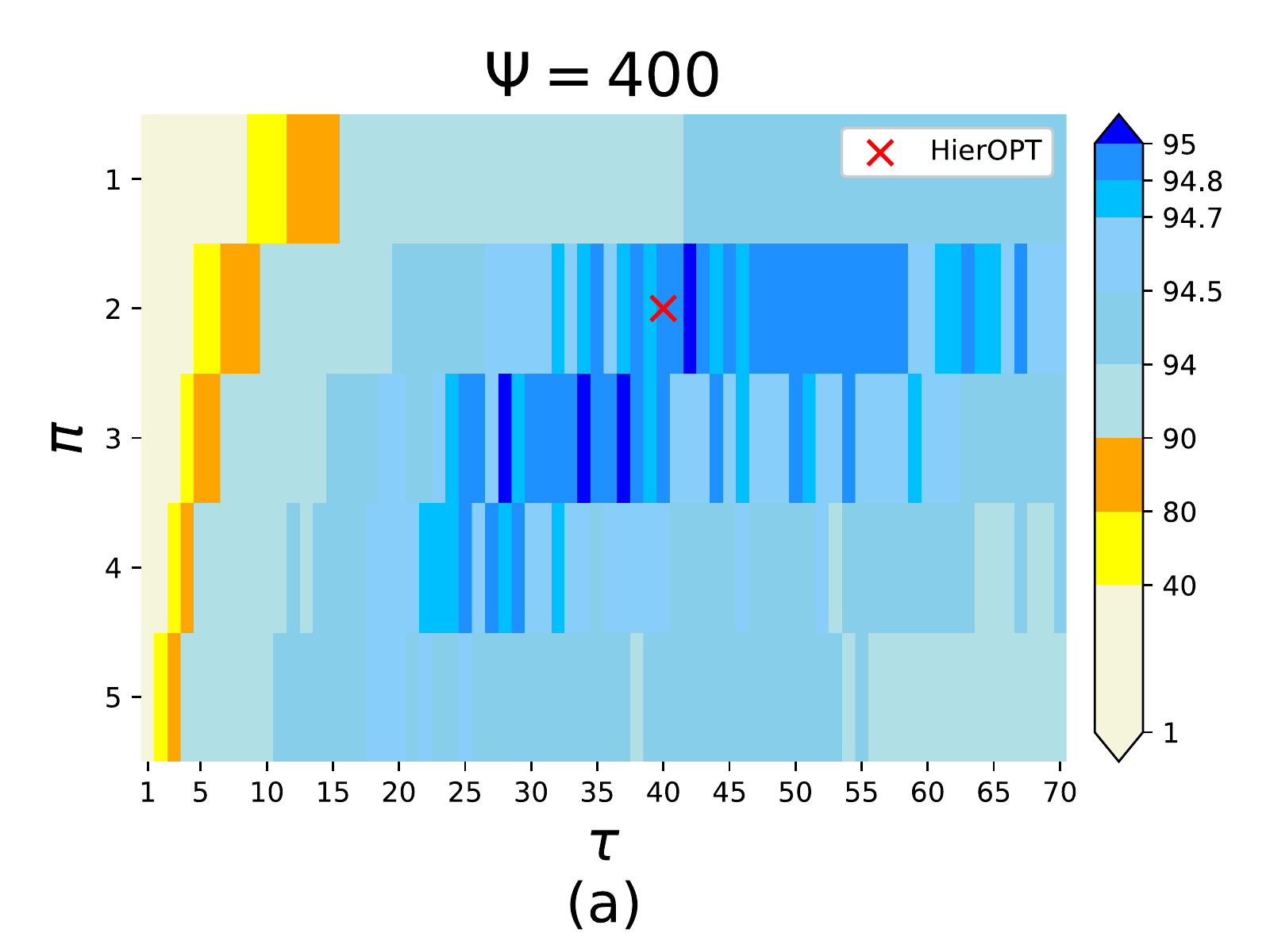}
    \end{minipage}
    \begin{minipage}[t]{0.241\textwidth}
        \centering
        \includegraphics[width=\textwidth]{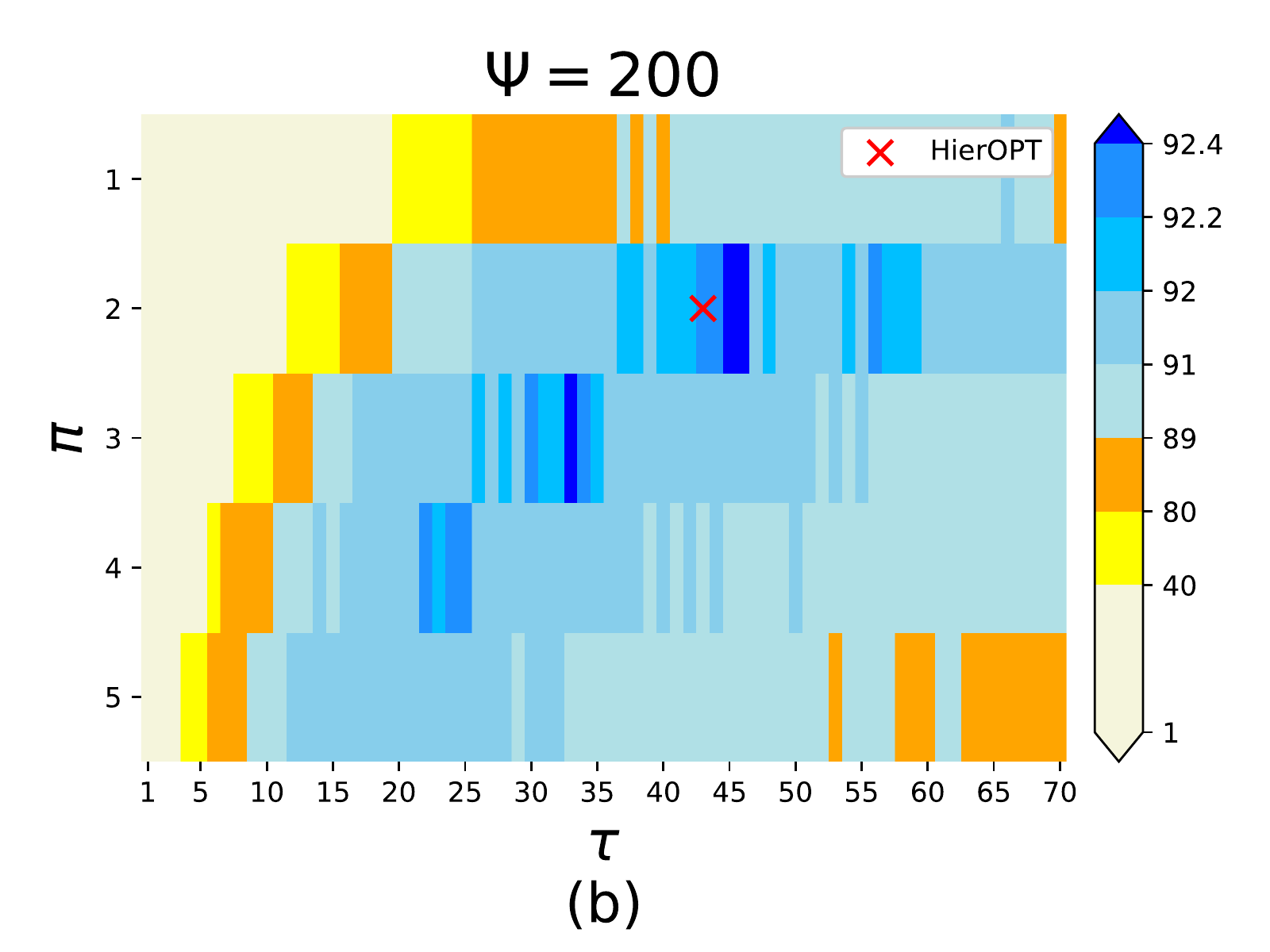}
    \end{minipage}
        \begin{minipage}[t]{0.241\textwidth}
        \centering
        \includegraphics[width=\textwidth]{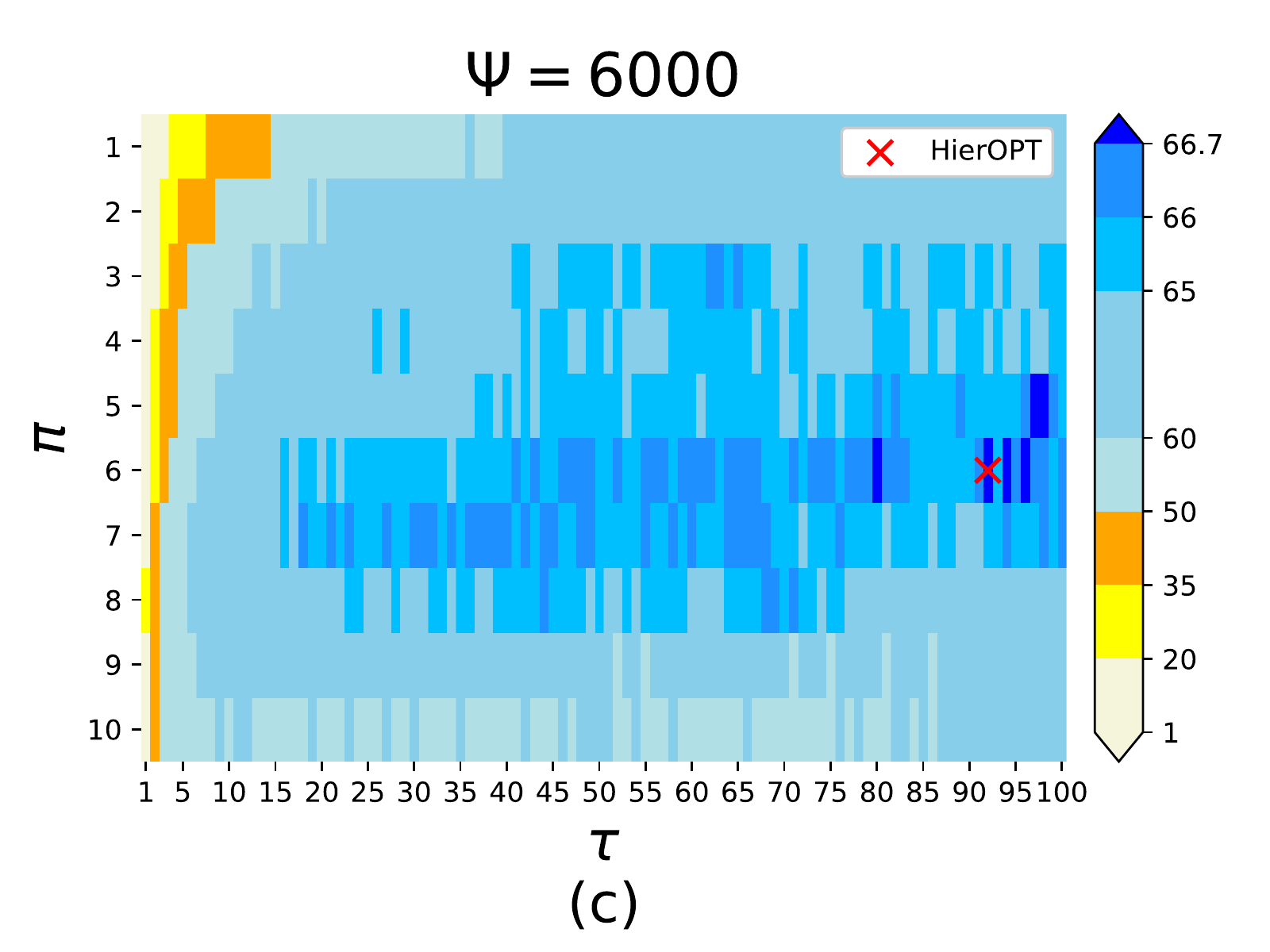}
    \end{minipage}
    \begin{minipage}[t]{0.241\textwidth}
        \centering
        \includegraphics[width=\textwidth]{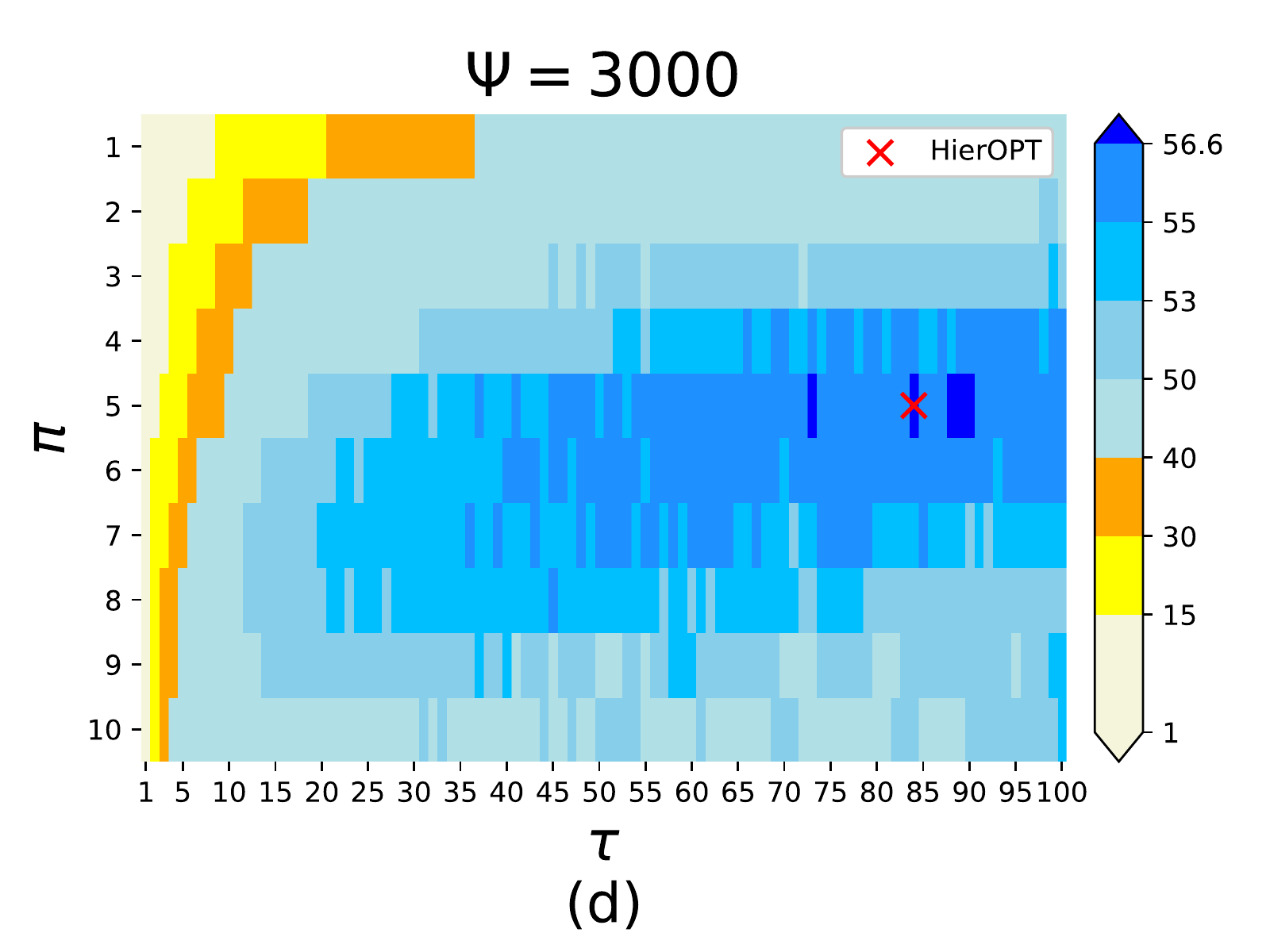}
    \end{minipage}
    \caption{Accuracy comparison for HierMo with derived pair of $(\tau^*,\pi^*)$ by HierOPT (red cross) and different pairs of $(\tau, \pi)$ under limited total training time $\Psi$. The darker color indicates the higher training accuracy (in \%). (a): $\Psi=400$s on MNIST. (b): $\Psi=200$s on MNIST. {(c): $\Psi=6000$s on CIFAR10. (d): $\Psi=3000$s on CIFAR10.} }
    \label{fig:opt}
\end{figure*}

The cloud server is  connected to the Internet via another ISP's access network.
The worker communication delays  are sampled between the workers and the edge node. The edge communication delays  are sampled between the edge node and  the server via the public Internet. Please note that for two-tier FL algorithms, since the workers directly communicate with the cloud, the worker-to-cloud communication delays are sampled as the delays from the devices to the server. We use the trace of the sequence of iterations and the sampled delays to figure out the overall delays as if the training process is conducted in real-world three-tier or two-tier FL environment. Please note that such approach to use a digital representation of physical objects to conduct the experiment is widely used in distributed systems, IoT, Industry 4.0, and machine learning applications \cite{teng2021recent,kirchhof2021understanding}. It can generate a convincing system performance evaluation without deploying physical devices. 

\subsubsection{Total Training Time Comparison}
In Fig.~\ref{fig:compare_delay}, we compare the total training time of HierMo and benchmarks when CNN is trained on MNIST. The experiment is conducted under two settings:
\tikztextcircle{1} $\gamma=0.5, \gamma_a=0.5$, $\tau=20$ (two-tier) or $\tau=10,\pi=2$ (three-tier) and \tikztextcircle{2} $\gamma=0.5, \gamma_a=0.5$, $\tau=40$ (two-tier) or $\tau=20, \pi=2$ (three-tier). 
There are 4 workers and 2 edge nodes with each edge node serving 2 workers (three-tier algorithm). There are 4 workers directly served by the cloud (two-tier algorithm).

We observe that 
to reach the accuracy 0.95, HierMo spends 558.94s under setting \tikztextcircle{1} and 459.48s under setting \tikztextcircle{2} while other benchmarks spend 706.18s--1544.76s under setting \tikztextcircle{1} and 599.73s--1532.65s under setting \tikztextcircle{2} respectively. This demonstrates that HierMo is efficient and decreases the total training time by 21--70\%  compared with the benchmarks.

\subsubsection{Performance of HierOPT}
In Fig.~\ref{fig:opt}, we illustrate the performance of HierOPT. In this experiment, CNN is trained on MNIST and CIFAR10. We set $\gamma=0.5$, $\gamma_a=0.5$, $\Psi=400$s or $\Psi=200$s (MNIST), and {$\Psi=6000$s or $\Psi=3000$s (CIFAR10)}. There are 16 workers and 4 edge nodes with each edge node serving 4 workers. All constants in the objective function \eqref{opt:obj_new} can be sampled in advance of the training process~\cite{wang2019adaptive,dinh2020proxVR}. 

We show the accuracy under different pairs of $(\tau, \pi)$ and flag $(\tau^*, \pi^*)$ derived by HierOPT. The darker color in the chromatography indicates a higher training accuracy. The red cross indicates the derived $(\tau^*, \pi^*)$  by HierOPT. We observe that in all figures, HierOPT can find near-optimal solutions. In Fig.~\ref{fig:opt}(a), when $\Psi=400$s, the optimal accuracy is $95.05\%$, with optimal $(\tau, \pi)=(42, 2)$, while HierOPT finds $(\tau^*, \pi^*)=(40, 2)$, with accuracy $94.82\%$, only a $0.23\%$ gap from the optimum. In Fig.~\ref{fig:opt}(b), when $\Psi=200$s, the optimal accuracy is $92.52\%$, with optimal $(\tau, \pi)=(46, 2)$, while HierOPT finds $(\tau^*, \pi^*)=(43, 2)$, with accuracy $92.23\%$, only a $0.29\%$ gap from the optimum. {For CIFAR10, HierOPT can still find the near-optimal ($\tau^*$, $\pi^*$), with only $0.04\%$ ($67.09\%$ to $67.05\%$) and $0.16\%$ ($56.82\%$ to $56.66\%)$ gap from the real-world optimum, when $\Psi=6000$s and $\Psi=3000$s respectively.}




\section{Conclusion}\label{sec:conclusion}
In this paper, we propose HierMo, a three-tier hierarchical FL algorithm that applies momentum to accelerate convergence. We provide convergence analysis for HierMo, showing that it converges with a rate of $\mathcal{O}\left(\frac{1}{T}\right)$ for smooth non-convex problems under non-i.i.d. data. In the analysis, we {develop a new two-level virtual update (edge and cloud) method to characterize the multi-time cross-two-tier momentum interaction and the cross-three-tier momentum interaction. The performance gain of momentum is also quantified.} We also propose HierOPT to derive a near-optimal setting of worker-edge and edge-cloud aggregation periods $(\tau,\pi)$ under a limited total training time. We verify that HierMo outperforms existing mainstream benchmarks under a wide range of settings. In addition, HierOPT can achieve a near-optimal performance when we test HierMo under different values of $(\tau,\pi)$.

\appendix
\subsection{Proof of Theorem \ref{theorem:xt-xkt}} \label{app_theorem1}

\subsubsection{Equivalent Update}First, we  define $\boldsymbol{v}_{i,\ell}^t\triangleq\boldsymbol{y}_{i,\ell}^t-\boldsymbol{y}_{i,\ell}^{t-1}$ with $\boldsymbol{v}_{i,\ell}^0=\boldsymbol{0}$ for all $i,\ell$. We can obtain $\boldsymbol{x}_{i,\ell}^{t-1}=\boldsymbol{y}_{i,\ell}^{t-1}+\gamma\boldsymbol{v}_{i,\ell}^{t-1}$. The worker momentum/model update in Lines \ref{eq:yit}--\ref{eq:xit} in Algorithm~\ref{alg:HierMo} can then be equivalently written as 
\begin{align}
\boldsymbol{v}_{i,\ell}^t &\gets \gamma\boldsymbol{v}_{i,\ell}^{t-1} - \eta\nabla F_{i,\ell}(\boldsymbol{x}_{i,\ell}^{t-1}),\label{eq:localVi}\\ 
\boldsymbol{x}_{i,\ell}^t 
&\gets \boldsymbol{x}_{i,\ell}^{t-1}+ \gamma\boldsymbol{v}_{i,\ell}^t - \eta\nabla F_{i,\ell}(\boldsymbol{x}_{i,\ell}^{t-1}).\label{eq:localWi}
\end{align}
The aggregated value $\boldsymbol{v}_{\ell}^t$ and the intermediate value $\boldsymbol{x}_{\ell-}^t$ can also be equivalently written as
\begin{align}
\boldsymbol{v}_\ell^t &\gets \sum_{i=1}^{C_\ell}\frac{ D_{i,\ell}}{D_\ell} \boldsymbol{v}_{i,\ell}^t,\quad
\label{eq:globalW}
\boldsymbol{x}_{\ell-}^t \gets \sum_{i=1}^{C_\ell}\frac{ D_{i,\ell}}{D_\ell} \boldsymbol{x}_{i,\ell}^t.
\end{align}
Similarly, the edge and cloud virtual updates \eqref{eq:ykt}--\eqref{eq:xkt} and \eqref{eq:ykt_c}--\eqref{eq:xkt_c} can be equivalently written as
\begin{align}
\boldsymbol{v}_{[k],\ell}^t \gets& \gamma\boldsymbol{v}_{[k],\ell}^{t-1} - \eta\nabla F_{\ell}(\boldsymbol{x}_{[k],\ell}^{t-1}),\nonumber\\ 
\boldsymbol{x}_{[k],\ell}^t 
 \gets& \boldsymbol{x}_{[k],\ell}^{t-1} + \gamma\boldsymbol{v}_{[k],\ell}^t - \eta\nabla F_{\ell}(\boldsymbol{x}_{[k],\ell}^{t-1}),\label{eq:xkt_app}\\
\boldsymbol{v}_{\{p\}}^t \gets& \gamma\boldsymbol{v}_{\{p\}}^{t-1} - \eta\nabla F(\boldsymbol{x}_{\{p\}}^{t-1}),\nonumber\\
\boldsymbol{x}_{\{p\}}^t 
 \gets& \boldsymbol{x}_{\{p\}}^{t-1} + \gamma\boldsymbol{v}_{\{p\}}^t - \eta\nabla F(\boldsymbol{x}_{\{p\}}^{t-1})\label{eq:xpt_app}. 
\end{align}
 We employ the above equivalent update format \eqref{eq:localVi}--\eqref{eq:xpt_app} to complete the proof in the rest of the Appendix.
 

\subsubsection{Constant Definition}We define the constants as follows, which are more conveniently used in the rest of the Appendix.
{\small
\begin{align*}
&A\triangleq\frac{(1+\eta\beta)(1+\gamma)+\sqrt{(1+\eta\beta)^2(1+\gamma)^{2}-4\gamma(1+\eta\beta)}}{2\gamma},\\
&B\triangleq\frac{(1+\eta\beta)(1+\gamma)-\sqrt{(1+\eta\beta)^2(1+\gamma)^{2}-4\gamma(1+\eta\beta)}}{2\gamma},\\
&I\triangleq\frac{\gamma A+A-1}{(A-B)(\gamma A-1)},J\triangleq\frac{\gamma B+B-1}{(A-B)(1-\gamma B)},\\
U& \triangleq \frac{\frac{1+\eta\beta+\eta\beta\gamma}{\gamma}-B}{A-B}=\frac{A-1}{A-B},
V \triangleq \frac{A-\frac{1+\eta\beta+\eta\beta\gamma}{\gamma}}{A-B}=\frac{1-B}{A-B}.
\end{align*}
}

\subsubsection{Subscript \texorpdfstring{$\ell$}{l}}
Since Theorem~\ref{theorem:xt-xkt} focuses on a specific edge node $\ell$, for presentation convenience, in the proofs of Theorem~\ref{theorem:xt-xkt} (including  Lemmas~\ref{lemma:sequence}--\ref{lemma:vt-vkt}), we ignore all  subscript $\ell$. We use $\boldsymbol{x}_i, \boldsymbol{v}_i, \boldsymbol{x}, \boldsymbol{v}, \boldsymbol{x}_{[k]}, \boldsymbol{v}_{[k]}, F_i, F, D_i, D, \delta_i, \delta$, and $C$ to represent $\boldsymbol{x}_{i,\ell}, \boldsymbol{v}_{i,\ell}, \boldsymbol{x}_{\ell-}, \boldsymbol{v}_\ell, \boldsymbol{x}_{[k],\ell}, \boldsymbol{v}_{[k],\ell}, F_{i,\ell}, F_\ell, D_{i,\ell}, D_\ell, \delta_{i,\ell}, \delta_\ell$, and $C_\ell$ respectively. Please note that in the proofs of the theorems other than Theorem~\ref{theorem:xt-xkt}, we do not ignore subscript $\ell$.

\subsubsection{Prerequisite Lemmas for the Proof of Theorem~\ref{theorem:xt-xkt}}
To prove Theorem \ref{theorem:xt-xkt}, the progress mainly includes four steps. (1) We first introduce an important equality in Lemma \ref{lemma:sequence}, which will be used to prove Lemma~\ref{lemma:w_it-w_kt}. (2) We bound $\Vert \boldsymbol{x}_i^t - \boldsymbol{x}_{[k]}^t \Vert$ in Lemma \ref{lemma:w_it-w_kt} based on Lemma \ref{lemma:sequence}. (3) Based on the result of Lemma \ref{lemma:w_it-w_kt}, we then bound $\Vert \boldsymbol{v}^t - \boldsymbol{v}_{[k]}^t \Vert$ in Lemma \ref{lemma:vt-vkt}. Please note that the proofs of Lemmas \ref{lemma:sequence}--\ref{lemma:vt-vkt} are in Appendix~\ref{app:lemma:sequence}--\ref{app:lemma:vt-vkt} respectively. 
(4) Finally, based on the result of Lemma \ref{lemma:vt-vkt}, we bound $\Vert \boldsymbol{x}^t - \boldsymbol{x}_{[k]}^t \Vert$, which concludes Theorem \ref{theorem:xt-xkt}.

\begin{lemma} \label{lemma:sequence}
Given
\begin{align}
&a_{t}=\frac{\delta_{i}}{\beta}\left(\frac{\frac{1+\eta\beta+\eta\beta\gamma}{\gamma}-B}{A-B} A^{t}-\frac{\frac{1+\eta\beta+\eta\beta\gamma}{\gamma}-A}{A-B} B^{t}\right),\\
&A+B=\frac{1+\eta\beta+\eta\beta\gamma+\gamma}{\gamma}=\frac{(1+\eta\beta)(1+\gamma)}{\gamma},\nonumber\\
&AB=\frac{1+\eta\beta}{\gamma},
\end{align}
where $t=0,1,2,...,0<\gamma<1,\eta\beta>0$, we have $(1+\eta\beta)a_{t-1}+\eta\beta\gamma\sum_{i=0}^{t-1} a_i = \gamma a_t.$
\end{lemma}


\begin{lemma} \label{lemma:w_it-w_kt}
For any interval $[k]$, $\forall t \in [(k-1)\tau,k\tau]$, we have $ \Vert\boldsymbol{x}_i^t-\boldsymbol{x}_{[k]}^t \Vert \leq f_i(t-(k-1)\tau),$
where we define the function $f_i(x)$ as $f_i(x) \triangleq \frac{\delta_i}{\beta}(\gamma^x(UA^x+VB^x)-1)$ and the function $u(x)$ as $u(x) \triangleq \gamma^x(UA^x+VB^x)-1$.
\end{lemma}

 \begin{lemma} \label{lemma:vt-vkt}
For any interval $[k]$, $\forall t \in [(k-1)\tau,k\tau]$, we have:
\begin{align}
    \Vert\boldsymbol{v}^t-\boldsymbol{v}_{[k]}^t \Vert\nonumber\leq\eta\delta\left(\frac{U(\gamma A)^{t_0}}{\gamma(A-1)}+\frac{V(\gamma B)^{t_0}}{\gamma(B-1)}-\frac{\gamma^{t_0}-1}{\gamma-1}\right),
\end{align}
where $t_0=t-(k-1)\tau$.
\end{lemma}

\subsubsection{Derivation of Theorem~\ref{theorem:xt-xkt}}
From \eqref{eq:localWi} 
and \eqref{eq:globalW}, we have
\begin{align} \label{eq:wt=}
\boldsymbol{x}^t=\boldsymbol{x}^{t-1}+\gamma\boldsymbol{v}^t-\eta\frac{\sum_{i=1}^{C} D_i \nabla F_i(\boldsymbol{x}_i ^{t-1})}{D}.
\end{align}
From \eqref{eq:xkt_app} and \eqref{eq:wt=}, and according to $\beta$-smoothness, Lemma \ref{lemma:w_it-w_kt}, 
the definition of $f_i(x)$ and $u(x)$, and Assumption~\ref{def:delta}, we have
\begin{align*}
&\Vert\boldsymbol{x}^t-\boldsymbol{x}_{[k]}^t \Vert
=\Vert \boldsymbol{x}^{t-1}+\gamma\boldsymbol{v}^t-\eta\frac{\sum_{i=1}^{C} D_i \nabla F_i(\boldsymbol{x}_i ^{t-1})}{D}\\&-\boldsymbol{x}_{[k]}^{t-1}-\gamma\boldsymbol{v}_{[k]}^t+\eta\nabla F(\boldsymbol{x}_{[k]} ^{t-1}) \Vert&\nonumber\\
\leq&\Vert\boldsymbol{x}^{t-1}-\boldsymbol{x}_{[k]}^{t-1} \Vert +\gamma\Vert\boldsymbol{v}^t-\boldsymbol{v}_{[k]}^t \Vert+\eta\delta u(t-1-(k-1)\tau).&
\end{align*}
Then, according to Lemma \ref{lemma:vt-vkt}, we have
{\small
\begin{align}
&\Vert\boldsymbol{x}^t-\boldsymbol{x}_{[k]}^t \Vert - \Vert\boldsymbol{x}^{t-1}-\boldsymbol{x}_{[k]}^{t-1} \Vert&\nonumber\\
\leq& \gamma\eta\delta\left(\frac{U(\gamma A)^{t_0}}{\gamma(A-1)}+\frac{V(\gamma B)^{t_0}}{\gamma(B-1)}-\frac{\gamma^{t_0}-1}{\gamma-1}\right)\nonumber\\&+\eta\delta(\gamma^{t_0-1}(UA^{t_0-1}+VB^{t_0-1})-1)&\label{ieq:wt-wkt_gap1}\\
=&\eta\delta\left(\frac{U(\gamma A)^{t_0-1}}{A-1}(\gamma A+A-1)+\frac{V(\gamma B)^{t_0-1}}{B-1}(\gamma B+B-1)\right.\nonumber\\
&\left.-\frac{\gamma^{t_0+1}-1}{\gamma-1}\right).&\label{ieq:wt-wkt_gap}
\end{align}
When $t=(k-1)\tau$, we have $\Vert \boldsymbol{x}^t-\boldsymbol{x}_{[k]}^t \Vert =0$. When $t\in((k-1)\tau,k\tau]$, we sum up \eqref{ieq:wt-wkt_gap} for $t, t-1,\dots, (k-1)\tau+1$, leading to
\begin{align*}
&\Vert\boldsymbol{x}^t-\boldsymbol{x}_{[k]}^t \Vert\leq\sum_{x=1}^{t_0}\eta\delta\left(\frac{U(\gamma A)^{x-1}}{A-1}(\gamma A+A-1)\right.\\
&\left.+\frac{V(\gamma B)^{x-1}}{B-1}(\gamma B+B-1)-\frac{\gamma^{x+1}-1}{\gamma-1}\right)&\\
=&\eta \delta\left[I\left((\gamma A)^{t_{0}}-1\right)+J\left((\gamma B)^{t_{0}}-1\right)\right.\\
&\left.-\frac{\gamma^2(\gamma^{t_{0}}-1)-(\gamma-1) t_{0}}{(\gamma-1)^{2}}\right]&\\
=&\eta \delta\left[I(\gamma A)^{t_{0}}+J(\gamma B)^{t_{0}}-\frac{1}{\eta \beta}-\frac{\gamma^2(\gamma^{t_{0}}-1)-(\gamma-1) t_{0}}{(\gamma-1)^{2}}\right]&\\
=&h(t_0),&
\end{align*}
}

\noindent where $I=\frac{\gamma A+A-1}{(A-B)(\gamma A-1)}$ and  $J=\frac{\gamma B+B-1}{(A-B)(1-\gamma B)}$ (as defined before).   $I+J=\frac{1}{\eta\beta}$. $t_0=t-(k-1)\tau$.  We complete the proof of Theorem \ref{theorem:xt-xkt}.

\subsection{Proof of Lemma~\ref{lemma:sequence}}\label{app:lemma:sequence}

Based on the definitions of $U, V$, and $a_t$, we have  $a_t=\frac{\delta_i}{\beta}(UA^t+VB^t).$
According to the inverse theorem of Vieta's formulas, 
we have
\begin{align} \label{eq:quadratic}
    \gamma x^2-(1+\eta\beta+\eta\beta\gamma+\gamma)x+\eta\beta+1=0,
\end{align}
where $x$ values are the roots of the quadratic equation.
The discriminant of the quadratic equation is positive.
\begin{align*} 
\Delta&=(1+\eta\beta+\eta\beta\gamma+\gamma)^2-4(1+\eta\beta)\gamma\\
&>(1+\eta\beta+\gamma)^2-4(1+\eta\beta)\gamma=((1+\eta\beta)-\gamma)^2 >0.
\end{align*}
Thus, the roots of \eqref{eq:quadratic} can be expressed as $A$ and $B$. 
Therefore, we can obtain
\begin{align*}
&(1+\eta\beta) a_{t-1}+\eta \beta \gamma \sum_{i=0}^{t-1} a_{i} - \gamma a_{t}&\\
=&(1+\eta\beta)\frac{\delta_{i}}{\beta}\left(U A^{t-1}+V B^{t-1}\right)+\eta \beta \gamma  \frac{\delta_{i}}{\beta} U \frac{A^{t}-1}{A-1}&\\&+\eta \beta \gamma \frac{\delta_{i}}{\beta} V \frac{B^{t}-1}{B-1}-\gamma \frac{\delta_{i}}{\beta} U A^{t}-\gamma \frac{\delta_{i}}{\beta} V B^{t}&\\
=&\frac{\delta_{i}}{\beta}\left[\frac{A^{t-1} U}{1-A}\left(\gamma A^{2}-(1+\eta \beta+ \eta \beta \gamma+ \gamma) A+1+\eta \beta \right)\right.&\\
&\left.+\frac{B^{t-1} V}{1-B}\left(\gamma B^{2}-(1+\eta \beta+ \eta \beta \gamma+ \gamma) B+1+\eta \beta \right)\right]&\\
&-\frac{\delta_{i}}{\beta}\eta \beta\gamma\left(\frac{U}{A-1}+\frac{V}{B-1}\right)&\\
=&0-\eta \delta_{i}\gamma\left(\frac{U}{A-1}+\frac{V}{B-1}\right)
=0.
\end{align*}
We complete the proof of Lemma \ref{lemma:sequence}.

\subsection{Proof of Lemma~\ref{lemma:w_it-w_kt}}\label{app:lemma:xit-xkt}

To prove Lemma \ref{lemma:w_it-w_kt}, (1) we first bound the gap of  $\Vert\boldsymbol{v}_i^t-\boldsymbol{v}_{[k]}^t \Vert$; (2) then we bound the gap of $\Vert\boldsymbol{x}_i^t-\boldsymbol{x}_{[k]}^t \Vert$, which concludes Lemma \ref{lemma:w_it-w_kt}.

When $t=(k-1)\tau$, we know $\boldsymbol{x}_i^t=\boldsymbol{x}^t=\boldsymbol{x}_{[k]}^t$ by the definition of $\boldsymbol{x}_{[k]}^t$ and the aggregation rules. Hence, we have $\Vert\boldsymbol{x}_i^t-\boldsymbol{x}_{[k]}^t \Vert =0$. Meanwhile, when $t=(k-1)\tau$, we have $x=0$ and $f_i(0)=0$ (Lemma \ref{lemma:w_it-w_kt} holds).

When $t \in ((k-1)\tau, k\tau]$, we bound the momentum gap \begin{flalign} \label{eq:v_i-v_k}
    & \Vert \boldsymbol{v}_i^t - \boldsymbol{v}_{[k]}^t \Vert\nonumber\\
     =& \Vert \gamma\boldsymbol{v}_i^{t-1} - \eta \nabla F_i(\boldsymbol{x}_i^{t-1})\nonumber-(\gamma\boldsymbol{v}_{[k]}^{t-1} - \eta \nabla F(\boldsymbol{x}_{[k]}^{t-1})) \Vert\nonumber \\
     =& \Vert \gamma(\boldsymbol{v}_i^{t-1} - \boldsymbol{v}_{[k]}^{t-1}) - \eta [ \nabla F_i(\boldsymbol{x}_i^{t-1})-\nabla F_i(\boldsymbol{x}_{[k]}^{t-1}) \nonumber\\&+ \nabla F_i(\boldsymbol{x}_{[k]}^{t-1}) - \nabla   F(\boldsymbol{x}_{[k]}^{t-1}) ] \Vert \nonumber\\
    \overset{(a)}{\leq}& \gamma\Vert \boldsymbol{v}_i^{t-1} - \boldsymbol{v}_{[k]}^{t-1} \Vert + \eta \Vert \nabla F_i(\boldsymbol{x}_i^{t-1})- \nabla F_i(\boldsymbol{x}_{[k]}^{t-1}) \Vert \nonumber\\&+ \eta \Vert \nabla F_i(\boldsymbol{x}_{[k]}^{t-1})-\nabla F(\boldsymbol{x}_{[k]}^{t-1}) \Vert \nonumber\\
    \overset{(b)}{\leq}& \gamma\Vert \boldsymbol{v}_i^{t-1} - \boldsymbol{v}_{[k]}^{t-1} \Vert +\eta\beta \Vert \boldsymbol{x}_i^{t-1} - \boldsymbol{x}_{[k]}^{t-1} \Vert + \eta\delta_i, 
\end{flalign}
where (a) is from triangle inequality and (b) is from $\beta$-smoothness and Assumption~\ref{def:delta}.

We use $\gamma^0, \gamma^1,\dots, \gamma^{t-(k-1)\tau-1}$ as multipliers to multiply \eqref{eq:v_i-v_k} when $t, t-1,\dots, (k-1)\tau+1$, respectively.
{\small
\begin{align*}
    & \Vert \boldsymbol{v}_i^t - \boldsymbol{v}_{[k]}^t \Vert 
      \leq \gamma\Vert \boldsymbol{v}_i^{t-1} - \boldsymbol{v}_{[k]}^{t-1} \Vert +\eta\beta \Vert \boldsymbol{x}_i^{t-1} - \boldsymbol{x}_{[k]}^{t-1} \Vert + \eta\delta_i,& \\
    & \gamma\Vert \boldsymbol{v}_i^{t-1} - \boldsymbol{v}_{[k]}^{t-1} \Vert 
      \leq \gamma(\gamma\Vert \boldsymbol{v}_i^{t-2} - \boldsymbol{v}_{[k]}^{t-2} \Vert+\eta\beta \Vert \boldsymbol{x}_i^{t-2} - \boldsymbol{x}_{[k]}^{t-2} \Vert + \eta\delta_i),& \\
    &\dots& \\
    & \gamma^{t-(k-1)\tau-1}\Vert \boldsymbol{v}_i^{(k-1)\tau+1} - \boldsymbol{v}_{[k]}^{(k-1)\tau+1} \Vert\leq \gamma^{t-(k-1)\tau-1}&\\&\quad(\gamma\Vert \boldsymbol{v}_i^{(k-1)\tau} - \boldsymbol{v}_{[k]}^{(k-1)\tau} \Vert 
    +\eta\beta \Vert \boldsymbol{x}_i^{(k-1)\tau} - \boldsymbol{x}_{[k]}^{(k-1)\tau} \Vert + \eta\delta_i).&
\end{align*}
}

\noindent For convenience, we define $G_i(t) \triangleq \Vert \boldsymbol{x}_i^t-\boldsymbol{x}_{[k]}^t\Vert$. Summing up all of the above inequalities with respect to $b\in [1, t-(k-1)\tau]$, we have
\begin{flalign*}
&\Vert\boldsymbol{v}_{i}^t-\boldsymbol{v}_{[k]}^t \Vert 
\leq\eta\beta\sum_{b=1}^{t-(k-1)\tau}\gamma^{b-1} G_i(t-b)+\eta\delta_{i}\sum_{b=1}^{t-(k-1)\tau}\gamma^{b-1}\\&+\gamma^{t-(k-1) \tau}\Vert\boldsymbol{v}_{i}^{(k-1)\tau}-\boldsymbol{v}_{[k]}^{(k-1)\tau}\Vert.
\end{flalign*}
When $t=(k-1)\tau$, we know that $\boldsymbol{v}_i^t=\boldsymbol{v}^t=\boldsymbol{v}_{[k]}^t$ by the definition of $\boldsymbol{v}_{[k]}^t$ and aggregation rules. Then we have $\Vert\boldsymbol{v}_{i}^{(k-1)\tau}-\boldsymbol{v}_{[k]}^{(k-1)\tau}\Vert =0$, so that the last term of above inequality is zero and
\begin{flalign}
\label{ieq:vit-vkt}
&\Vert\boldsymbol{v}_{i}^t-\boldsymbol{v}_{[k]}^t \Vert \leq \eta\beta\sum_{b=1}^{t-(k-1)\tau}\gamma^{b-1} G_i(t-b)+\eta\delta_{i}\sum_{b=1}^{t-(k-1)\tau}\gamma^{b-1}.
\end{flalign}
Now, we can bound the gap between  $\boldsymbol{x}_i^t$ and $\boldsymbol{x}_{[k]}^t$. When $t \in ((k-1)\tau, k\tau]$, we have
{\small
\begin{flalign}
\label{eq:wit-wkt}
& \Vert \boldsymbol{x}_i^t - \boldsymbol{x}_{[k]}^t \Vert\nonumber\\ 
\overset{(a)}{=}& \Vert\boldsymbol{x}_i^{t-1} + \gamma\boldsymbol{v}_i^t - \eta \nabla F_i(\boldsymbol{x}_i^{t-1}) - (\boldsymbol{x}_{[k]}^{t-1} + \gamma\boldsymbol{v}_{[k]}^t - \eta \nabla F(\boldsymbol{x}_{[k]}^{t-1}))\nonumber\Vert\\
=& \Vert\boldsymbol{x}_i^{t-1}- \boldsymbol{x}_{[k]}^{t-1}+ \gamma(\boldsymbol{v}_i^t -\boldsymbol{v}_{[k]}^t) - \eta [\nabla F_i(\boldsymbol{x}_i^{t-1})-\nabla F_i(\boldsymbol{x}_{[k]}^{t-1})\nonumber\\&+\nabla F_i(\boldsymbol{x}_{[k]}^{t-1})- \nabla F(\boldsymbol{x}_{[k]}^{t-1})]\Vert\nonumber\\
\overset{(b)}{\leq}& \Vert \boldsymbol{x}_i^{t-1} - \boldsymbol{x}_{[k]}^{t-1} \Vert+\gamma\Vert \boldsymbol{v}_i^t - \boldsymbol{v}_{[k]}^t \Vert +\eta\beta \Vert \boldsymbol{x}_i^{t-1} - \boldsymbol{x}_{[k]}^{t-1} \Vert + \eta\delta_i \nonumber\\
=&(\eta\beta+1)\Vert \boldsymbol{x}_i^{t-1} - \boldsymbol{x}_{[k]}^{t-1} \Vert+\gamma\Vert \boldsymbol{v}_i^t - \boldsymbol{v}_{[k]}^t \Vert+\eta\delta_i,
\end{flalign}
}

\noindent where (a) is from \eqref{eq:localWi} and \eqref{eq:xkt_app}, and (b) is from triangle inequality, $\beta$-smoothness, and Definition \ref{def:delta}.

Substituting \eqref{ieq:vit-vkt} into \eqref{eq:wit-wkt} and using $G_i(t)$ to denote $\Vert \boldsymbol{x}_i^t-\boldsymbol{x}_{[k]}^t\Vert$ for $t, t-1,\ldots, (k-1)\tau+1$, we have
\begin{flalign}\label{ieq:Gitleq}
G_i(t)\leq&(\eta\beta+1)G_i^{t-1}+\eta\beta\gamma\sum_{b=1}^{t-(k-1)\tau}\gamma^{b-1} G_i(t-b)\nonumber\\
&+\eta\delta_{i}\gamma\sum_{b=1}^{t-(k-1)\tau}\gamma^{b-1}+\eta\delta_i&\nonumber\\
=&(\eta\beta+1)G_i^{t-1}+\eta\beta\gamma\sum_{b=1}^{t-(k-1)\tau}\gamma^{b-1} G_i(t-b)\nonumber\\
&+\eta\delta_{i}\sum_{b=0}^{t-(k-1)\tau}\gamma^{b}.&
\end{flalign}

For convenience, we define $g_i(x)\triangleq\frac{\delta_i}{\beta}(UA^x+VB^x)$. 
We have $f_i(x)=\gamma^x  g_i(x)-\frac{\delta_i}{\beta}.$ 

We use induction to prove $G_i(t)\leq f_i(t-(k-1)\tau)$, $\forall t\in[(k-1)\tau,k\tau]$. First of all, we know that it is true when $t=(k-1)\tau$ because $G_i((k-1)\tau)=f_i(0)$. Then, we assume that $G_i(c)\leq f_i(c-(k-1)\tau)$
holds for all $c\in [(k-1)\tau,t)$, and we show it  also holds for $t$.
\begin{flalign*}
&G_i(t)\\
\overset{(a)}{\leq}&(\eta\beta+1)f_i(t-1-(k-1)\tau)\\&+\eta\beta\sum_{b=1}^{t-(k-1)\tau}\gamma^{b}f_i(t-b-(k-1)\tau)+\eta\delta_{i}\sum_{b=0}^{t-(k-1)\tau}\gamma^{b}\\
\overset{(b)}{=}&(\eta\beta+1)\left(\gamma^{t-1-(k-1)\tau}g_i(t-1-(k-1)\tau)-\frac{\delta_i}{\beta}\right)\\
&+\eta\beta\sum_{b=1}^{t-(k-1)\tau}\left(\gamma^{t-(k-1)\tau}g_i(t-b-(k-1)\tau)-\gamma^{b}\frac{\delta_i}{\beta}\right)\\&+\eta\delta_i \sum_{b=0}^{t-(k-1)\tau}\gamma^b\\
=&\gamma^{t-1-(k-1)\tau}\left((\eta\beta+1)g_i(t-1-(k-1)\tau)\right.\\
&\left.+\eta\beta\gamma\sum_{b=1}^{t-(k-1)\tau}g_i(t-b-(k-1)\tau)\right)-\frac{\delta_i}{\beta}\\
\overset{(c)}{=}&\gamma^{t-(k-1)\tau}g_i(t-(k-1)\tau)-\frac{\delta_i}{\beta}
=f_i(t-(k-1)\tau),
\end{flalign*}
where (a) is from \eqref{ieq:Gitleq}, (b) is from definition of $f_i(x)$, and (c) is from Lemma \ref{lemma:sequence} and $G_i(t)=a_t$. 
We complete the proof of Lemma \ref{lemma:w_it-w_kt}.


\subsection{Proof of Lemma \ref{lemma:vt-vkt}}\label{app:lemma:vt-vkt}


Based on the definition of $u(x)$ in Lemma~\ref{lemma:w_it-w_kt}, we get $f_i(x) = \frac{\delta_i}{\beta}u(x).$
From \eqref{eq:localWi} and \eqref{eq:globalW}, we have
\begin{align} \label{eq:vt=}
    \boldsymbol{v}^t=\gamma\boldsymbol{v}^{t-1}-\eta\frac{\sum_{i=1}^{C} D_i \nabla F_i(\boldsymbol{x}_i ^{t-1})}{D}.
\end{align}
For $t \in ((k-1)\tau, k\tau]$, we have
\begin{align} \label{eq:vt-vkt}
&\Vert\boldsymbol{v}^t-\boldsymbol{v}_{[k]}^t \Vert\nonumber\\
\overset{(a)}{=}&\Vert \gamma\boldsymbol{v}^{t-1}-\eta\frac{\sum_{i=1}^{C} D_i \nabla F_i(\boldsymbol{x}_i ^{t-1})}{D}-\gamma\boldsymbol{v}_{[k]}^{t-1}+\eta\nabla F(\boldsymbol{x}_{[k]} ^{t-1}) \Vert&\nonumber\\
\leq&\gamma\Vert\boldsymbol{v}^{t-1}-\boldsymbol{v}_{[k]}^{t-1}\Vert+\eta\frac{\sum_{i=1}^{C} D_i\Vert \nabla F_i(\boldsymbol{x}_i^{t-1})- \nabla F_i(\boldsymbol{x}_{[k]}^{t-1})\Vert}{D}&\nonumber\\
\overset{(b)}{\leq}&\gamma\Vert\boldsymbol{v}^{t-1}-\boldsymbol{v}_{[k]}^{t-1}\Vert+\eta\beta\frac{\sum_{i=1}^{C} D_i f_i(t-1-(k-1)\tau)}{D}\nonumber\\
\overset{(c)}{=}&\gamma\Vert\boldsymbol{v}^{t-1}-\boldsymbol{v}_{[k]}^{t-1}\Vert+\eta\delta u(t-1-(k-1)\tau),&
\end{align}
where (a) is from \eqref{eq:vt=} and \eqref{eq:xkt_app}; (b) is from $\beta$-smoothness and Lemma \ref{lemma:w_it-w_kt}; and (c) is from definition of $f_i(x)$ and Assumption~\ref{def:delta}.

We use $\gamma^0, \gamma^1,\dots, \gamma^{t-(k-1)\tau-1}$ as multipliers to multiply \eqref{eq:vt-vkt} when $t, t-1,\dots, (k-1)\tau+1$, respectively.
{\small
\begin{align*}
&\Vert \boldsymbol{v}^t - \boldsymbol{v}_{[k]}^t \Vert\leq \gamma\Vert \boldsymbol{v}^{t-1} - \boldsymbol{v}_{[k]}^{t-1} \Vert+ \eta\delta u(t-1-(k-1)\tau),&\\
&\gamma\Vert \boldsymbol{v}^{t-1} - \boldsymbol{v}_{[k]}^{t-1} \leq\gamma^2(\Vert\boldsymbol{v}^{t-2} - \boldsymbol{v}_{[k]}^{t-2}\Vert+\gamma\eta\delta u(t-2-(k-1)\tau),&\\ 
&\dots& \\
&\gamma^{t-(k-1)\tau-1}\Vert \boldsymbol{v}^{(k-1)\tau+1}-\boldsymbol{v}_{[k]}^{(k-1)\tau+1}\Vert&\\&\leq \gamma^{t-(k-1)\tau}\Vert\boldsymbol{v}^{(k-1)\tau}-\boldsymbol{v}_{[k]}^{(k-1)\tau}\Vert+\gamma^{t-1-(k-1)\tau}\eta\delta u(0).&
\end{align*}
}

\noindent Summing up all of the above inequalities, and according to $\Vert\boldsymbol{v}^{(k-1)\tau} - \boldsymbol{v}_{[k]}^{(k-1)\tau}\Vert=0$, we have
\begin{align} 
&\Vert \boldsymbol{v}^t - \boldsymbol{v}_{[k]}^t \Vert\leq\eta\delta\sum_{b=1}^{t-(k-1)\tau}\gamma^{t-b-(k-1)\tau}u(b-1)&\label{ieq:vt-vkt1}\\
=&\eta\delta\left(\gamma^{t-1-(k-1)\tau}U\sum_{b=1}^{t-(k-1)\tau}A^{b-1}\right.&\nonumber\\&\left.+\gamma^{t-1-(k-1)\tau}V\sum_{b=1}^{t-(k-1)\tau}B^{b-1}-\sum_{b=1}^{t-(k-1)\tau}\gamma^{b-1}\right)&\nonumber\\
=&\eta\delta\left(\gamma^{t_0-1}U\frac{A^{t_0}-1}{A-1}+\gamma^{t_0-1}V\frac{B^{t_0}-1}{B-1}-\frac{\gamma^{t_0}-1}{\gamma-1}\right)&\nonumber\\
=&\eta\delta\left(\frac{U(\gamma A)^{t_0}}{\gamma(A-1)}+\frac{V(\gamma B)^{t_0}}{\gamma(B-1)}-\frac{\gamma^{t_0}-1}{\gamma-1}\right)\nonumber\\&-\eta\delta\gamma^{t_0-1}\left(\frac{U}{A-1}+\frac{V}{B-1}\right)&\nonumber\\
=&\eta\delta\left(\frac{U(\gamma A)^{t_0}}{\gamma(A-1)}+\frac{V(\gamma B)^{t_0}}{\gamma(B-1)}-\frac{\gamma^{t_0}-1}{\gamma-1}\right)&\label{ieq:vt-vkt}
\end{align}
where $t_0=t-(k-1)\tau$. We complete the proof of Lemma \ref{lemma:vt-vkt}.


\subsection{Proof of Theorem \texorpdfstring{\ref{theorem:xgktau-xktau}}{2}} \label{app_theorem2_new}

Based on the edge momentum update rules in Lines \ref{eq:ygktau}--\ref{eq:xgktau} in Algorithm~\ref{alg:HierMo}, and \eqref{eq:localWi} we have
{\small
\begin{flalign}\label{eq:app_xgktau-xktau}
&\boldsymbol{x}_{\ell+}^{k\tau}-\boldsymbol{x}_{\ell-}^{k\tau}=\gamma_a\left(\boldsymbol{x}_{\ell-}^{k\tau}-\boldsymbol{x}_{\ell-}^{(k-1)\tau}\right)
=\gamma_a\sum_{t=(k-1)\tau}^{k\tau-1}\left(\boldsymbol{x}_{\ell-}^{t+1}-\boldsymbol{x}_{\ell-}^t\right)\nonumber\\
&=\gamma_a\sum_{t=(k-1)\tau}^{k\tau-1}\sum_{i=1}^{C_\ell}\frac{D_{i,\ell}}{D_\ell}\left(\boldsymbol{x}_{i,\ell}^{t+1}-\boldsymbol{x}_{i,\ell}^{t}\right)\nonumber\\
&=\gamma_a\sum_{t=(k-1)\tau}^{k\tau-1}\sum_{i=1}^{C_\ell}\frac{D_{i,\ell}}{D_\ell}\left(\gamma^2\boldsymbol{v}_{i,\ell}^t-\eta(\gamma+1)\nabla F_{i,\ell}(\boldsymbol{x}_{i,\ell}^t)\right),
\end{flalign}
}
and we define 
\begin{align}
  \mu \triangleq \max _{p\in[1,P], \forall t,\ell, i}\left\{\frac{\Vert\gamma(\boldsymbol{v}_{\{p\}}^t)\Vert}{\Vert\eta\nabla F(\boldsymbol{x}_{\{p\}}^t)\Vert}, \frac{\Vert\gamma(\boldsymbol{v}_{i,\ell}^t)\Vert}{\Vert\eta\nabla F_{i,\ell}(\boldsymbol{x}_{i,\ell}^t)\Vert} \right\}.\label{app:mu}  
\end{align}
Because $F_{i,\ell}(\cdot)$ is $\rho$-Lipschitz, and according to  \cite[Lecture 2, Lemma 1]{IFT6085}, we have $\Vert\nabla F_{i,\ell}(\cdot)\Vert^2\leq\rho^2$. Therefore, based on the definition of $\mu$ and \eqref{eq:app_xgktau-xktau}, we can derive
\begin{flalign}\label{eq:app_xl+ktau-xl-ktau}
&\left\Vert\boldsymbol{x}_{\ell+}^{k\tau}-\boldsymbol{x}_{\ell-}^{k\tau}\right\Vert\nonumber\\
\leq&\gamma_a\sum_{t=(k-1)\tau}^{k\tau-1}\sum_{i=1}^{C_\ell}\frac{D_{i,\ell}}{D_\ell}\left\Vert\gamma^2\boldsymbol{v}_{i,\ell}^t-\eta(\gamma+1)\nabla F_{i,\ell}(\boldsymbol{x}_{i,\ell}^t)\right\Vert\nonumber\\
\leq&\gamma_a\sum_{t=(k-1)\tau}^{k\tau-1}\sum_{i=1}^{C_\ell}\frac{D_{i,\ell}}{D_\ell}\left(\gamma\mu\eta+\eta(\gamma+1)\right)\rho\nonumber\\
=&\gamma_a\tau\rho\eta(\gamma\mu+\gamma+1).
\end{flalign}
We complete the proof of Theorem~\ref{theorem:xgktau-xktau}.

\subsection{Proof of Theorem \texorpdfstring{\ref{theorem:x_ppi_ptaupi-x_p_ptaupi}}{3}}
\label{app_theorem3}
First, we define edge virtual update which is meaningful in cloud interval $\{p\}$ as $\boldsymbol{y}_{\{p\},\ell}^t$ and $\boldsymbol{x}_{\{p\},\ell}^t$. The value synchronization and edge virtual update on $\{p\}$ are conducted as
\begin{align}
    \boldsymbol{y}_{\{p\},\ell}^{(p-1)\tau\pi} &\gets \boldsymbol{y}^{(p-1)\tau\pi}\label{eq:yk(k-1)tau_e},\\
    \boldsymbol{x}_{\{p\},\ell}^{(p-1)\tau\pi} &\gets \boldsymbol{x}_{}^{(p-1)\tau\pi}\label{eq:xk(k-1)tau_e},
\end{align}
when $t=(p-1)\tau\pi$, and
\begin{align}
\boldsymbol{y}_{\{p\},\ell}^t &\gets \boldsymbol{x}_{\{p\},\ell}^{t-1} - \eta\nabla F_\ell(\boldsymbol{x}_{\{p\},\ell}^{t-1}),\label{eq:ykt_e}\\
\boldsymbol{x}_{\{p\},\ell}^t 
    &\gets \boldsymbol{y}_{\{p\},\ell}^t + \gamma(\boldsymbol{y}_{\{p\},\ell}^t - \boldsymbol{y}_{\{p\},\ell}^{t-1})\label{eq:xkt_e},
\end{align}
when $p \in ((p-1)\tau\pi,p\tau\pi]$. According to Theorem~\ref{theorem:xt-xkt}, we have proved the gap between intermediate worker update on the edge $\sum_{i=1}^{C_\ell}\frac{D_{i,\ell}}{D_\ell}\boldsymbol{x}_{i,\ell}^t$ and edge virtual update $\boldsymbol{x}_{[k],\ell}^{t}$. Equivalently, the gap between the intermediate edge virtual update on the cloud $\sum_{\ell=1}^L\frac{D_\ell}{D}\boldsymbol{x}_{\{p\},\ell}^t$ and the cloud virtual update $\boldsymbol{x}_{\{p\}}^t$ can be derived as the same way as Theorem~\ref{theorem:xt-xkt}. The only difference is the gradient divergence. The edge-level gradient divergence is $\delta_\ell$ and the cloud-level gradient divergence is $\delta$.
Therefore, for any cloud interval $\{p\}, \forall t \in [(p-1)\tau\pi,p\tau\pi], \forall \ell \in L$, we have
\begin{align}\label{ieq:xt_c-xkt_c}
    \left\Vert\sum_{\ell=1}^L\frac{D_\ell}{D}\boldsymbol{x}_{\{p\},\ell}^t-\boldsymbol{x}_{\{p\}}^t \right\Vert \leq h(t-(p-1)\tau\pi,\delta).
\end{align}
At the end of cloud interval ${\{p\}}$, when $t=p\tau\pi$, we have
\begin{align}\label{ieq:theorem3_1}
    \left\Vert\sum_{\ell=1}^L\frac{D_\ell}{D}\boldsymbol{x}_{\{p\},\ell}^{p\tau\pi}-\boldsymbol{x}_{\{p\}}^{p\tau\pi}\right\Vert \leq h(\tau\pi,\delta).
\end{align}
Based on the definition of $\boldsymbol{x}_{[p\pi]}^{p\tau\pi}$ in Theorem~\ref{theorem:x_ppi_ptaupi-x_p_ptaupi} and the definition of $\boldsymbol{x}_{\{p\},\ell}^{p\tau\pi}$, we obtain
\begin{align}\label{ieq:theorem3_2}
&\left\Vert\boldsymbol{x}_{[p\pi]}^{p\tau\pi}-\sum_{\ell=1}^L\frac{D_\ell}{D}\boldsymbol{x}_{\{p\},\ell}^{p\tau\pi}\right\Vert\leq\sum_{\ell=1}^L\frac{D_\ell}{D} \left\Vert\boldsymbol{x}_{[p\pi],\ell}^{p\tau\pi}-\boldsymbol{x}_{\{p\},\ell}^{p\tau\pi}\right\Vert\nonumber\\
&\leq \pi \sum_{\ell=1}^L\frac{D_\ell}{D}(h(\tau,\delta_\ell)+s(\tau)).
\end{align}
Combining \eqref{ieq:theorem3_1} and \eqref{ieq:theorem3_2}, we complete the proof of Theorem~\ref{theorem:x_ppi_ptaupi-x_p_ptaupi}.

\subsection{Proof of Monotone of \texorpdfstring{$h(x)$}{h(x)}}\label{app_monotone} 


To prove the monotone increasing of $h(x)$, it is equivalent to prove $h(x)-h(x-1)\geq 0$ for all integer $x\geq 1$. 

When $x=0$ or $x=1$, because $IA+JB=\frac{1+\eta\beta+\eta\beta\gamma}{\eta\beta\gamma}$, we have
$h(0)=\eta\delta(I+J-\frac{1}{\eta\beta})=0$ and $h(1)=\eta\delta\left(\gamma(IA+JB)-\frac{1}{\eta\beta}-\gamma-1\right)=0.$
Then, when $x=1$, we have $h(x)-h(x-1)=0$.
When $x>1$, according to the definitions of $A$, $B$, $U$, and $V$, we can obtain that $\gamma A >1, 0<\gamma B<1,\frac{1}{\gamma +1}<B<1, I>0, J>0,U>0, V>0$, and $U+V=1$. Then, we have 
\begin{align}\label{app:monotone}
U(\gamma A)^i+V(\gamma B)^i \geq (1+\eta\beta+\eta\beta\gamma)^i
\end{align}
holds $\forall i=0,1,\ldots$. This is because: \tikztextcircle{1} When $i=0, U(\gamma A)^i+V(\gamma B)^i = (1+\eta\beta+\eta\beta\gamma)^i =1$, \eqref{app:monotone} holds. \tikztextcircle{2} When $i=1$, we have
$
U(\gamma A)^i+V(\gamma B)^i= \gamma(UA+VB)
= \gamma\left(\frac{A-1}{A-B}A+\frac{1-B}{A-B}B\right)
=\gamma(A+B-1)
=1+\eta\beta+\eta\beta\gamma.
$
\eqref{app:monotone} still holds. \tikztextcircle{3} When $i>1$, according to Jensen inequality, 
and because any function $n(x)=x^i$ is convex, we have $U(\gamma A)^i+V(\gamma B)^i\geq(\gamma UA+\gamma VB)^i=(1+\eta\beta+\eta\beta\gamma)^i$. \eqref{app:monotone} still holds.

According to \eqref{app:monotone} and the definition of $u(x)$ in Lemma \ref{lemma:w_it-w_kt}, we have $u(x)=U(\gamma A)^x+V(\gamma B)^x -1 \geq (1+\eta\beta+\eta\beta\gamma)^x -1>0$. Then,  we have
{\small
\begin{flalign*}
&h(x)-h(x-1)
=\eta\delta\left(\frac{U(\gamma A)^{x}(\gamma A+A-1)}{\gamma A(A-1)}\right.\\&\left.+\frac{V(\gamma B)^{x}(\gamma B+B-1)}{\gamma B(B-1)}-\frac{\gamma^{x+1}-1}{\gamma-1}\right)\\
\overset{(a)}{=}&\gamma\eta\delta\left(\frac{U(\gamma A)^{x}}{\gamma(A-1)}+\frac{V(\gamma B)^{x}}{\gamma(B-1)}-\frac{\gamma^{x}-1}{\gamma-1}\right)\\&+\eta\delta(\gamma^{x-1}(UA^{x-1}+VB^{x-1})-1)\\
\overset{(b)}{=}&\gamma\eta\delta\sum_{b=1}^{x}\gamma^{x-b}u(b-1)+\eta\delta u(x-1)>0,\\
\end{flalign*}
}

\noindent where (a) is because \eqref{ieq:wt-wkt_gap} equals \eqref{ieq:wt-wkt_gap1}; (b) is because \eqref{ieq:vt-vkt} equals \eqref{ieq:vt-vkt1}, $x=t-(k-1)\tau$, and the definition of $u(x)$. To conclude, we have proven that $h(0)=h(1)=0$ and $h(x)$ increases with $x$ when $x\geq1$. 


\subsection{Proof of Theorem \texorpdfstring{\ref{theorem:Fwt-Fw*}}{4}} \label{app_theorem4}

For convenience, we define
$c_{\{p\}}(t) \triangleq F(\boldsymbol{x}_{\{p\}}^t) - F(\boldsymbol{x}^*)$
for a given cloud interval $\{p\}$, where $t \in [(p-1)\tau\pi,p\tau\pi]$.
We also define the following constants in this subsection.
\begin{align}
&\omega \triangleq \min _{p\in [1,P], t \in\{p\}} \frac{1}{\Vert\boldsymbol{x}_{\{p\}}^t-\boldsymbol{x}^{*}\Vert^{2}},\nonumber\\
&\sigma \triangleq \min_{p\in [1,P], {t_1,t_2}\in\{p\}} \frac{\Vert\nabla F(\boldsymbol{x}_{\{p\}}^{t_1})\Vert}{\Vert\nabla F(\boldsymbol{x}_{\{p\}}^{t_2})\Vert},\\
&\alpha \triangleq \eta(\gamma+1)\left(1-\frac{\beta\eta(\gamma+1)}{2}\right)-\frac{\beta\eta^2\gamma^2 \mu^2}{2}\nonumber\\ &\quad-\eta\gamma\mu(1-\beta\eta(\gamma+1)).
\end{align}

According to the convergence lower bound of any gradient descent methods given in \cite[Theorem 3.14]{bubeck2014convex}, we always have $c_{\{p\}}(t)>0$
for any $t$ and $p$. 
Then we derive the upper bound of $c_{\{p\}}(t+1)-c_{\{p\}}(t)$, where $t \in [(p-1)\tau\pi,p\tau\pi-1]$. Because $F(\cdot)$ is $\beta$-smooth, according to \cite[Lemma 3.4]{bubeck2014convex}, we have
{\small
\begin{flalign} \label{eq:ck(t+1)-ckt}
&c_{\{p\}}(t+1)-c_{\{p\}}(t)
=F(\boldsymbol{x}_{\{p\}}^{t+1})-F(\boldsymbol{x}_{\{p\}}^t)\nonumber\\
\leq& \langle\nabla F(\boldsymbol{x}_{\{p\}}^t), \boldsymbol{x}_{\{p\}}^{t+1}-\boldsymbol{x}_{\{p\}}^t\rangle+\frac{\beta}{2}\|\boldsymbol{x}_{\{p\}}^{t+1}-\boldsymbol{x}_{\{p\}}^t\|^{2}&\nonumber \\
=&\gamma\langle\nabla F(\boldsymbol{x}_{\{p\}}^t), \boldsymbol{v}_{\{p\}}^{t+1}\rangle-\eta\|\nabla F(\boldsymbol{x}_{\{p\}}^t)\|^2\nonumber\\&+\frac{\beta}{2}\|\gamma\boldsymbol{v}_{\{p\}}^{t+1}-\eta \nabla F(\boldsymbol{x}_{\{p\}}^t)\|^2&\nonumber\\
\overset{(a)}{=}&-\eta(\gamma+1)\left(1-\frac{\beta\eta(\gamma+1)}{2}\right)\|\nabla F(\boldsymbol{x}_{\{p\}}^t)\|^2\nonumber\\
&+\frac{\beta\gamma^4}{2}\|\boldsymbol{v}_{\{p\}}^t\|^2
+\gamma^2\left(1-\beta\eta(\gamma+1)
\right)\langle\nabla F(\boldsymbol{x}_{\{p\}}^t), \boldsymbol{v}_{\{p\}}^t\rangle&\nonumber\\
\overset{(b)}{\leq}& \left(-\eta(\gamma+1)\left(1-\frac{\beta\eta(\gamma+1)}{2}\right)+\frac{\beta\eta^2\gamma^2 \mu^2}{2}\right.\nonumber\\
&\left.+\eta\gamma\mu(1-\beta\eta(\gamma+1))\right)\|\nabla F(\boldsymbol{x}_{\{p\}}^t)\|^{2},&
\end{flalign}
}

\noindent where (a) is replacing $\boldsymbol{v}_{\{p\}}^{t+1}$ by \eqref{eq:xpt_app} and rearranging the formula; (b) is because $\|\gamma \boldsymbol{v}_{\{p\}}^t\| \leq \mu \|\eta \nabla F(\boldsymbol{x}_{\{p\}}^t)\|$ with the definition of $\mu$. According to Cauchy-Schwarz inequality, we can obtain 
$
\langle\nabla F(\boldsymbol{x}_{\{p\}}^t), \boldsymbol{v}_{\{p\}}^t\rangle\leq\|\nabla F(\boldsymbol{x}_{\{p\}}^t)\|  \|\boldsymbol{v}_{\{p\}}^t\| \leq\frac{\mu\eta}{\gamma}\|\nabla F(\boldsymbol{x}_{\{p\}}^t)\|^2. 
$ According to the definition of $\alpha$, and Condition (2.1) of Theorem \ref{theorem:Fwt-Fw*} with $h(\tau,\delta_\ell)\geq0$ and $h(\tau\pi,\delta)\geq0$ which are proved in Appendix~\ref{app_monotone}, we have $\alpha>0$. Then from \eqref{eq:ck(t+1)-ckt}, we have
\begin{align} 
\label{ieq:ckt+1<ckt}
c_{\{p\}}(t+1)\leq c_{\{p\}}(t) -\alpha\|\nabla F(\boldsymbol{x}_{\{p\}}^t)\|^{2}.
\end{align}

Because $F(\cdot)$ is $\rho$-Lipschitz, and according to \cite[Lecture 2, Lemma 1]{IFT6085}, there exists a point $\boldsymbol{x}_{\{p\}}^{t_2}$ such that $F(\boldsymbol{x}_{\{p\}}^t)-F(\boldsymbol{x}^{*})
= \langle\nabla F(\boldsymbol{x}_{\{p\}}^{t_2}), \boldsymbol{x}_{\{p\}}^t-\boldsymbol{x}^{*}\rangle$. Hence, by Cauchy-Schwarz inequality, we have $
c_{\{p\}}(t) =F(\boldsymbol{x}_{\{p\}}^t)-F(\boldsymbol{x}^{*})
\leq\|\nabla F(\boldsymbol{x}_{\{p\}}^{t_2})\|\|\boldsymbol{x}_{\{p\}}^t-\boldsymbol{x}^{*}\|.$ Based on the definition of $\sigma$, and replacing $t$ with $t_1$, we have $\Vert\nabla F(\boldsymbol{x}_{\{p\}}^{t})\Vert\geq \sigma\Vert\nabla F(\boldsymbol{x}_{\{p\}}^{t_2})\Vert$. Thus, $
\|\nabla F(\boldsymbol{x}_{\{p\}}^t)\| \geq \sigma\Vert\nabla F(\boldsymbol{x}_{\{p\}}^{t_2})\Vert \geq \frac{\sigma c_{\{p\}}(t)}{\|\boldsymbol{x}_{\{p\}}^t-\boldsymbol{x}^{*}\|}. $
Substituting above inequality into \eqref{ieq:ckt+1<ckt}, and noting $\omega\leq\frac{1}{\|\boldsymbol{x}_{\{p\}}^t-\boldsymbol{x}^{*}\|^{2}}$ by the definition of $\omega$, we get $
c_{\{p\}}(t+1)\leq c_{\{p\}}(t) - \frac{\alpha\sigma^2 c_{\{p\}}(t)^2}{\|\boldsymbol{x}_{\{p\}}^t-\boldsymbol{x}^{*}\|^2}
\leq c_{\{p\}}(t) - \omega\alpha\sigma^2 c_{\{p\}}(t)^2. $
Because $\alpha >0$, $c_{\{p\}}(t)>0$, and \eqref{ieq:ckt+1<ckt}, we have $0<c_{\{p\}}(t+1)\leq c_{\{p\}}(t)$. Dividing both sides by $c_{\{p\}}(t+1)c_{\{p\}}(t)$, we get
$
\frac{1}{c_{\{p\}}(t)} \leq \frac{1}{c_{\{p\}}(t+1)}-\omega\alpha\sigma^2\frac{c_{\{p\}}(t)}{c_{\{p\}}(t+1)}.
$
We note that $\frac{c_{\{p\}}(t)}{c_{\{p\}}(t+1)} \geq 1$. Thus,
$
\frac{1}{c_{\{p\}}(t+1)}-\frac{1}{c_{\{p\}}(t)}\geq \omega\alpha\sigma^2\frac{c_{\{p\}}(t)}{c_{\{p\}}(t+1)}\geq \omega\alpha\sigma^2.
$
Summing up the above inequality by $t\in [(p-1)\tau\pi, p\tau\pi-1]$, we have
$
\frac{1}{c_{\{p\}}(p\tau\pi)}-\frac{1}{c_{\{p\}}((p-1)\tau\pi)}
=\sum_{t=(p-1)\tau\pi}^{p\tau\pi-1}\left(\frac{1}{c_{\{p\}}(t+1)}-\frac{1}{c_{\{p\}}(t)}\right)
\geq \sum_{t=(p-1)\tau\pi}^{p\tau\pi-1}\omega\alpha\sigma^2 = \tau\pi\omega\alpha\sigma^2.
$
Then, we sum up the above inequality by $p\in [1,P]$, after rearranging the left-hand side and noting that $T=P\tau\pi$, we can get
{\small
\begin{flalign}\label{ieq:cKT-c10}
&\sum_{p=1}^{P}\left(\frac{1}{c_{\{p\}}(p\tau\pi)}-\frac{1}{c_{\{p\}}((p-1)\tau\pi)}\right)\nonumber\\
=&\frac{1}{c_{\{p\}}(T)}-\frac{1}{c_{\{1\}}(0)}-\sum_{p=1}^{P-1}\left(\frac{1}{c_{\{p+1\}}(p\tau\pi)}-\frac{1}{c_{\{p\}}(p\tau\pi)}\right)\nonumber\\
\geq& P\tau\pi\omega\alpha\sigma^2 = T\omega\alpha\sigma^2.
\end{flalign}
}

Following \eqref{ieq:cKT-c10}, we note that
\begin{flalign} \label{ieq:frac:c(k+1)-ck}
&\frac{1}{c_{\{p+1\}}(p\tau\pi)}-\frac{1}{c_{\{p\}}(p\tau\pi)}
=\frac{c_{\{p\}}(p\tau\pi)-c_{\{p+1\}}(p\tau\pi)}{c_{\{p\}}(p\tau\pi)c_{\{p+1\}}(p\tau\pi)}\nonumber\\
=&\frac{F(\boldsymbol{x}_{\{p\}}^{p\tau\pi})-F(\boldsymbol{x}_{\{p+1\}}^{p\tau\pi})}{c_{\{p\}}(p\tau\pi)c_{\{p+1\}}(p\tau\pi)}
=\frac{F(\boldsymbol{x}_{\{p\}}^{p\tau\pi})-F(\boldsymbol{x}_{}^{p\tau\pi})}{c_{\{p\}}(p\tau\pi)c_{\{p+1\}}(p\tau\pi)}\nonumber\\
=&\frac{F(\boldsymbol{x}_{\{p\}}^{p\tau\pi})-F(\boldsymbol{x}_{[p\pi]}^{p\tau\pi})+\left(F(\boldsymbol{x}_{[p\pi]}^{p\tau\pi})-F(\boldsymbol{x}_{}^{p\tau\pi})\right)}{c_{\{p\}}(p\tau\pi)c_{\{p+1\}}(p\tau\pi)}\nonumber\\
\overset{(a)}{\geq}& \frac{-\rho\sum_{\ell=1}^{L}\frac{D_\ell}{D}\left( h(\tau,\delta_\ell)+s(\tau)\right)}{c_{\{p\}}(p\tau\pi)c_{\{p+1\}}(p\tau\pi)}\nonumber\\
&\overset{(b)}{+}\frac{-\rho\left(h(\tau\pi,\delta)+\pi\sum_{\ell=1}^L\frac{D_\ell}{D}\left(h(\tau,\delta_\ell)+s(\tau)\right)\right)}{c_{\{p\}}(p\tau\pi)c_{\{p+1\}}(p\tau\pi)}\nonumber\\
=&\frac{-\rho j(\tau,\pi,\delta_\ell,\delta)}{c_{\{p\}}(p\tau\pi)c_{\{p+1\}}(p\tau\pi)},
\end{flalign}
where (a) is because of combining Theorem~\ref{theorem:xt-xkt} and Theorem~\ref{theorem:xgktau-xktau}; (b) is because of Theorem~\ref{theorem:x_ppi_ptaupi-x_p_ptaupi}.


From \eqref{ieq:ckt+1<ckt},  we can get $F(\boldsymbol{x}_{\{p\}}^t)\geq F(\boldsymbol{x}_{\{p\}}^{t+1})$ for any $t\in [(p-1)\tau\pi,p\tau\pi)$. Recalling Condition (2.2) in Theorem \ref{theorem:Fwt-Fw*}, where $F(\boldsymbol{x}_{\{p\}}(p \tau\pi))-F\left(\boldsymbol{x}^{*}\right) \geq \varepsilon$ for all $p$, we can obtain $c_{\{p\}}(t)= F(\boldsymbol{x}_{\{p\}}^t) - F(\boldsymbol{x}^*) \geq \varepsilon$ for all $t\in [(p-1)\tau\pi,p\tau\pi]$ and $p$. Thus,
$
c_{\{p\}}(p\tau\pi)c_{\{p+1\}}(p\tau\pi)\geq\varepsilon^2.
$
According to Appendix~\ref{app_monotone}, we have $h(\tau,\delta_\ell)\geq0$ and $h(\tau\pi,\delta)\geq0$. Then substituting above inequalities into \eqref{ieq:frac:c(k+1)-ck}, we obtain
$
\frac{1}{c_{\{p+1\}}(p\tau\pi)}-\frac{1}{c_{\{p\}}(p\tau\pi)} \geq\frac{-\rho j(\tau,\pi,\delta_\ell,\delta)}{\varepsilon^2}.
$ Substituting the above inequality into \eqref{ieq:cKT-c10} and rearrange, we get
\begin{align}\label{ieq:frac:cKT-c10}
\frac{1}{c_{\{p\}}(T)}-\frac{1}{c_{\{1\}}(0)}\geq T\omega\alpha\sigma^2-(P-1)\frac{\rho j(\tau,\pi,\delta_\ell,\delta)}{\varepsilon^2}.
\end{align}
Recalling Condition (2.3) in Theorem \ref{theorem:Fwt-Fw*}, where $F(\boldsymbol{x}_{}^T)-F(\boldsymbol{x}^{*}) \geq \varepsilon$, and noting that $c_{\{p\}}(T)\geq\varepsilon$, we get
$
(F(\boldsymbol{x}_{}^T)-F(\boldsymbol{x}^{*}))c_{\{p\}}(T)\geq \varepsilon^2
.$
Thus,
{\small
\begin{flalign} \label{ieq:frac:FwT-cKT}
&\frac{1}{F(\boldsymbol{x}_{}^T)-F\left(\boldsymbol{x}^{*}\right)}-\frac{1}{c_{\{p\}}(T)}
=\frac{c_{\{p\}}(T)-(F(\boldsymbol{x}_{}^T)-F\left(\boldsymbol{x}^{*}\right))}{(F(\boldsymbol{x}_{}^T)-F\left(\boldsymbol{x}^{*}\right))c_{\{p\}}(T)}\nonumber\\
&=\frac{F(\boldsymbol{x}_{\{p\}}^T)-F(\boldsymbol{x}_{}^T)}{(F(\boldsymbol{x}_{}^T)-F\left(\boldsymbol{x}^{*}\right))c_{\{p\}}(T)}\nonumber\\
&\geq \frac{-\rho j(\tau,\pi,\delta_\ell,\delta)}{(F(\boldsymbol{x}_{}^T)-F\left(\boldsymbol{x}^{*}\right))c_{\{p\}}(T)}
\geq -\frac{\rho j(\tau,\pi,\delta_\ell,\delta)}{\varepsilon^2},
\end{flalign}
}

\noindent where the first inequality follows the same method to prove \eqref{ieq:frac:c(k+1)-ck}.

Combining \eqref{ieq:frac:cKT-c10} with \eqref{ieq:frac:FwT-cKT}, we get
$
\frac{1}{F(\boldsymbol{x}_{}^T)-F\left(\boldsymbol{x}^{*}\right)}-\frac{1}{c_{\{1\}}(0)}\geq T\omega\alpha\sigma^2-P\frac{\rho j(\tau,\pi,\delta_\ell,\delta)}{\varepsilon^2}
=T\omega\alpha\sigma^2-T\frac{\rho j(\tau,\pi,\delta_\ell,\delta)}{\tau\pi\varepsilon^2}
=T\left(\omega\alpha\sigma^2-\frac{\rho j(\tau,\pi,\delta_\ell,\delta)}{\tau\pi\varepsilon^2}\right).
$
Noting that $c_{\{1\}}(0)=F(\boldsymbol{x}_{\{1\}}^0)-F(\boldsymbol{x}^*)>0$, the above inequality can be expressed as
$
\frac{1}{F(\boldsymbol{x}_{}^T)-F\left(\boldsymbol{x}^{*}\right)}\geq T\left(\omega\alpha\sigma^2-\frac{\rho  j(\tau,\pi,\delta_\ell,\delta)}{\tau\pi\varepsilon^2}\right).
$
Recalling Condition (2.1) in Theorem \ref{theorem:Fwt-Fw*}, where $\omega\alpha\sigma^2-\frac{\rho j(\tau,\pi,\delta_\ell,\delta)}{\tau\pi \varepsilon^{2}}>0$, we obtain that the right-hand side of above inequality is greater than zero. Therefore, taking the reciprocal of the above inequality, we finally complete the proof of Theorem \ref{theorem:Fwt-Fw*}.

\subsection{Proof of Theorem \texorpdfstring{\ref{theorem:Fwf-Fw*}}{5}} \label{app_theorem5}


At the beginning, we see that Condition (1) in Theorem \ref{theorem:Fwt-Fw*} holds due to the Condition in Theorem \ref{theorem:Fwf-Fw*} ($0<\beta\eta(\gamma+1)\leq 1$, $0<\gamma< 1$, $0<\gamma_a< 1$, and $\forall\tau, \pi \in\{ 1,2,\ldots\}$).

\subsubsection{$\rho j(\tau,\pi)=0$} In this case, there is an arbitrarily small $\varepsilon>0$ that let Conditions (2.1)--(2.3) in Theorem \ref{theorem:Fwt-Fw*} hold. In this case, Theorem \ref{theorem:Fwt-Fw*}  holds. We also note that the right-hand side of \eqref{ieq:Fwf-Fw*} is equivalent to the right-hand side of \eqref{ieq:FwT-Fw*} when $\rho j(\tau,\pi)=0$. According to the definition of $\boldsymbol{x}^{\mathrm{f}}$ in \eqref{eq:wf}, we have
$
F\left(\boldsymbol{x}^{\mathrm{f}}\right)-F\left(\boldsymbol{x}^{*}\right) \leq F(\boldsymbol{x}_{}^T)-F\left(\boldsymbol{x}^{*}\right) \leq \frac{1}{T\omega\alpha\sigma^2},
$
which satisfies the result in Theorem \ref{theorem:Fwt-Fw*} directly. Thus, Theorem \ref{theorem:Fwf-Fw*} holds when $\rho j(\tau,\pi)=0$.

\subsubsection{$\rho j(\tau,\pi)>0$} In this case, we aim to find an $\varepsilon$ satisfying Condition (2.1), but Conditions (2.2) and (2.3) cannot be satisfied together so that $F(\boldsymbol{x}^{\mathrm{f}})-F(\boldsymbol{x}^*)$ can be bounded. We first define an $\varepsilon_0$, then we claim that any  $\varepsilon>\varepsilon_0$ is what we want to find.

We set $\varepsilon_0$ as the root of the following equation, 
\begin{align} \label{eq:varepsilon_0}
\varepsilon_0= \frac{1}{T\left(\omega\alpha\sigma^2-\frac{\rho j(\tau,\pi)}{\tau\pi\varepsilon_0^2}\right)}.
\end{align}
The positive root is
\begin{align}\label{eq:varepsilon_0_cont.}
\varepsilon_{0}=\frac{1}{2 T \omega \alpha\sigma^2}+\sqrt{\frac{1}{4 T^{2} \omega^{2} \alpha^{2} \sigma^4}+\frac{\rho j(\tau,\pi)}{\omega \alpha \sigma^2\tau\pi}}.
\end{align}
Through this way, since $\omega\alpha\sigma^2-\frac{\rho j(\tau,\pi)}{\tau\pi\varepsilon^2}$ increases with $\varepsilon$, $\varepsilon > \varepsilon_0$ will lead to Condition (2.1).


Next, using the proof by contradiction, we can prove that when $\varepsilon>\varepsilon_0$, there does not exist $\varepsilon>\varepsilon_0$ that satisfies both Conditions (2.2) and (2.3) in Theorem \ref{theorem:Fwt-Fw*} at the same time.

We assume that there exists such $\varepsilon > \varepsilon_0$, so that Conditions (2.1)--(2.3) hold and thus
Theorem \ref{theorem:Fwt-Fw*} holds. Then we have
$
F(\boldsymbol{x}_{}^T)-F\left(\boldsymbol{x}^{*}\right) \leq \frac{1}{T\left(\omega \alpha\sigma^2-\frac{\rho j(\tau,\pi)}{\tau\pi \varepsilon^{2}}\right)}
< \frac{1}{T\left(\omega\alpha\sigma^2-\frac{\rho j(\tau,\pi)}{\tau\pi\varepsilon_0^2}\right)} = \varepsilon_0,
$
which contradicts the Condition (2.3) in Theorem \ref{theorem:Fwt-Fw*}. 

Therefore, for any $\varepsilon>\varepsilon_0$, one of the following (A) or (B) holds. 
 (A) $\exists p\in [1,P]$ such that $F(\boldsymbol{x}_{\{p\}}^{p \tau\pi})-F\left(\boldsymbol{x}^{*}\right) \leq \varepsilon_0$ or (B) $F(\boldsymbol{x}_{}^T)-F\left(\boldsymbol{x}^{*}\right) \leq \varepsilon_0$.
(A) or (B) gives 
\begin{align}\label{ieq:min}
\min \left\{\min _{p\in [1,P]} F(\boldsymbol{x}_{\{p\}}^{p \tau\pi}) ; F(\boldsymbol{x}^T)\right\}-F\left(\boldsymbol{x}^{*}\right) \leq \varepsilon_{0}.
\end{align}

According to \eqref{ieq:frac:FwT-cKT}, when $t=p\tau\pi$, we have $F(\boldsymbol{x}_{}^{p\tau\pi}) \leq F(\boldsymbol{x}_{\{p\}}^{p\tau\pi}) + \rho j(\tau,\pi)$ for any cloud interval $\{p\}$. 
Combining it with \eqref{ieq:min}, we have
$
\min _{p\in [1,P]} F(\boldsymbol{x}_{}^{p \tau\pi})-F\left(\boldsymbol{x}^{*}\right) \leq \varepsilon_{0}+\rho j(\tau,\pi).
$
Recalling the definition of $\boldsymbol{x}^{\mathrm{f}}$ in \eqref{eq:wf}, $T=P\tau\pi$, and combining $\boldsymbol{x}^{\mathrm{f}}$ with above inequality, we get
$
F\left(\boldsymbol{x}^{\mathrm{f}}\right)-F\left(\boldsymbol{x}^{*}\right) \leq \varepsilon_{0}+\rho j(\tau,\pi).
$
Substituting \eqref{eq:varepsilon_0_cont.} into above inequality, we finally get the result in \eqref{ieq:Fwf-Fw*}, which completes the proof of Theorem \ref{theorem:Fwf-Fw*}.

\subsection{Proof of Theorem \texorpdfstring{\ref{theorem:fast}}{6}} \label{app_theorem6}

{
When $\eta \to 0^+$, we have $\gamma A \simeq 1$, $\gamma B \simeq \gamma$, and $J \simeq \frac{\gamma^2}{(1-\gamma)^2}$. Therefore,
{\small
\begin{flalign*}
&\lim_{\eta\to0^+}{h(\tau, \delta_\ell)}&\\
=&\lim_{\eta\to0^+}{\eta \delta_\ell\left[I(\gamma A)^{\tau}+J(\gamma B)^{\tau}-\frac{1}{\eta \beta}-\frac{\gamma^2(\gamma^{\tau}-1)-(\gamma-1) \tau}{(\gamma-1)^{2}}\right]}&\\
=&\lim_{\eta\to0^+}{\eta\delta_\ell\left(I-\frac{1}{\eta\beta}\right)}&\\
=&\lim_{\eta\to0^+}{\eta\delta_\ell\left(\frac{1}{(1-\gamma)(\gamma A-1)}-\frac{1}{\eta\beta}\right)}&\\
=&\frac{\delta_\ell}{1-\gamma}\lim_{\eta\to0^+}{\frac{\eta}{\gamma A-1}}-\frac{\delta_\ell}{\beta}&\\
=&\frac{\delta_\ell}{1-\gamma}\lim_{\eta\to0^+}{\frac{1}{(\gamma A-1)^{'}}}-\frac{\delta_\ell}{\beta}&\\
=&\frac{\delta_\ell}{1-\gamma}\frac{1-\gamma}{\beta}-\frac{\delta_\ell}{\beta}=0
\end{flalign*}
}

\noindent where the second last line is because of the L'H\^{o}pital's rule. Then, we can  derive $s(\cdot)\simeq0$. Afterwards, we have $j(\cdot)\simeq0$ and $\hat{{j}}(\cdot)\simeq0$. Therefore, $f_{HierMo}(T) \simeq \frac{1}{T\omega\alpha\sigma^2}$ and $f_{HierFAVG}(T) \simeq \frac{1}{T\omega\hat{\alpha}\sigma^2}$. Based on the conditions in Theorem \ref{theorem:fast}, we have $\alpha>\hat{\alpha}$. Therefore, we have $f_{HierFAVG}(T) - f_{HierMo}(T) >0$, which completes the proof of Theorem \ref{theorem:fast}.

}

\bibliographystyle{IEEEtran}
\bibliography{main}


 





\end{document}